\documentclass[11pt]{article}
\usepackage[title]{appendix}
\usepackage[english]{babel}
\usepackage[utf8x]{inputenc}
\usepackage[T1]{fontenc}
\usepackage{pdfpages}
\usepackage{booktabs}
\usepackage[authoryear]{natbib}
\usepackage{caption}
\usepackage{breqn}
\usepackage{subcaption}
\usepackage{comment}
\usepackage[a4paper,top=3cm,bottom=2cm,left=3cm,right=3cm,marginparwidth=1.75cm]{geometry}

\usepackage{amsmath, amssymb}
\usepackage{amsmath}
\usepackage{amsfonts}
\usepackage{amsthm}
\usepackage{bbold}
\usepackage{graphicx}
\usepackage{enumitem}  
\usepackage{soul}
\usepackage{float}
\usepackage[ruled,vlined]{algorithm2e}
\usepackage[colorinlistoftodos]{todonotes}
\usepackage[colorlinks=true, allcolors=blue]{hyperref}
\usepackage{pgfplots}
\usepackage[ruled,vlined]{algorithm2e}
\usepackage{setspace}

\usepackage{todonotes}

\DeclareMathOperator{\Var}{Var}

\newcommand{\FedSGD}{\texttt{FedSGD}}
\newcommand{\FedAvg}{\texttt{FedAvg}}
\newcommand{\FedHybrid}{\texttt{FedHybrid}}
\newcommand{\FedNewton}{\texttt{FedNewton}}

\newtheorem{theorem}{Theorem}[section]

\newtheorem{lemma}[theorem]{Lemma}
\newcommand\numberthis{\addtocounter{equation}{1}\tag{\theequation}}

\newtheorem{proposition}[theorem]{Proposition}
\newtheorem{assumption}{Assumption}

\newtheorem{definition}[theorem]{Definition}
\newtheorem{remark}{Remark}[section]

\newtheorem*{namedthm*}{\thistheoremname}
\newcommand{\thistheoremname}{} 

\def\b{\boldsymbol}

\def\P{\mathbb P}

\def\EE{\mathbb E}
\def\II{\mathbb{I}}
\def\RR{\mathbb R}
\def\eps{\varepsilon}
\def\Var{\operatorname{Var}}

\def\calS{\mathcal{S}}
\def\bu{\textbf{u}}
\def\bx{\textbf{x}}

\theoremstyle{definition}

\title{Statistical Limits and Efficient Algorithms for Differentially Private Federated Learning}

\author{Arnab Auddy \qquad Xiangni Peng \qquad Subhadeep Paul\\
Department of Statistics\\ The Ohio State University}

\date{}

\begin{document}

\maketitle

\begin{abstract}
Federated Learning is a leading framework for training ML and AI models collaboratively across numerous user devices or databases. We study the trade-offs among estimation accuracy, privacy constraints, and communication cost for differentially private (DP) federated M estimation. The two standard methods in the literature are \FedAvg, which may suffer from high federation bias, and \FedSGD, which can incur high communication cost. Aimed at improving accuracy at a reduced communication cost, we propose \FedHybrid, which uses \FedSGD\ starting with an improved initialization by the \FedAvg\ estimator. We propose \FedNewton, which averages local Newton iterations to reduce bias in \FedAvg, achieving an estimation accuracy comparable to \FedSGD\ with much fewer communication rounds when the number of clients grows sufficiently slowly. We establish finite sample upper bounds on the mean-squared error rates of the DP versions of these estimators as functions of the number of clients, local sample sizes, privacy budget, and number of iterations. We further derive a minimax lower bound on the MSE of any iterative private federated procedure that provides a benchmark to assess the optimality gap of these methods. We numerically evaluate our methods for training a logistic regression and a neural network on the computer vision datasets MNIST and CIFAR-10. 
\end{abstract}


\section{Introduction}

Federated Learning (FL) is a machine learning technique in which multiple client devices collaboratively train a model without sharing raw data with a central server \citep{mcmahan2017communication,kairouz2021advances}. The data typically resides on numerous client devices (e.g., mobile or wearable devices) or in client databases (e.g., hospital or financial institution databases) and cannot be sent to a central server due to user privacy or data ownership concerns. Recently, some AI systems promise to deliver AI tools to users that are trained on user data without requiring the data to leave the user's device \citep{apple}. The key tools for accomplishing this ambitious goal are privacy-preserving FL \citep{paulik2021federated}. Moreover, FL has also gained prominence recently as a possible solution to train AI and ML systems while complying with regulations that prevent data sharing, e.g., the European Union's General Data Protection Regulations (GDPR) \citep{brauneck2023federated}. In addition to protecting privacy, FL can also reduce costs associated with moving and storing data in centralized cloud servers \citep{noble2022differentially}.

Two foundational methods in FL are Federated Stochastic Gradient Descent (\FedSGD) and Federated Averaging (\FedAvg) \citep{mcmahan2017communication}.
In \FedAvg, the central server aggregates (for example, by taking a weighted average) local model updates from clients, often communicating only occasionally after multiple rounds of local client computations. In \FedSGD, in each round, each client computes the gradient from its local data and sends it to the server, which then updates the global model pooling these local gradients, requiring communication in every round. Many federated optimization methods such as \texttt{FedProx} \citep{li2020federated}, \texttt{SCAFFOLD} \citep{karimireddy2020scaffold}, \texttt{FedPAQ} \citep{reisizadeh2020fedpaq}, and adaptive federated optimization methods, such as \texttt{FedAdagrad}, \texttt{FedAdam}, and \texttt{FedYogi} \citep{reddi2020adaptive} can be viewed as methods built on \FedAvg~and \FedSGD.

A major concern in all of the above methods is guaranteeing user privacy. Although FL does not directly share raw data, it does not automatically guarantee privacy, as information may still be leaked through gradients, model updates, or other summaries sent by the clients \citep{geiping2020inverting,wei2020federated,geyer2017differentially}. Differential privacy (DP) provides a framework for sharing functions of data (i.e., statistics) while preserving the privacy of individual users \citep{dwork2008differential,dong2022gaussian,cai2021cost}. Several works study differentially private federated learning from the optimization viewpoint \citep{mcmahan2017learning,kato2024uldp,kairouz2021practical}. Ensuring privacy necessitates the introduction of randomization, and these works study the effect of this randomization on optimization behavior.

A second major concern in federated learning is communication cost. It is quite common for iterative procedures to be used in FL, where the server and the clients exchange model information repeatedly. This cost can become very large for exchanging large parameter and gradient vectors, such as those in deep neural networks. Several papers therefore study communication-efficient federated learning. In homogeneous settings, communication can be reduced by local updates, periodic averaging, quantization, double compression, and less frequent communication with the central server \citep{mcmahan2017communication,konevcny2016federated,stich2018local,reisizadeh2020fedpaq,spiridonoff2021communication,gao2021convergence}. We note that the two concerns of privacy and communication are closely related, and \cite{noble2022differentially,zhang2022understanding} show that privacy, clipping, and communication cost need to be understood jointly.

Despite enormous interest in FL in recent years and the proliferation of methods as described above, key challenges remain unresolved, especially in terms of statistical theory. The first challenge is to theoretically understand the performance and estimation error rates of FL methods, especially when there are a large number of clients, each with a small amount of data. In particular, how to effectively reduce the bias of Federated averaging to boost the accuracy of parameter estimation in this scenario. The second one is to design methods that effectively preserve the privacy of the users and understand how privacy impacts estimation accuracy. Overcoming these challenges is the goal of the current work.

Our contribution in this paper is \textit{twofold}. Firstly, we connect private federated learning with statistical estimation theory by deriving finite-sample mean squared error bounds for the parameter estimation error of private federated algorithms. We study how privacy noise and communication rounds affect statistical accuracy. This leads to guarantees that are meaningful for practical use. 
We consider several methods for private Federated Learning of M-estimators using noise-added gradient descent \citep{bassily2014private,avella2023differentially,mcmahan2017communication} and provide MSE upper bounds for them. All of these methods are based on variants of noisy gradient descent or noisy Newton iterations, where a Gaussian noise is added in each iteration of the gradient descent algorithm. The methods vary in terms of the extent of communication between the server and clients. We further prove a \textit{lower bound} over the MSE of any Private Federated Learning method, to understand the \textit{optimality gaps} of the methods.

 Our second contribution is to propose two new methods \FedHybrid\ and \FedNewton. Recall that in \FedSGD, (we consider only the non-stochastic version of it for theoretical results), the gradient descent steps are all performed by the server aggregating gradients from the clients. 
Clearly, this method has a high communication cost of $O(mdK)$, where $m$ is the number of clients, $d$ is the dimension of the parameters, and $K$ is the number of gradient iterations, and is identical to the number of rounds of communication for this method. 
Our first novel estimator, \FedHybrid, is a hybrid between \FedSGD~and \FedAvg. This method differs from \FedSGD~in terms of its initialization. While the server starts from an arbitrary initial value in \FedSGD, in our \FedHybrid, we run \FedAvg~with 1 round of communication ($R=1$) consisting of $K_1$ local training iterations at each client to obtain an initial value. The warm start from \FedAvg~allows this algorithm to run $K_2$ gradient communication rounds, where $K_2\ll K$. This method has a \textit{lower communication cost} of $O(mdK_2)$ than \FedSGD\ and theoretically achieves \textit{higher accuracy} than communication cost-efficient \FedAvg. Therefore, this method represents a \textit{middle ground} in terms of communication cost and accuracy tradeoff. Our second proposal is the \FedNewton~method, which consists of \FedAvg~followed by 1 Newton iteration locally at the clients and aggregation of the updated parameters at the clients. This method is designed to mitigate the bias from \FedAvg\ in a communication-efficient way. We theoretically prove that this method achieves comparable MSE to \FedSGD~with much less communication cost, provided the number of clients grows sufficiently slowly as the total sample size increases. The \FedNewton~estimator is related to the \texttt{FedFisher} in \cite{jhunjhunwala2024fedfisher}, however, unlike \texttt{FedFisher}, we do not need to communicate the Hessian matrix or an approximation of it to the central server and therefore is more communication efficient.

Our estimators are derived, and their properties studied, under federated $\mu$-GDP guarantees, which use the Gaussian differential privacy framework of \cite{dong2022gaussian} on data distributed across $m$ clients. This allows us to protect data privacy not only against a third party, but also towards the server, which is the so-called ``honest but curious'' server scheme of FL \citep{noble2022differentially}. Under such a setting, with strong convexity and other assumptions standard in the $M$-estimation literature, we derive the following rates for the mean squared error of federated $M$-estimators:
\[
\inf_{\hat{\theta}}\sup_{\theta,P_{\theta}\in \mathcal{P}}
\mathbb{E}\|\hat{\theta}-\theta\|^2
\asymp 
\left(
\frac{d}{N}
+
\frac{md^2L(d,N,\mu)}{N^2\mu^2\log(1/\mu)}
\right)
\]
where $N$ is the total sample size, $L(d,N,\mu)$ is a factor of order at most $O(\log(d)\log(Nd)\log(1/\mu))$.

For ease of presentation, in the above, we assume that $N$ samples are distributed evenly across $m$ clients. More general cases of uneven sample size distributions can be found in later sections. The above bound builds on an MSE lower bound over the class of all possible $\mu$-federated GDP estimators. On the other hand, the near-optimal upper bound is attained by the \FedSGD, \FedHybrid, and \FedNewton. 

With the above guarantees of statistical accuracy, we now turn to the communication rounds required across the different methods. Note that the number of client-server communication rounds coincides with the number of gradient iterations $(K)$ for \FedSGD, and with the warm-started gradient iterations $(K_2)$ for \FedHybrid. The next table makes this comparison explicit in terms of $m$ and $N$.

\begin{table}[h]
\begin{center}
\begin{tabular}{|c|c|c|c|c|}
\hline
method                                                         & \FedSGD                                                 & \FedHybrid                                              & \FedAvg                                                & \FedNewton                                              \\ \hline
\begin{tabular}[c]{@{}c@{}}Communication\\ Rounds\end{tabular} & $\Omega(\log N)$                                                 & $\Omega(\log m)$                                                 & 1                                                     & 2                                                      \\ \hline
MSE                                                            & \begin{tabular}[c]{@{}c@{}}near\\ optimal\end{tabular} & \begin{tabular}[c]{@{}c@{}}near\\ optimal\end{tabular} & \begin{tabular}[c]{@{}c@{}}sub\\ optimal\end{tabular} & \begin{tabular}[c]{@{}c@{}}near\\ optimal\end{tabular} \\ \hline
\end{tabular}
\end{center}
\end{table}
In contrast to \FedSGD, \FedHybrid, and \FedNewton, we show that the \FedAvg~ suffers from federation bias and has an MSE of $O\left(\frac{d}{N} + \frac{m^2 d^2}{N^2}+\frac{mKd^2}{\mu^2N^2} \right)$, which is worse compared to \FedSGD~ and for $d,\mu=O(1)$, away from the lower bound if $m\gg N^{1/2}$, i.e., when we have a large number of clients each holding relatively small amount of data. Both the new proposed methods \FedHybrid~and \FedNewton~achieve error rates comparable to \FedSGD, and within $\log n$ factor of the optimal rate. While \FedSGD~and \FedHybrid~are communication-heavy, they do not impose any restriction on the number of clients. In contrast, the near optimality of  \FedNewton~estimator is contingent on the number of clients satisfying $m\ll N^{2/3}$, which is a condition weaker than that imposed by \FedAvg.

We evaluate the finite sample performance of our methods using two simulation studies, corresponding to Poisson and Logistic regression models, respectively. In both settings, the results show that the MSE decreases as either the number of clients or the local client sample sizes increases. Further, we see that as the number of iterations increases, while the MSE decreases at first, it goes up with a higher number of iterations due to the privacy-accuracy tradeoff. Finally, when the number of clients increases without increasing the total sample size, the performance of \FedAvg~deteriorates significantly, while the new \FedNewton~provides excellent protection against this decline. We apply our methods to the benchmark MNIST and CIFAR10 image datasets using binary and multi-class logistic regression models, as well as for training Convolution Neural Networks (CNN). For \FedNewton~applying Newton iteration on the entire set of parameters of a CNN model maybe unstable. We therefore propose to apply the Newton iteration only to the last fully connected layer of the CNN, leaving the convolution layers unchanged during the Newton iteration. For CNN training, we also devise an iterated version of the \FedNewton~method. We compare the methods in terms of test accuracy using CNN for the MNIST and CIFAR10 datasets. 

Our work is closely related to a growing body of theoretical results in federated learning. Firstly, there are earlier works that focus on optimization convergence rather than statistical estimation accuracy. These works mainly analyze convergence rates, optimization error, and the trade-off between communication and optimization accuracy \citep{haddadpour2019convergence,li2019convergence,stich2018local,reddi2020adaptive,qu2020federated,gao2021convergence, wei2020federated}.  For example, \cite{karimireddy2020scaffold} analyzes \FedAvg~under heterogeneous data and proposes SCAFFOLD to correct client drift, providing a convergence rate. Moreover, this framework is extended to the private setting in~\cite{noble2022differentially}, who analyze the privacy-utility trade-off through convergence bounds. In both works, the main theoretical results are expressed in terms of optimization error rather than parameter estimation error. As shown in \cite{avella2023differentially}, bounds on objective sub-optimality do not directly control parameter estimation error, which is not sufficient for statistical inference.

There is also a distributed inference literature focusing on estimation accuracy for M-estimators. For example, \cite{zhang2013communication} showed that distributed averaging of M-estimators can achieve comparable estimation accuracy to centralized estimation under regularity conditions. \cite{huang2019distributed} further discusses that adding one Newton update at the central server (one-step-estimation) allows the aggregated distributed estimator to achieve the same asymptotic behavior as the centralized estimator. Communication-efficient distributed statistical inference methods using surrogate likelihood were studied in \citep{duan2022heterogeneity,jordan2019communication}. Further, \cite{gu2023distributed,gu2024statistical} considered heterogeneous and decentralized distributed statistical inference. However, these works do not study the properties of gradient descent iterations that are typically employed in modern federated learning, nor do they consider the impact of privacy noise. In contrast, our work considers privacy, analyzes the relationship between privacy, accuracy, communication cost, and the number of gradient iterations for FL methods, and provides a lower bound to benchmark accuracy against. Finally, our contribution is related to the recent works on distributed statistical learning under differential privacy requirements. See, for example, the results on nonparametric distributed learning in \cite{cai2024optimal,auddy2024minimax,xue2024optimal}. Our work contributes to this literature by characterizing the cost of distributed privacy requirements in parametric large-dimensional $M$-estimation problems, while also proposing methods that reduce the communication requirements.

\section{Background and notations}

In this section, we present the basic settings and notations under which we develop our privacy-preserving federated M-estimators. 
Suppose there are $m$ clients where client $i$ possesses local dataset $\{X_j^{(i)}\}, \, j=1,\ldots,n$ with i.i.d. samples from distribution $F$. We denote $n_i$ as the local sample size for client $i$, so that the total sample size is $N = \sum^{m}_{i=1}n_i$. It will also be convenient to denote the average sample size as $n=N/m$.
Let $\{\rho(x; \theta): \theta \in \Theta \subset \mathbb{R}^d\}$ be a family of 
loss functions. Then our objective is to estimate
\[
\theta_0 = \arg\min_{\theta \in \Theta} \mathbb{E}[\rho(X, \theta)],
\quad \text{where } \theta_0 \text{ is assumed to be unique,}
\]
without having the data sent to a central server and protectingthe  privacy of users.

The following notations are used throughout the following sections. The local empirical loss for each client $i$ is 
    $L_i(\theta)=\frac{1}{n_i} \sum_{j=1}^{n_i} \rho(x_j^{(i)},\theta)$, and the global loss is 
$L(\theta) = \frac{1}{N} \sum_{i=1}^{m} n_iL_i(\theta)$. Accordingly we denote $\hat{\theta}_i$ as the minimizer of the local empirical loss, $\hat{\theta}_i = \arg\min_{\theta \in \Theta} L_i(\theta).$
This estimator is widely known as the $M$-estimator or the empirical risk minimization estimator.

Let us denote the first-order derivatives of the global loss as $
    \dot{L}(\theta)= \nabla L(\theta) =\sum_{i=1}^m \frac{n_i}{N} \nabla L_i(\theta).$
The corresponding first-order derivatives of the local loss are denoted as $
\dot{L}_i(\theta) = \nabla L_i(\theta)$. The derivatives of the expected loss function are denoted as, $ \dot{L}_0(\theta) = \nabla L_0(\theta)$.
 The expectation of the global loss is denoted as $
    L_0(\theta) = \mathbb{E}[\rho(X,\theta)].$ As stated before, our objective is to estimate the unique global minimizer of this expected loss function.

\subsection{Assumptions}
Before introducing our estimators and stating our theoretical results, we introduce the necessary assumptions. Similar assumptions can be found in previous work on M-estimation, see, e.g., \cite{avella2023differentially}.
\begin{assumption}
The parameter space $\Theta \subset \mathbb{R}^d$ is a compact and convex set.
\end{assumption}

\begin{assumption}
The expectation of the loss function $\rho(x;\theta)$ is $\tau_1$-strongly convex in $\theta$, i.e.,
\[
\EE (\nabla^2 \rho(x,\theta))\succeq \tau_1I_d
\quad
\text{for all }\theta.
\]
\end{assumption}

\begin{assumption}[sub-Gaussian Gradient and Hessian]
\label{ass:grad}
There exist constants $C_1,C_2>0$ such that
\[
\P\left( \mathbf{v}^{\top}
\left(\nabla \rho(X,\theta) 
- \EE [\nabla \rho(X,\theta)]\right)
\ge t\right)
\le \exp(-t^2/C_1^2).
\]
\[
\P\left( \mathbf{v}^{\top}
(\nabla^2 \rho(X,\theta) 
- \EE [\nabla^2 \rho(X,\theta)])\mathbf{v}
\ge t\right)
\le \exp(-t^2/C_2^2).
\]
for any fixed $\theta\in \Theta$ and $\mathbf{v}$ with $\|\mathbf{v}\|=1$.
\end{assumption}

\begin{assumption}[Smoothness]
There exists positive constants $W$ and $\tau_2$ such that for all $\theta, \theta' \in \Theta$ and $x \in \mathcal{X}$:
\[
\|\nabla^2 \rho(x,\theta) - \nabla^2 \rho(x,\theta')\|
\le W \|\theta - \theta'\|,
\qquad
\|\nabla^2 \rho(x,\theta)\| \le \tau_2.
\]
\end{assumption}

\subsection{Differential Privacy}

We will use the following notion of $\mu$- Gaussian Differential Privacy 
from \cite{dong2022gaussian}.

\begin{definition}[Central Gaussian Differential Privacy]\label{def:cent-gdp} For any $n\in \mathbb{N}$, let $\mathcal{D}_n^*$ denote the space of all datasets of size $n$.  Consider two datasets $D, D' \in \mathcal{D}^*_N$ that differ in exactly one datum. A mechanism $M$ is said to be $\mu$-Gaussian differentially private ($\mu$-GDP) if
\[
T(\alpha,M((D), M(D')|D,D')\ge \Phi(\Phi^{-1}(1 - \alpha) - \mu)\quad
\text{for all }\alpha\in (0,1]
\]
where $\Phi(\cdot)$ is the standard Gaussian cdf, and for two random variates $X\sim P$ and $Y\sim Q$
\[
T(\alpha, X, Y)= \inf \{1-\EE_Q(\phi) : \EE_P(\phi) \le \alpha\} 
\]
and the infimum is over all possible measurable rejection rules $0\le\phi\le 1$.
    
\end{definition}

\begin{definition}[Federated Gaussian Differential Privacy]\label{def:fed-gdp} For any $n\in \mathbb{N}$, let $\mathcal{D}_n^*$ denote the space of all datasets of size $n$. Consider datasets $\{D_s, D'_s \in \mathcal{D}^*_{n_s}:1\le s \le m\}$ where for each $s$, $D_s$ and $D_s'$ differ in exactly one datum. 

A mechanism $M$ is said to be $\mu$-federated Gaussian differentially private ($\mu$-fed-GDP) if
\[
T(\alpha,M(D_s), M(D'_s)|\{D_j,D_j':j\neq s\})\ge \Phi(\Phi^{-1}(1 - \alpha) - \mu)
\quad 
\text{for all }\alpha\in (0,1]
\text{ and }1\le s\le m
\]
where $\Phi$ and $T$ are as defined in Definition~\ref{def:cent-gdp}.    
\end{definition}

Our choice of GDP and {\rm fed}-GDP as our preferred frameworks for differential privacy are motivated by their exact privacy accounting for the \emph{Gaussian mechanism}, which adds Gaussian noise calibrated to the global sensitivity of the estimators. For completeness, we state the related ideas in the appendix.

\section{Proposed Methods and Theoretical Results}

In this section, we analyze four methods for estimating M-estimators under the Gaussian mechanism of privacy. As discussed in the introduction, these methods are:
\begin{enumerate}
    \item a $K$-iteration server-side private gradient descent M-estimator similar to \texttt{FedSGD};
    \item a $R$ round federated averaging M-estimator with client-level privacy \FedAvg;
    \item a new $K_1$-local and $K_2$-server private federated M-estimator, \FedHybrid;
    \item a one local Newton iteration improvement for \FedAvg~
    that we call \FedNewton.
\end{enumerate}
All of these algorithms are variants of noisy gradient descent (or Newton iteration) where Gaussian privacy noise is added to gradient (or Newton) iterations. All of these algorithms aim to approximate the same population M-estimator, but differ in how local updates, communication, and privacy noise are incorporated across clients and the central server. These design differences allow us to study the trade-offs between statistical efficiency and privacy protection in federated settings. 

In what follows, to make the methods comparable to each other, we constrain that all methods must maintain that client updates are $\mu$-fed GDP, as defined in Definition~\ref{def:fed-gdp}.
We operate under the ``honest but curious'' server scheme of FL \citep{noble2022differentially}, where the clients need to ensure privacy of their users not only from a malicious third party observing the server's model outputs, but also from the server itself. Therefore, the functions of data (parameter estimates, gradients, etc.) that the clients send to the server need to be privatized under the $\mu$-fed GDP framework. Following previous work in \cite{noble2022differentially,wei2020federated}, we distinguish between two types of privacy notions. First, we provide a privacy guarantee for the clients towards an external third party who can only view the aggregated model updates released by the server, and not the client updates. The second notion of privacy is to the honest but curious server who can observe all the privatized local client updates. These two notions are formalized by Definition~\ref{def:cent-gdp} and Definition~\ref{def:fed-gdp}, respectively.

\subsection{ Server Gradient Descent \FedSGD}

\begin{algorithm}[h]
\caption{$K$-Iteration \texttt{FedSGD} }
\label{alg:AG1}
\DontPrintSemicolon
\KwIn{Data $\{X^{(i)}\}_{i=1}^m$ with local client sample sizes $\{n_i\}$; step size $\eta$; iters $K$; privacy parameter $\mu>0$; initialization $\theta^{(0)}$; loss function $\rho(\cdot,\theta)$; weight vector $\{\mathbf{w}\in [0,1]^m:\sum_iw_i=1\}$; clipping bound $B$.}
\KwOut{$\theta^{(K)}$}
\BlankLine
\textbf{Set} $g(x,\theta)=\nabla\rho(x,\theta)$,\quad $N=\sum_{i=1}^m n_i$.\;
\textbf{Noise:} $\sigma_i=\dfrac{2B \sqrt{K}}{\mu n_i}$.\;

\BlankLine
\For{$k=0,\dots,K-1$}{
\For(\textbf{}){$\textit{each client } i=1,\dots,m$}{
$g_i^{(k)}=\dfrac{1}{n_i}
\displaystyle
\sum_{j=1}^{n_i} g(x_j^{(i)},\theta^{(k)})$,\quad
$\tilde g_i^{(k)}=g_i^{(k)}+\sigma_i Z_{ik}$,\ $Z_{ik}\sim\mathcal{N}(0,I_d)$.\;
Send $\tilde g_i^{(k)}$ to server.\;
}
$\tilde g^{(k)}=\displaystyle\sum_{i=1}^m w_i\tilde g_i^{(k)}$,\quad
$\theta^{(k+1)}=\theta^{(k)}-\eta\,\tilde g^{(k)}$.\;
Broadcast $\theta^{(k+1)}$.\;
}
\textbf{Return} $\theta^{(K)}$.\;
\end{algorithm}

In the first algorithm, which we refer to as \FedSGD~ (displayed in Algorithm~\ref{alg:AG1}), the bulk of the computation is carried out on the server. In particular, the server runs $K$ gradient descent iterations starting from an arbitrary initializer $\theta^{(0)}$. At each iteration, every client provides a privatized gradient by computing its local gradient and perturbing it with Gaussian noise calibrated to the privacy parameter. These privatized gradients are then sent to the server, where they are aggregated using a weighted average to form a global gradient estimate. The server updates the global parameter using this aggregated gradient and learning rate $\eta$ and broadcasts the updated parameter back to all clients. The current iteration ends with this broadcast step, and the next iteration begins. This algorithm is closely related to centralized gradient descent, with the key difference being that privacy noise is added at the client level before aggregation. This design mimics our ``honest but curious'' server assumption, where the clients (e.g., mobile devices, hospitals, financial institutions) need to ensure the privacy of their users by privatizing gradients that they communicate to the server.

The next Lemma shows that if the scaling for the Gaussian noise is set to $\frac{2B\sqrt{K}}{\mu n_i}$, then Algorithm \ref{alg:AG1} is $\mu$-fed-GDP as defined in Definition~\ref{def:fed-gdp}. 
\begin{lemma}
\label{lem:AG1muGDP}
    (Privacy guarantee for \FedSGD) With the scaling of the Gaussian noise at the $i$th client set as $\frac{2B\sqrt{K}}{\mu n_i}$, the \FedSGD~algorithm is $\mu$-fed-GDP in the sense of Definition~\ref{def:fed-gdp} and the full Algorithm \ref{alg:AG1} is  $\mu/\sqrt{m}-$GDP towards a third party in the sense of Definition~\ref{def:cent-gdp}.
\end{lemma}

The next theorem is our main result for Algorithm \ref{alg:AG1}, which provides a finite sample bound on the MSE of the parameter estimate.

\begin{theorem}[Error Bound of \FedSGD]\label{th:alg1}
Suppose Assumptions 1-4 above hold. Let $\theta^{(K)}_{({\rm AG1})}$ be the output of Algorithm~\ref{alg:AG1} with $K$ iterations initialized at $\theta^{(0)}$ and learning rate $\frac{1}{2\tau_2}\leq \eta \leq \frac{9}{10\tau_2}$. 
Then with $B=C_B\sqrt{d\vee \log N}$ for a constant $C_B>0$, we have the bound
\begin{align}\label{eq:alg1-gd-mse}
        \mathbb{E}\left[\left\|\theta^{(K)}_{({\rm AG1})}-\theta_0\right\|^2\right]
    \le~&~
    3\left(1-\frac{\tau_1}{3\tau_2}\right)^{2K}\EE\|\theta^{(0)}-\theta_0\|^2 \nonumber \\
&+
\frac{4}{(1-\gamma_1)^2}
\left(
\eta^2
\EE\left\Vert
\sum_{i=1}^mw_i
\dot{L}_i(\theta_0)\right\Vert^2
+\frac{\eta^2C_B^2Kd(d\vee \log N)}{\mu^2}
\sum_{i=1}^m\frac{w_i^2}{n_i^2}
\right)
\end{align}
where $\gamma_1=\left(1+C\sqrt{\frac{d\vee \log N}{n_{\min}}}-\eta\tau_1\right)$. In particular, if $n_{\min}\ge Cd$ for some constant $C>0$ and $K=2\log(dN)/\log(3\tau_2/(3\tau_2-\tau_1))$ iterations and any $\theta^{(0)}$ satisfying $\|\theta^{(0)}\|=O(\sqrt{d})$, the mean squared error of estimating $\theta_0$ by $\hat{\theta}_{\rm (AG1)}=\theta^{(K)}_{({\rm AG1})}$  satisfies
\begin{align}
\mathbb{E}\!\left[\left\|\hat{\theta}_{({\rm AG1})}- \theta_0 \right\|^2\right]
    \leq 
    \frac{17}{\tau_1^2}
    \left(
 \sum_{i=1}^m
 \left(
    \frac{{\rm trace}(\Sigma)}{n_i}
    +\frac{2C_B^2 d(d\vee \log N)\log(dN)}{\mu^2 n_i^2\log(3\tau_2/(3\tau_2-\tau_1))}\right)^{-1}
 \right)^{-1}
    \label{ag1rate}
\end{align}
where $\Sigma:=\Var(\nabla \rho(x_1^{(1)},\theta_0))$, and optimal weights $\hat{w}_i\propto \left(
    \frac{{\rm trace}(\Sigma)}{n_i}
    +\frac{C_B^2d(d\vee \log N)K}{\mu^2 n_i^2}\right)^{-1}$ with $\displaystyle\sum_{i=1}^m\hat{w}_i=1$.
\end{theorem}

In the above theorem, we obtain an upper bound on the MSE of Algorithm \ref{alg:AG1} after $K$ iterations without any restriction on the sample sizes. For ease of understanding, we present two simplifications of the above rate, aiming to explain the dichotomy between utility and federated privacy in our problem.

\begin{remark}[Rate Simplification] The upper bound in Theorem~\ref{th:alg1} can be simplified to
\begin{equation}\label{eq:mse-ag1-simpler}
\mathbb{E}\!\left[\left\|\hat{\theta}_{({\rm AG1})}- \theta_0 \right\|^2\right]
    \leq 
   \frac{Cd}{\sum_{i=1}^m(n_i\wedge (n_i^2\mu^2/(d\vee \log N)\log(N)))}
\end{equation}    
for a constant $C>0$, under standard assumptions wherein ${\rm trace}(\Sigma)\le Cd$. The denominator implies that the effective sample size contribution from each client is the minimum of the actual sample size $n_i$ and its privacy-affected version $n_i^2\mu^2/d\log N$. If $\mu$ is sufficiently small, in particular $\mu\lesssim n_i^{-1/2}$, this sample size reduction is significant and leads to a slower MSE decay. This upper bound also coincides with the minimax lower bound in Theorem \ref{th:low-bd}.
\label{rmk:simpler-rate}
\end{remark}

\begin{remark}[Comparable sample sizes] When $n_{\max}/n_{\min}\le C$ for some constant $C>0$, the above MSE bound reduces to:
\begin{equation}\label{eq:mse-ag1-equal-ni}
\mathbb{E}\!\left[\left\|\hat{\theta}_{({\rm AG1})}- \theta_0 \right\|^2\right]
    \leq 
    \frac{C}{\tau_1^2}
    \left(
    \frac{{\rm trace}(\Sigma)}{N}
    +
    \frac{C_B^2md(d\vee \log N)\log(dN)}{\mu^2N^2}
    \right)
\end{equation}
for a constant $C>0$, where $N=\displaystyle\sum_{i=1}^m n_i$ and we use the weights $\{w_i=n_i/N:1\le i\le m\}$ when computing the weighted average of gradients $\{\tilde{g}_i^{(k)}:1\le i\le m\}$ for $k=1,\dots,K$.
\label{rmk:equalni}
\end{remark}

\begin{remark}[Weight selection]
The optimal choice of weights in Theorem~\ref{th:alg1} depends on the knowledge of ${\rm trace}(\Sigma)$, which is not known in practice. However, our assumptions \ref{ass:grad} imply bounded singular values of $\Sigma$, so that ${\rm trace}(\Sigma)$ can be replaced in the weights by $Cd$ for a constant $C>0$, leading to an MSE that will be worse by a factor of at most a numerical constant.
\end{remark}

The result in Equation~\eqref{eq:alg1-gd-mse} of Theorem \ref{th:alg1} represents a tradeoff between accuracy improvement due to higher number of iterations $K$ (first term) and accuracy loss due to enhanced variance of added Gaussian noise needed to protect against privacy leak for higher number of iterations. Note that in our framework, $K$ needs to be pre-selected, since the amount of noise introduced by the clients is a function of $K$. This is needed to guarantee the entire algorithm is private. Therefore, from Lemma \ref{lem:AG1muGDP}, a smaller K leads to the introduction of smaller privacy noise. The Equation \ref{ag1rate} in Theorem \ref{th:alg1} provides a finite sample bound on the MSE of Algorithm \ref{alg:AG1} for any large enough $K$ chosen such that the first term vanishes (such as $O(\log (dN))$) and any bounded initial value $\theta^{(0)}$ (for example, $\theta^{(0)}=0$). 

Moving to the individual terms of the MSE bound, we note from Equation~\eqref{eq:mse-ag1-equal-ni} that the first part of the right-hand side is the familiar MSE for M-estimators, while the second part is the additional error due to privacy. Since the entire gradient descent computation is being performed at the server (with aggregation of privatized client gradients), there is no loss of accuracy due to federation in the first term. Therefore, this method is well-suited when we have numerous clients. We further see in the Remark \ref{rmk:equalni}, that noting $N=mn$, where $n$ is the average sample size at the clients,  the second term is $O(\frac{1}{mn^2})$, while the first term is $O(\frac{1}{mn})$. Therefore, as the average sample size in the clients increases, the additional error due to privacy becomes smaller compared to the error rate in the non-private case.

\begin{remark}
    For $d\ge C\log N$, and keeping $K > C_1 \log (dN)$ explicitly in the upper bound, the privacy-related error term in Theorem \ref{th:alg1} and Remark \ref{rmk:equalni} scales as $O\left(\frac{mKd^2}{\mu^2N^2}\right)$. This can be compared with the scaling for the non-federated settings in \cite{avella2023differentially}. The federated framework we consider causes an additional factor of $m$ in our bound, 
    because of our stronger privacy requirement in our framework, to ensure clients satisfy $\mu$-fed-GDP in the sense of Definition~\ref{def:fed-gdp}.
\end{remark}

\medskip

In the next algorithm, we explore the idea of a better initialization, which can possibly lead to similar accuracy with fewer iterations.

\begin{algorithm}
\caption{$K_1$-Local / $K_2$-Server FedHybrid}
\label{alg:AG2}
\DontPrintSemicolon
\KwIn{Data $\{X^{(i)}\}_{i=1}^m$ with local client sample sizes $\{n_i\}$; step sizes $\eta_1,\eta_2$; iteration $K_{1i}=K_1,K_2$; privacy parameter $\mu>0$; loss function $\rho(\cdot,\theta)$; clipping bound $B$.}
\KwOut{$\theta^{(K_2)}$}
\BlankLine
\textbf{Set} $g(x,\theta)=\nabla\rho(x,\theta)$. Initialize $\theta_i^{(0)}=0$ for all $i$.\;

\BlankLine
\textbf{Stage I (clients):} For each client $i=1,\dots,m$, for $t=0,\dots,K_{1i}-1$:\;
\Indp
$\theta_i^{(t+1)}=\theta_i^{(t)}-\dfrac{\eta_1}{n_i}
\displaystyle
\sum_{j=1}^{n_i} g(x_j^{(i)},\theta_i^{(t)}) + a_i Z_{it}$,\quad
$Z_{it}\sim\mathcal{N}(0,I_d)$,\quad
$a_i=\dfrac{2B\eta_1\sqrt{2K_{1i}}}{\mu n_i}$.\;
\Indm
Client $i$ sends $\theta_i^{(K_{1i})}$ to server.\;

\BlankLine
\textbf{Stage II (server):} $\bar\theta=\displaystyle\sum_{i=1}^m \dfrac{n_i}{N}\theta_i^{(K_{1i})}$.\;
Run $K_2$ steps of Algorithm~\ref{alg:AG1} initializing at $\bar\theta$ with noise multipliers
$\dfrac{2B\sqrt{2K_2}}{\mu n_i}$. \;
\textbf{Return} $\theta^{(K_2)}$.\;
\end{algorithm}

\subsection{Algorithm 2 (\FedHybrid: $K_1$ local/$K_2$ server gradient descent)}

The second algorithm which we call \FedHybrid, adopts a two-stage design that combines private local estimation with server-side refinement using \FedSGD. In the first stage, clients independently perform $K_{1i}$ steps (for $i=1, \ldots, m$) of private gradient descent on its local data, starting from a common initialization (e.g., $\theta^{(0)}_i=0$), to obtain a privatized local estimator $\theta_i^{(K_{1i})}$, which is then sent to the server. In the second stage, the server aggregates the local estimators by weighted averaging to form an initial global estimator $
    \Bar{\theta}=\frac{1}{N}\sum_{i=1}^m n_i \theta_i^{(K_{1i})},$ and subsequently applies $K_2$ steps of server-side private gradient descent following \FedSGD. The method is displayed in Algorithm \ref{alg:AG2}.

The MSE bound for this estimator can then be derived by combining the MSE from Equation \eqref{eq:alg1-gd-mse} with suitable bounds on the initialization $\Bar{\theta}$. We state the result on the initialization $\bar{\theta}$ first.  This MSE bound is the bound on \FedAvg~ \citep{mcmahan2017communication} with just one round of communication, i.e., $R=1$, and is of independent interest to understand tradeoffs between accuracy, and costs of federation and privacy.

\begin{theorem}\label{thm:one-step-fedavg} (\FedAvg~ with $R=1$) Suppose Assumptions 1-4 above hold. Let $\Bar{\theta}$ be as defined above. Then there exists a numerical constant $C>0$ such that
\begin{align*}
    \mathbb{E}\left[\left\|\bar{\theta}-\theta_0\right\|^2\right]
    \le &~
    \frac{16}{\tau_1^2}
    \left(
    \frac{{\rm trace}(\Sigma)}{N}
    +\frac{ 2mC_B^2K_{\max}d(d\vee\log N)}{\mu^2N^2}\right)\\
    &~+
    \frac{C}{N^2}
    \left(m {\rm trace}(\Sigma)
    +
    \sum_{i=1}^m
    \frac{C_B^2K_{1i}d(d\vee \log N)}{\mu^2n_i}
    \right)^2
\end{align*}
where $B=C_B\sqrt{d\vee \log N}$, $K_{1i}\ge C\log(dN)/\log(3\tau_2/(3\tau_2-\tau_1))$ for constants $C,C_B>0$ and $K_{\max}=\max\{K_{1i}:1\le i\le m\}$.
\end{theorem}

\begin{remark}[Cost of federation and privacy] Opting for the case of fixed dimensions for cleaner presentation, the above Theorem can be compared with the error rate of a centralized non-private M estimator, which is $O(\frac{1}{N})$, to assess the costs of federation and privacy. The dominant term for the cost of federation is the third term, which scales as $O(\frac{m^2}{N^2})$. Therefore, the cost of federation is smaller than the accuracy of a centralized estimator as long as $m=O(N^{1/2})$, which has been observed in several works in the literature \cite{zhang2013communication,huang2019distributed}. The dominant term for the cost of privacy is the rightmost term, which amounts to $O(\frac{K^2 m^4}{\mu^4 N^4})$. For $\mu=O(1)$, this term becomes negligible compared to the centralized non-private estimation error provided $m=O(N^{3/4})$. 
\end{remark}

With the above initialization we can now use \eqref{eq:alg1-gd-mse} to derive the following MSE on $\hat{\theta}_{({\rm AG2})}$, the output of \FedHybrid~ in Algorithm~\ref{alg:AG2}.

\begin{lemma}
\label{lem:AG2muGDP}
    (Privacy guarantee for \FedHybrid) With the scaling of the Gaussian noise at the $i$th client in stage 1 set at $a_i=\frac{2B\eta_1\sqrt{2K_{1i}}}{\mu n_i}$ and in stage 2 set at $\frac{2B\sqrt{2K_2}}{\mu n_i}$, each client is $\mu$-fed-GDP in the sense of Definition~\ref{def:fed-gdp}, and the full Algorithm is $\frac{\mu}{\sqrt{m}}$-GDP towards a third-party in the sense of Definition~\ref{def:cent-gdp}.
\end{lemma}

\begin{theorem}\label{cor:dp-gd-init-fed-avg}(Error bound for \FedHybrid) Suppose assumptions 1-4 above hold. Let $\hat{\theta}_{\rm AG2}$ be the output of Algorithm~\ref{alg:AG2} with norm bound $B=C_B\sqrt{d\vee \log N}$ for a constant $C_B^2>0$, $K_{1i}$ iterations in Stage 1 and $K_2$ iterations in Stage 2. For $K_{1i}\ge 4\log(n_i)/\log(3\tau_2/(3\tau_2-\tau_1))$ and $K_2= 2\log(m)/\log(3\tau_2/(3\tau_2-\tau_1))$ the mean squared error of estimating $\theta_0$ by $\hat{\theta}_{\rm AG2}$  satisfies
\[
\mathbb{E}\!\left[\left\|\hat{\theta}_{({\rm AG2})}- \theta_0 \right\|^2\right]
    \leq 
    \frac{17}{\tau_1^2}
    \left(
 \sum_{i=1}^m
 \left(
    \frac{{\rm trace}(\Sigma)}{n_i}
    +\frac{2C_B^2d(d\vee \log N)\log(m)}{\mu^2 n_i^2\log(3\tau_2/(3\tau_2-\tau_1))}\right)^{-1}
 \right)^{-1}
\]
where $\Sigma:=\Var(\nabla \rho(x_1^{(1)},\theta_0))$.
\end{theorem}

To compare with the analogous result for Algorithm~\ref{alg:AG1}, we state the simplified upper bound in the case where sample sizes are comparable.

\begin{remark}[Comparable sample sizes] When $n_{\max}/n_{\min}\le C$ for some constant $C>0$, the above MSE bound reduces to:
\begin{equation}\label{eq:mse-ag2-equal-ni}
\mathbb{E}\!\left[\left\|\hat{\theta}_{({\rm AG2})}- \theta_0 \right\|^2\right]
    \leq 
    \frac{C}{\tau_1^2}
    \left(
    \frac{{\rm trace}(\Sigma)}{N}
    +
    \frac{C_B^2md(d\vee \log N)\log(m)}{\mu^2N^2}
    \right)
\end{equation}
for a constant $C>0$, where $N=\displaystyle\sum_{i=1}^m n_i$ and we use the weights $\{w_i=n_i/N:1\le i\le m\}$ when computing the weighted average of gradients $\{\tilde{g}_i^{(k)}:1\le i\le m\}$ for $k=1,\dots,K$.
\end{remark}

Comparing Equations~\eqref{eq:mse-ag1-equal-ni} and \eqref{eq:mse-ag2-equal-ni} shows that in the comparable sample size regime, the MSEs of $\hat{\theta}_{\rm (AG1)}$ and $\hat{\theta}_{\rm (AG2)}$ coincide in the ``non-private'' term of ${\rm trace}(\Sigma)/N$ but differ in the ``private'' term. This difference stems from the dependence on $K$ and $K_2$. A quick inspection of the required number of iterations (and consequently rounds of server-client communications) reveals that $K=\Omega(\log(N))$ for Algorithm~\ref{alg:AG1} while $K_2=\Omega(\log(m))$ for Algorithm~\ref{alg:AG2}. It is reasonable to assume that the number of servers, $m$, grows at a much slower rate than the total sample size $N$. This implies that the MSE of \FedHybrid~ $\hat{\theta}_{\rm (AG2)}$ is lower than that of \FedSGD~ $\hat{\theta}_{\rm (AG1)}$ while also improving the server-client communication cost from $O(d\log N)$ to $O(d\log m)$. This result theoretically shows the benefit of our novel \FedHybrid~method over the alternative method \FedSGD.

\begin{algorithm}
\caption{\FedAvg}
\label{alg:AG3}
\DontPrintSemicolon
\KwIn{Data $\{X^{(i)}\}_{i=1}^m$ with local client sample sizes $\{n_i\}$; communication rounds $R$; local iteration $K$; privacy parameter $\mu>0$; loss function $\rho(\cdot,\theta)$; clipping bound $B$.}
\KwOut{$\theta^{(R)}$}
\BlankLine
\textbf{Set} $g(x,\theta)=\nabla\rho(x,\theta)$. Initialize $\theta^{(0)}=0$.\;
\textbf{Noise scale:} $\sigma_i=\dfrac{2B\eta\sqrt{RK}}{\mu n_i}$. \;

\BlankLine
\For(\textbf{round}){$r=1,\dots,R$}{
Server broadcasts $\theta^{(r-1)}$.\;
\For(\textbf{}){$\textit{each client } i=1,\dots,m$}{
$\theta_i^{(0)}=\theta^{(r-1)}$.\;
\For(\textbf{local}){$t=0,\dots,K-1$}{
$\theta_i^{(t+1)}=\theta_i^{(t)}-\dfrac{\eta}{n_i}
\displaystyle
\sum_{j=1}^{n_i} g(x_j^{(i)},\theta_i^{(t)})+\sigma_i Z_{it}$,\quad
$Z_{it}\sim\mathcal{N}(0,I_d)$.\;
}
Send $\theta_i^{(K)}$ to server.\;
}
$\theta^{(r)}=\displaystyle\sum_{i=1}^m \dfrac{n_i}{N}\theta_i^{(K)}$, \quad $N=\displaystyle\sum_{i=1}^m n_i$.\;
}
\textbf{Return} $\theta^{(R)}$.\;
\end{algorithm}

\subsection{Algorithm 3: $R$ round \FedAvg}

Our third estimator involves federated averaging of client updates. The complete algorithm is described in Algorithm~\ref{alg:AG3}. 
The \FedAvg~\citep{mcmahan2017communication,wei2020federated} alternates between local private updates and server-side aggregation over multiple communication rounds. At each round, the server broadcasts the current global model to all clients. Each client then performs local gradient descent iterations, adding Gaussian noise at each iteration to ensure client-level privacy. After completing the local updates, the clients send their locally trained models to the server. Then the server aggregates the received local models using weighted average to form an updated global model, which is then broadcast in the next round. 

Unlike \FedSGD, which aggregates privatized gradients at every iteration, \FedAvg~ aggregates the parameter estimates at communication rounds. This aggregation strategy can be beneficial in terms of reducing communication cost without losing much accuracy when local sample sizes are sufficiently large, as the local models can be relatively accurate. However, as we show below, the benefits disappear if local sample sizes are small and we have numerous clients.

\begin{lemma}
    \label{lem:AG3muGDP}
    (\FedAvg~ DP guarantee) With the scaling of noise mentioned in Algorithm \ref{alg:AG3}, the clients are $\mu$-fed-GDP in the sense of Definition~\ref{def:fed-gdp}. With the Assumptions 2 and 4 on $\tau_1$ strong convexity and $\tau_2$ smoothness, and setting learning rate $\eta$ such that $C_{\eta} := \bigl(1-2\eta\tau_1+\eta^2\tau_2^2\bigr) <1$, the algorithm is $\frac{\mu}{\sqrt{m}}$-GDP towards a third-party in the sense of Definition~\ref{def:cent-gdp}.
\end{lemma}

\begin{theorem}\label{thm:ag3_mse}(Error bound for \FedAvg)
Suppose Assumptions~1--4 hold. For the estimator $\theta^{(R)}$ obtained after $R$ rounds of Algorithm \ref{alg:AG3},
\begin{align*}
\mathbb{E}\!\left[\left\|\hat{\theta}_{\mathrm{(AG3)}}-\theta_0\right\|^2\right]
\;\le\;&~
\frac{16}{\tau_1^2}
\left(
\frac{{\rm trace}(\Sigma)}{N}
+\frac{2mCRKd(d\vee \log N)}{\mu^2N^2}
\right)\\
&~+
\frac{C}{N^2}
\left(
m\,{\rm trace}(\Sigma)
+
\sum_{i=1}^m
\frac{RKd(d\vee \log N)}{\mu^2n_i}
\right)^2.
\end{align*}
where $B=C_B\sqrt{d\vee \log N}$, $K\ge C\log(dN)/\log(3\tau_2/(3\tau_2-\tau_1))$ for constants $C,C_B>0$.
\end{theorem}
This result is a $R$ communication round version of the result in Theorem \ref{thm:one-step-fedavg} and is proved through recursion starting from that result. These results also reveal a similar tradeoff between accuracy and costs of federation and privacy protection. We note that the result is intended for a finite number of rounds of communication, and does not show the advantage of higher $R$.

\begin{remark}
    (Bias of \FedAvg) We have the following expression for the bias of \FedAvg~ for 1 round of communication (equal sample size case), \begin{align*}\label{eq:bias-fed-avg-1step}
    \left\Vert
    \EE(\Bar{\theta}-\theta_0)
    \right\Vert
    \le&~
    O
    \left(
    \frac{m{\rm trace}(\Sigma)}{N}
    +
    \frac{d(d\vee \log N)m^2}{\mu^2N^2}
    \right),
\end{align*}
while the variance is $O
    \left(
    \frac{{\rm trace}(\Sigma)}{N}
    +
    \frac{md(d\vee \log N)RK}{\mu^2N^2}
    \right)$. Hence approximately the squared bias is $O(m^2)$ higher than that of variance. To achieve an MSE bound comparable to Theorem~\ref{th:alg1} the \FedAvg~estimator would thus require $m\le (N/C(d\vee \log N))^{1/2}$, when $\mu$ is large. This bias is the primary reason for the underperformance of \FedAvg. In the next section we present \FedNewton, as a communication-efficient approach to reducing this bias of \FedAvg.
    \label{rmk:biasfedavg}
\end{remark}

\subsection{The  \FedNewton~Algorithm}

Next, we propose a new method that adds one Newton step following the \FedAvg~estimator. The purpose of this Newton step is to reduce the bias in the \FedAvg~estimator as shown in Remark \ref{rmk:biasfedavg}. At the conclusion of the $R$ rounds of communication of the \FedAvg~estimator, the current solution is broadcasted to the clients for an additional round. In this round the clients make one (or several) updates, but instead of gradient descent, they run a Newton step which requires the Hessian matrix. The clients then send the updated model parameters back to the server, which then aggregates the parameters for the final estimator. 

\begin{algorithm}[h]
\caption{\FedNewton: One-Step Newton following \FedAvg}
\label{alg:privatenewtonfedavg}
\DontPrintSemicolon

\KwIn{Client data $\{X^{(i)}\}_{i=1}^m$ with local sample sizes $n_i$, Local step size $\eta$; 
local iterations per round $K$, Privacy parameter $\mu > 0$,loss function $\rho(\cdot,\theta)$; clipping bound $B$.}

\KwOut{Global model $\theta^{(2)}$}

\BlankLine

\textbf{Sample Splitting:} Split $X^{(i)}$ into disjoint sets of samples $X^{(i)1},X^{(i)2}$ each of size $n_i/2$.
\BlankLine

\textbf{Initialization:} Set $\theta^{(0)}=0$. \;

\textbf{Noise scale:}  Set noise scale for each client: $
    \sigma_i = \frac{4B\eta\sqrt{2K}}{\mu n_i} $.

\textbf{Federated Training:}

\textbf{DP\FedAvg:} Run \FedAvg Algorithm \ref{alg:AG3} with $R=1$ and noise scaling $\sigma_i$ on data $X^{(i)1}$.
 \BlankLine   
    
    \textbf{// Client-side local One-Step Newton}
    \BlankLine
    
    \For{each client $i = 1, \ldots, m$ \textbf{in parallel}}{
        \BlankLine
        Update Newton-step:
        \[
        \theta_i^{(New)}
    =~
    \theta^{(1)}
    -
    \left(
    \ddot{L}_i(\theta^{(1)},X^{(i)1})
    \right)^{-1}
    \dot{L}_i(\theta^{(1)},X^{(i)2})
    +
    \frac{2B^{(New)}}{\mu n_i}Z_i, \quad \text{where }B^{(new)}=\frac{4\sqrt{2}B}{\tau_1}
        \]       
        Client $i$ sends local model $\theta_i^{(New)}$ to server\;
    }
\BlankLine
    \textbf{// Server-side aggregation}
    \BlankLine
    
    Aggregate client models with sample-size weighting:
    \[
    \theta^{(2)} = \sum_{i=1}^{m}w_i\theta_i^{(New)}.
    \]

\BlankLine
\textbf{Return:} Final global model $\theta^{(2)}$.

\end{algorithm}

\begin{lemma}
    \label{lem:AG4muGDP}
    (\FedNewton~ DP guarantee) With the scaling of noise mentioned in Algorithm \ref{alg:privatenewtonfedavg}, the clients are $\mu$-fed-GDP in the sense of Definition~\ref{def:fed-gdp}. 
\end{lemma}

\begin{theorem}\label{th:newton-avg} (Error bound for \FedNewton)
For the estimator $\hat{\theta}_{\rm (AG4)}$ obtained from Algorithm~\ref{alg:privatenewtonfedavg}, we have
\begin{align}
        \mathbb{E}\left[\left\|\hat{\theta}_{\rm (AG4)}-\theta_0\right\|^2\right]
    \le&~
   \frac{C{\rm trace}(\Sigma)}{\tau_1^2 N}
    +
    \frac{Cmd(d\vee \log N)}{\mu^2N^2},
    \label{ag4rate1}
\end{align}
where $B^{(New)}=C_B\sqrt{d\vee \log N}$ for some constant $C_B>0$, if the weights $w_i$ are taken to be $w_i=n_i/N$, where $K_{\max}=\max K_{1i}$, provided there exists a constant $C>0$ such that 
\[m\le \min\{(N/Cd)^{2/3},(\mu^3Nn_{\min}^2/C(d\vee \log N)^3K_{\max}^2)^{2/3},(\mu^2Nn_{\min}^2/C(d\vee \log N)^2K_{\max}^2)^{1/2}\}.
\]
Moreover, if either i) $n_{\max}\le Cn_{\min}$, or ii) $m\le\{(1\wedge \{N\mu^2/d\})\mu^2n_{\min}^2N/C(d\vee \log N)^3K_{\max}^2\}^{1/3}$ for a constant $C>0$, then 
\begin{align}
        \mathbb{E}\left[\left\|\hat{\theta}_{\rm (AG4)}-\theta_0\right\|^2\right]
    \le&~
\frac{C}{\tau_1^2}
    \left(
 \sum_{i=1}^m
 \left(
    \frac{{\rm trace}(\Sigma)}{n_i}
    +\frac{2C_B^2d(d\vee \log N)}{\mu^2 n_i^2\log(3\tau_2/(3\tau_2-\tau_1))}\right)^{-1}
 \right)^{-1},
    \label{ag4rate}
\end{align}
where $\Sigma:=\Var(\nabla \rho(x_1^{(1)},\theta_0))$, and weights $\hat{w}_i\propto \left(
    \frac{{\rm trace}(\Sigma)}{n_i}
    +\frac{C_B^2d^2K}{\mu^2 n_i^2}\right)^{-1}$ with $\displaystyle\sum_{i=1}^m\hat{w}_i=1$.
\end{theorem}

The above theorem shows that when the number of clients $m$ is sufficiently small, \FedNewton~\\achieves the same upper bound on MSE, as done by \FedSGD~and \FedHybrid. To  better understand the benefit of \FedNewton~we focus on the case of comparable sample sizes, i.e., when $n_{\max}\le Cn_{\min}$ for some $C>0$. 

\begin{remark} 
    (Comparable sample sizes)
    When the sample sizes are comparable the restriction on the number of clients reduces to
    \[
    m\le 
    \min\{
    (N/C(d\vee \log N))^{2/3},
    (N\mu/(d\vee \log N))^{6/7}/CK_{\max}^{4/7},
    (\mu/(d\vee \log N))^{1/2}N^{3/4}/CK_{\max}^{1/2}
    \}.
    \]
\end{remark}

We note from the above theorem that \FedNewton~does not suffer from the same limitation as \FedAvg~did in terms of the growth rate of the number of clients and the result in Equation \ref{ag4rate1} of Theorem \ref{th:newton-avg} holds as long as $m$ is bounded by the above quantity. In particular when $\mu$ is sufficiently large, the restriction becomes $m\le (N/C(d\vee \log N))^{2/3}$, which is weaker than the requirement of $m\le (N/C(d\vee \log N))^{1/2}$ in Theorem~\ref{thm:ag3_mse}.

\begin{remark}
    (Dependence on $K$)
    A striking feature of the \FedNewton~estimator is that it corrects for the bias in the initial estimator obtained by 1 round \FedAvg, and in doing so, removes the privacy cost in terms of the number of iterations $K$. This is possible of course if $m$ is sufficiently small, and we note that the upper bound on $m$ becomes tighter if we were to allow larger values of $K_{\max}$. This conclusion is consistent with the notion of unnecessarily large $K_{\max}$ being detrimental to the accuracy of our estimator.
    \label{rmk:fednewton-depend-k}
\end{remark}

\subsection{Lower bound}
Next, we provide a minimax lower bound on the error rate of any $\mu$-Fed-GDP private federated procedure with $K$ rounds of client-server iterations.
\begin{theorem}\label{th:low-bd} (minimax lower bound) For $0<\mu<1$, the class of federated private estimators with $K$ rounds of client-server iterations satisfies
    \[
    \inf_{\hat{\theta}\in \mathcal{T}(\{n_1,\dots,n_m\},m,K,\mu)}
    \sup_{\theta,P_{\theta}\in \mathcal{P}}
    \EE\left\Vert \hat{\theta}-\theta \right\Vert^2
    \ge \frac{d}{
    \sum_{i=1}^m
    C_1n_i
    \wedge 
    C_2(n_i^2\mu^2\log(\mu^{-1})/d)
    )}
    \]
where $\mathcal{T}(\{n_1,\dots,n_m\},m,K,\mu)$ denotes the set of all estimators obtained from $n_i$ samples at the $i$-th client for $1\le i\le m$, where each client uses a fresh batch of samples for each of $K$ client-server rounds of communication, while satisfying $\mu$-fed-GDP in the sense of Definition~\ref{def:fed-gdp}.
\end{theorem}
The minimax lower bound brings out information-theoretic intuition on the cost of privacy during federated learning. Instead of having $N=\sum_i n_i$ in the lower bound, as we would have in centralized non-private estimation, the sample size contribution of each client is modulated by the effect of privacy. The contribution or effective sample size for each client is the minimum of the actual sample size and a quantity that depends on privacy parameters. The construction of the lower bound is via Van Trees inequality and characterizing the shrinkage in information due to federated privacy requirements. Similar techniques have been used earlier in nonparametric federated learning problems: see \cite{cai2024optimal,auddy2024minimax,xue2024optimal}.

The next two remarks compare the lower bound with the upper bounds obtained via our algorithms. 

\begin{remark}
   (Near-optimality of our methods) 
    As can be seen from the upper bounds in Theorems \ref{th:alg1}, \ref{cor:dp-gd-init-fed-avg}, and \ref{th:newton-avg}, \FedSGD, \FedHybrid, when $d\ge C\log N$, the upper bounds obtained by our algorithms match the lower bound in Theorem~\ref{th:low-bd} upto a factor of $\log(\mu)$ and $K$, $R$, which in our theoretical analyses are also of $\log(N)$ order. Thus our MSE bounds are near optimal, with the only multiplicative discrepancy being of a factor logarithmic in sample size. We note that such gaps in optimality are ubiquitous in the $(\varepsilon,\delta)$ differential privacy literature: see, e.g., \cite{cai2021cost,cai2024optimal}. 

    As pointed out in Remark~\ref{rmk:fednewton-depend-k}, the final MSE of the \FedNewton~estimator does not depend on $K$ and $R$. \FedNewton~thus matches the lower bound, with a sub-optimality of only the $\log(\mu)$ factor. We remind the reader once again of the important caveat that this is possible only when the number of clients $m$ is sufficiently small.
\end{remark}

\begin{remark}
    (Sample splitting) 
    Beyond what is pointed out in the above remark, we note that for a more convenient proof, the lower bound explicitly uses sample splitting at every communication round, while the upper bounds have no such requirement. Using a similar sample splitting in the algorithms would lead to an MSE which is worse by a factor of $R$ in each case, where $R$ refers to the number of client-server communications.
\end{remark}

The three methods, \FedSGD, \FedHybrid, and \FedNewton~now primarily differ in terms of communication cost, which we discuss in the next section.

\subsection{Communication cost of the methods}

The communication costs associated with the three algorithms are summarized in Table \ref{tab:communication}.

\begin{table}[h]
\centering
\begin{tabular}{l|c|c|c}
\hline
Algorithm & Effective Rounds &  Total Communication & Error Bound \\
\hline
\FedSGD
& $K=\Omega(\log Nd)$ & $O(mKd)$
& $O\left(\dfrac{d}{N}+\dfrac{md^2\log Nd}{\mu^2N^2}\right)$ \\

\FedHybrid
& $K_2=\Omega(\log m)$ &  $O(mK_2d)$
& $O\left(\dfrac{d}{N}+\dfrac{md^2\log m}{\mu^2N^2}\right)$ \\

\FedAvg
& $R=O(1)$ &  $O(mRd)$  
& $O\left(\dfrac{d}{N} + \dfrac{m^2 d^2}{N^2}+\dfrac{md^2RK}{\mu^2N^2} \right)$ \\

\FedNewton
& $2$ &  $O(md)$  
& 
$O\left(\dfrac{d}{N}+\dfrac{md^2}{\mu^2N^2}\right)$ \\
\hline
\end{tabular}
\caption{Comparison of total communication cost and mean squared error bounds for the four proposed differentially private federated algorithms.}
\label{tab:communication}
\end{table}

We note that \FedSGD~has high communication cost of $O(mKd)$ due to the need of collecting gradients from the clients at each step. The hybrid method (\FedHybrid), on the other hand, runs the first $K_1$ iterations locally at the clients, which does not require any communication. In stage 2, the algorithm runs $K_2 \ll K$ iterations, resulting in a communication cost of $O(mK_2d)$. We remind the readers that in our setup of empirical risk minimization with scalar valued response, both the parameter  and the gradient are of dimension $d$. This comparison of communication cost continues to hold even if we consider non-convex risk minimization such as training deep neural networks.

\subsection{\FedNewton~adaptation for training neural networks} Given the importance of federated learning for training deep neural networks (DNN), we briefly discuss practical adaptation of the proposed \FedNewton~to DNNs. Note the \FedNewton~algorithm requires computation of the Hessian matrix locally (at the clients), and in particular, does not require communication of the Hessian. Following recent works in federated learning for DNN and language model training \citep{jhunjhunwala2024fedfisher}, we approximate the Hessian with the sample average of the inner product of the derivative (score function) with its transpose. However, unlike \cite{jhunjhunwala2024fedfisher}, we do not further approximate this matrix by its diagonal, since we do not need to communicate the Hessian to the central server, and the Newton iteration is performed locally at the client. Our method thus achieves a better communication-accuracy tradeoff.

Further, in our data examples on MNIST and CIFAR10 data with convolution neural networks, we found that the full (approximated) Hessian matrix performs better than when only the diagonal is kept as an approximation. For numerical stability, we add a small ridge parameter and a damping parameter to control the step size. The entire method is described in Algorithm \ref{alg:fednewton_nn}. In real data examples, we further repeat the entire algorithm (with both stages) for a few iterations instead of finishing all stage 1 runs first and then doing stage 2 iterations. Note that since the loss function of deep neural networks is nonconvex and our Assumptions 1-4 do not hold, the theoretical bounds derived earlier for the four methods do not hold for the case of neural networks.  

\begin{algorithm}[h]
\caption{\FedNewton-NN: Neural network adaptation}
\label{alg:fednewton_nn}
\DontPrintSemicolon

\KwIn{Client datasets $\{X^{(i)}\}_{i=1}^m$ with sizes $n_i$; step size $\eta$; rounds $R$; local steps $K$; privacy $\mu$; loss $\rho(\cdot,\theta)$; clipping bound $B$; ridge parameter $\lambda$; damping parameter $\alpha$.}

\KwOut{Global model $\theta^{(R+1)}$}

\BlankLine

\textbf{Initialization:} $\theta^{(0)}=0$, \quad 
$\sigma_i=\frac{2B\eta\sqrt{2RK}}{\mu n_i}$.

\BlankLine

\textbf{Stage 1 (DP-\FedAvg):} Run DP-\FedAvg~ for $R$ rounds with noise $\sigma_i$ to obtain $\theta^{(R)}$.

\BlankLine

\textbf{Stage 2 (Approximate Newton refinement):}

\For{each client $i$ \textbf{in parallel}}{
    Compute
    \[
    g_i=\frac{1}{n_i}\sum_{x\in X^{(i)}} \nabla\rho(x,\theta^{(R)}), 
    \quad
   h_i = \left(
\frac{1}{n_i}\sum_{x\in X^{(i)}} 
\nabla\rho(x,\theta^{(R)})\,\nabla\rho(x,\theta^{(R)})^\top
\right)
    \]
    
    Clip $\bar g_i = g_i \cdot \min\{1, B/\|g_i\|_2\}$ \;
    
    Update
    \[
    \theta_i^{(\mathrm{New})}
    =
    \theta^{(R)}
    -
    \alpha (h_i+\lambda)^{-1}\bar g_i
    +
    \frac{2B^{(\mathrm{New})}}{\mu n_i} Z_i,
    \quad
    B^{(\mathrm{New})}=\frac{2\sqrt{2}B}{\tau_1}
    \]
}

\BlankLine

\textbf{Aggregation:}
\[
\theta^{(R+1)}=\sum_{i=1}^m \frac{n_i}{N}\theta_i^{(\mathrm{New})}.
\]

\BlankLine

\textbf{Return:} $\theta^{(R+1)}$.

\end{algorithm}

\section{Simulation Studies}

We conduct multiple simulation studies to evaluate the finite-sample performance of the considered methods and validate our theoretical results in finite samples under Poisson and logistic regression models. The generalized linear model (GLM) family takes the form $  g\!\left(\mathbb{E}\!\left[Y \mid X\right]\right) = X\boldsymbol{\beta},$
where $g(\cdot)$ denotes the canonical link function, $X$ is the design matrix of covariates and intercept, and $\boldsymbol{\beta}$ is the vector of model coefficients. Both the logistic regression and Poisson regression are special cases of GLM. We estimate the parameter $\boldsymbol{\beta}$ through the maximum likelihood estimator, which is a type of M estimator with the negative log likelihood of the model serving the role of the loss function. Since we know the true parameter vector $\boldsymbol{\beta}$ in simulation, our goal is to evaluate the performance of the methods in terms of empirical MSE for estimating this parameter vector. In all settings, a total sample of size $N$ is first generated and then partitioned across $m$ clients. We vary the number of clients and the local sample sizes in various scenarios, and apply the four methods, \FedSGD, \FedHybrid, \FedAvg, and \FedNewton, to estimate the model parameters.

\begin{figure}[h]
    \centering
    \begin{subfigure}{0.47\textwidth}
        \centering
        \includegraphics[width=\textwidth]{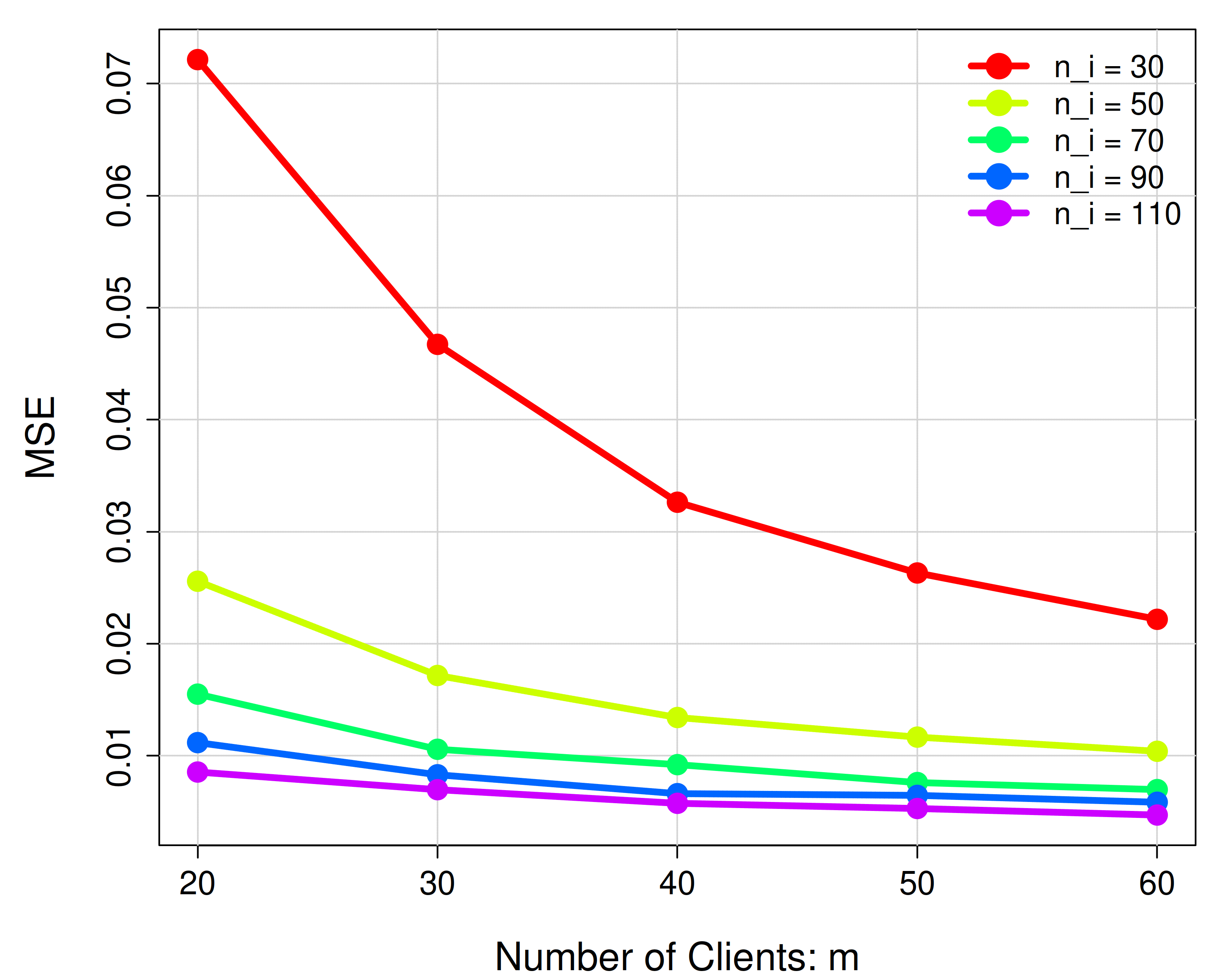}
        \caption{\FedSGD, same $n_i$}
    \end{subfigure}%
    \hfill
    \begin{subfigure}{0.47\textwidth}
        \centering
        \includegraphics[width=\textwidth]{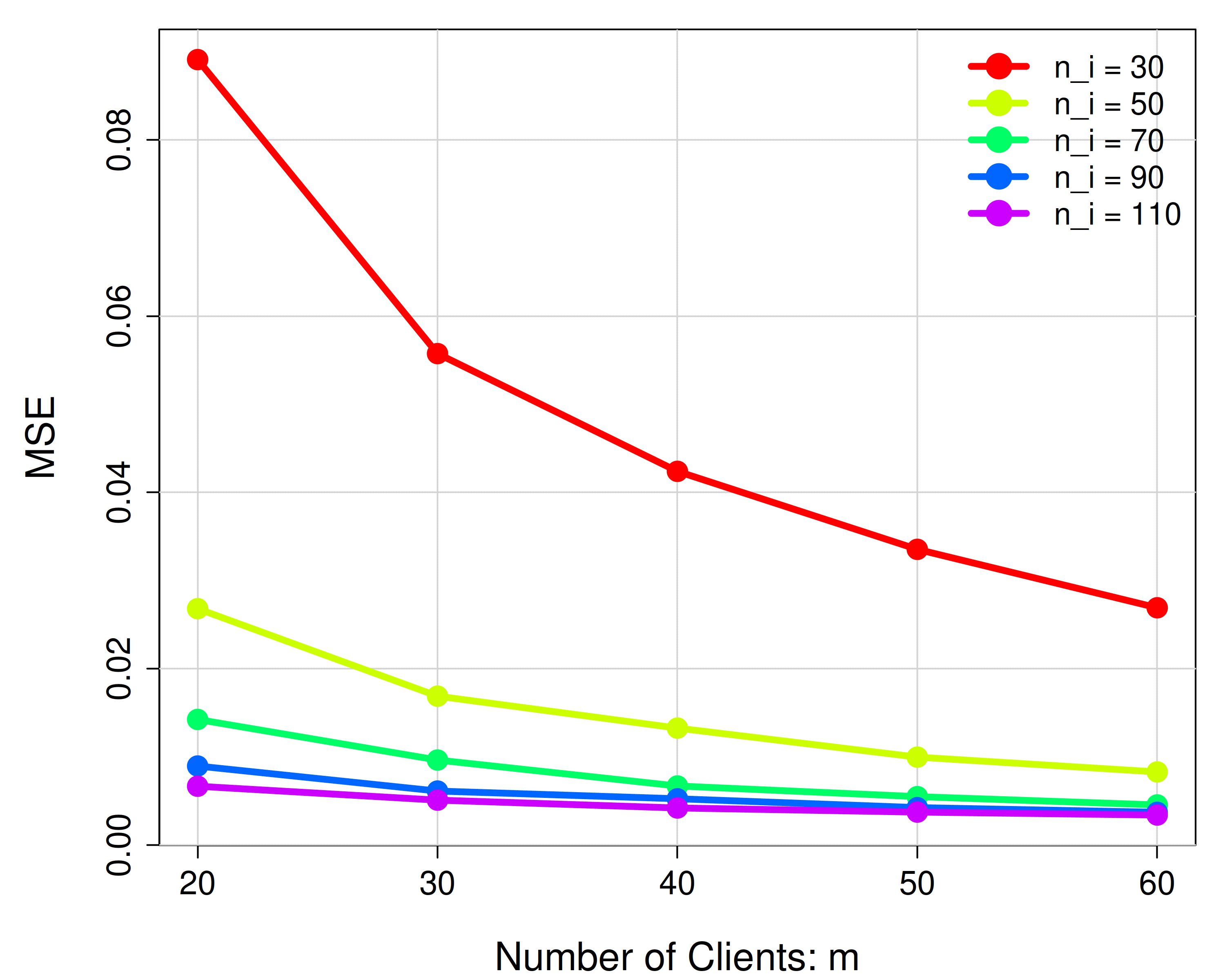}
        \caption{\FedHybrid, same $n_i$}
    \end{subfigure}%

    \vspace{0.6cm}

    \caption{Empirical MSE results of DP version of \FedSGD\ and \FedHybrid\ under equal sample sizes across clients in Logistic Regression.}

\label{fig:log_ag12_equal_ni}
\end{figure}

Each simulation study is repeated 100 times. For each repetition, when implementing \FedSGD, the number of iterations is set to $K = 50$ with step size $\eta = 0.5$. For \FedHybrid, Stage One consists of $K_1 = 30$ local iterations for each client with step size $\eta_1 = 0.5$, while Stage Two uses $K_2 = 20$ iterations with step size $\eta_2 = 0.5$. For \FedAvg, the number of local iterations is $K_3 = 50$, and the total number of communication rounds is set to $R = 2$, leading to an effective 100 iterations. For \FedNewton, the first stage uses \FedAvg~ with $K_4 = 50$ local iterations and $R = 1$, followed by just one private Newton step. In all cases, the scale of Gaussian noise to be added to the gradient iterations is calibrated accordingly. We will generally compare the performance of \FedHybrid~against \FedSGD, and \FedNewton~against \FedAvg and avoid comparison across those pairs due to differences in communication cost. The reported results are obtained by averaging the empirical squared estimation error over the 100 repetitions.

\subsection{General Federated Simulation Setup}

A pooled dataset of size \( N \) is simulated from a GLM with the canonical link function. Each data point in this pooled dataset includes a response variable \( y \) and its associated covariate vector \( \b x \).  Let \(\boldsymbol{\beta}_{\text{true}} \in \mathbb{R}^{p}\) denote the true parameter vector, where the first component corresponds to an intercept term. We first generate a covariate matrix \(\mathbf{C} \in \mathbb{R}^{N \times (p-1)}\), whose entries are independently sampled from a normal distribution with mean zero and variance \(\sigma_c^2\). An intercept column of ones is then added to this covariate matrix to obtain $\mathbf{X}$. After the pooled data are generated, they are split into \(m\) clients according to the local sample sizes \( \{n_i\}_{i=1}^{m} \). Therefore, client \( i \) receives a design matrix \( X^{(i)} \in \mathbb{R}^{n_i \times p} \) and a response vector \( y^{(i)} \in \mathbb{R}^{n_i} \). 

To control the global sensitivity, gradient clipping is applied at the initial value $\beta^{(0)} = 0$ for each client. For each observation $j =1, ..., n_i$ in client \( i = 1, \ldots, m \), we first compute its gradient norm as,
\(
\left\| \nabla \rho(x_j^{(i)}, y_j^{(i)}; \beta^{(0)}) \right\|_2.
\)
The local clipping threshold $B_i$ is then defined as the 90th percentile of the local sample gradient norms. Then the global sensitivity bound is defined as  
\(
B = \max_{1 \le i \le m} B_i.
\)
Any gradient vector \( g \in \mathbb{R}^p \) is then clipped using this global bound.
This step allows that all gradients are uniformly bounded by $B$, which is necessary to control the sensitivity before applying the Gaussian mechanism. We note that this data driven global clipping bound uses client data and therefore will leak information, however, following standard practice in the literature \citep{noble2022differentially}, we ignore this minor privacy leakage in our theoretical privacy guarantees. As an alternative, one can use a part of the privacy budget, say, $\mu/10$ to set $B$, and the rest for the estimation procedure. This would leave our theoretical guarantees and algorithmic properties unchanged.

\begin{figure}[h]
    \centering

    \begin{subfigure}[t]{0.47\textwidth}
        \centering
        \includegraphics[width=\textwidth]{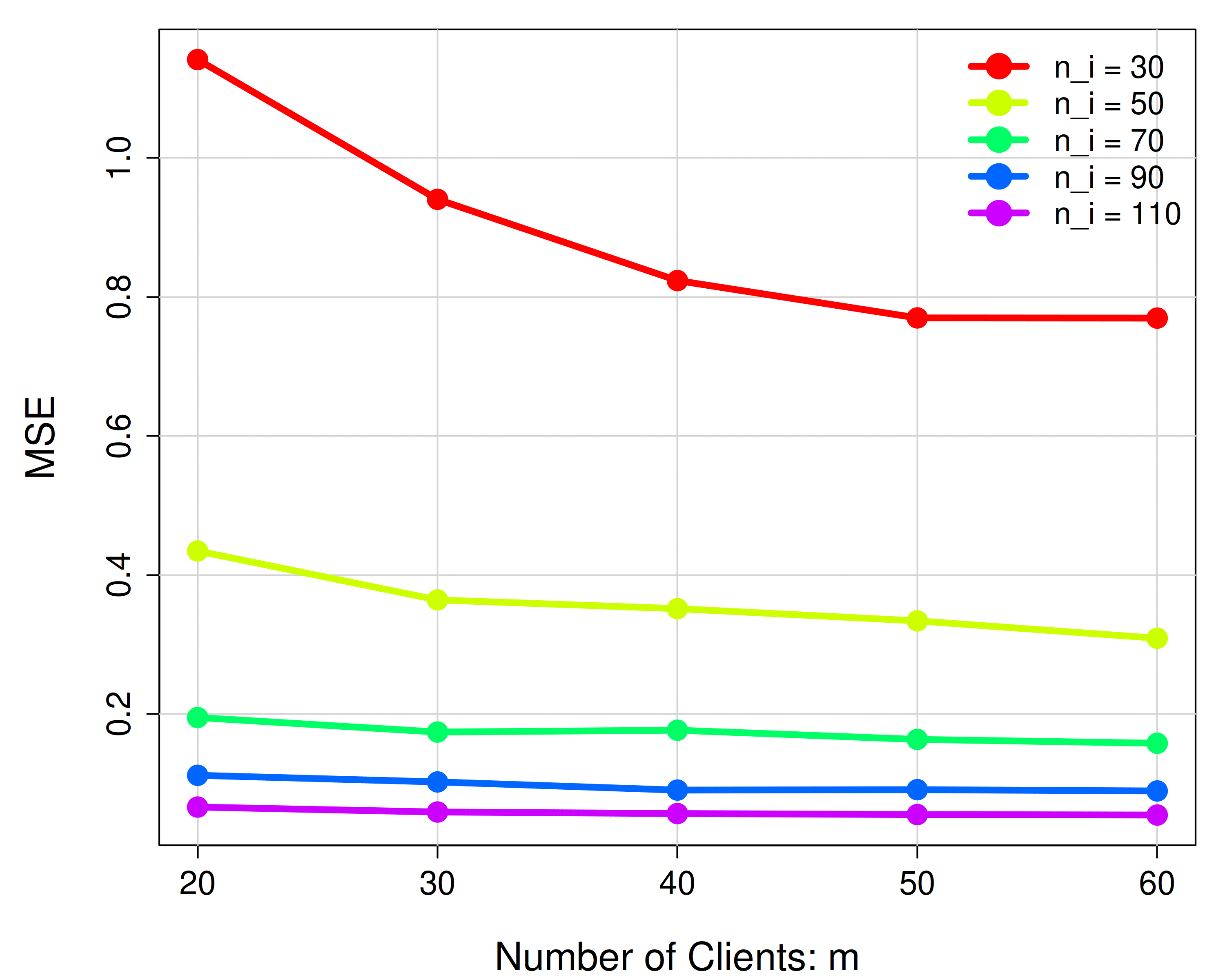}
        \caption{\FedAvg, same $n_i$}
    \end{subfigure}%
 \hfill
    \begin{subfigure}[t]{0.47\textwidth}
        \centering
        \includegraphics[width=\textwidth]{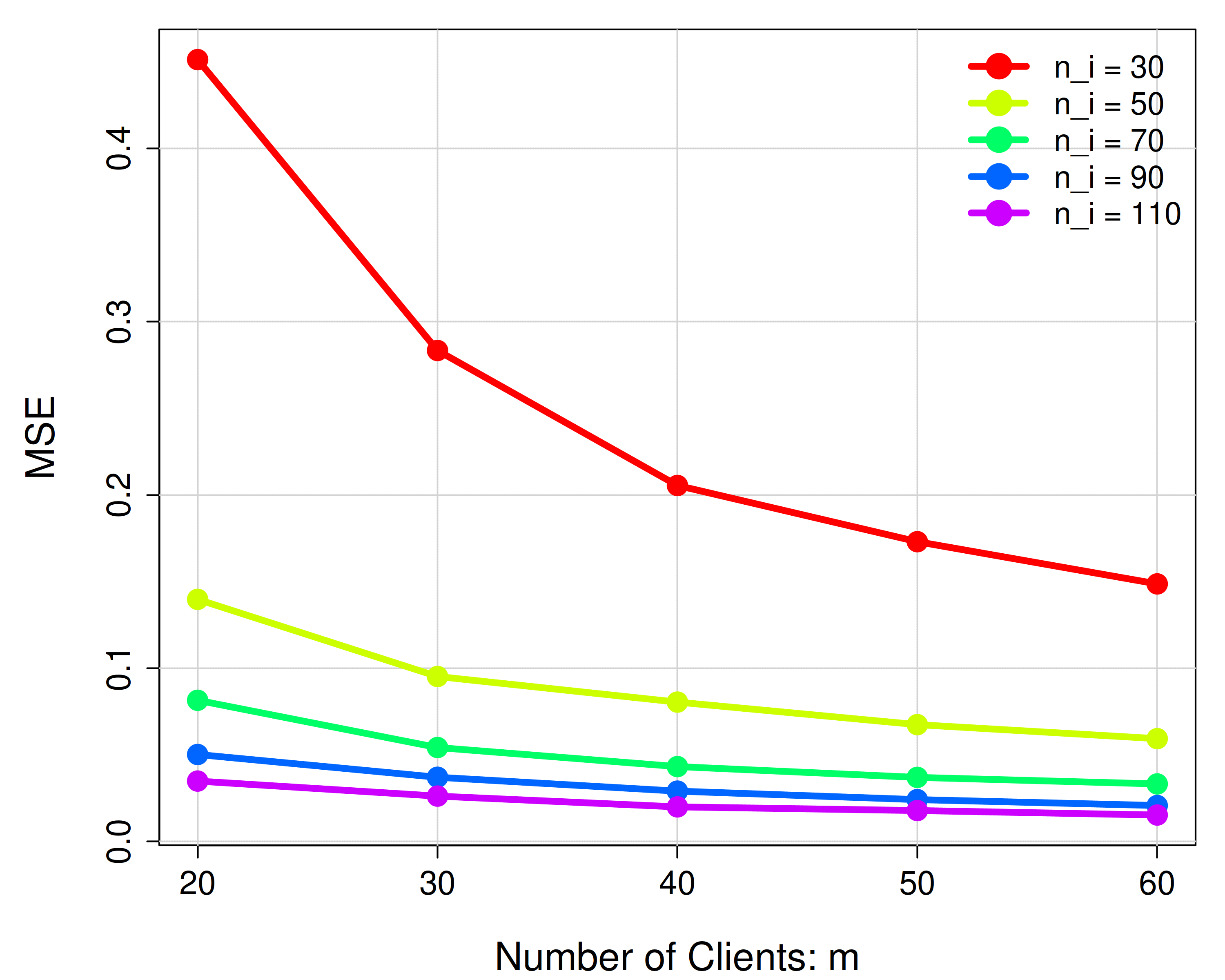}
        \caption{\FedNewton, same $n_i$}
    \end{subfigure}

    \caption{Empirical MSE results of \FedAvg\ and \FedNewton\ under equal client sample sizes in Logistic Regression.}
    \label{fig:log_ag34_equal_ni}
\end{figure}

In the simulation experiments, different numbers of clients and local sample sizes are considered. In the first set of simulations for both logistic and Poisson GLM, we vary number of clients
\(
m \in \{20, 30, 40, 50, 60\}.
\)
For the local sample sizes, two cases are examined. The first case assumes that all clients have the same local sample size of $n$ with \(n \in \{30, 50, 70, 90, 110\}.\)
The second case allows variation across clients in the local sample sizes. In this case, the sample sizes for clients \(n_i\) are independently sampled from a discrete uniform distribution. Specifically, the following ranges are considered:
\(
[10, 90],\ [20, 80],\ [30, 70],\ [60, 160],\ [80, 140],\ [100, 120].
\)

In the second set of simulations to better compare the four methods, we consider a setting of $m$ increasing from 60 to 140, under three different local sample size regimes: equal local sample sizes with \(n_i = 400\), local sample sizes generated from a discrete uniform distribution on \([100,700]\), and local sample sizes generated from a log-normal distribution with parameters \((5.5,1)\). These settings are used to compare the estimation performance of \FedSGD, \FedHybrid, \FedAvg, and \FedNewton\ under different client-size distributions with mean local sample size around 400 for both logistic and Poisson GLMs.

The third and fourth set of simulations evaluate the tradeoffs we observed theoretically in the previous sections in finite samples. For this we first vary the number of iterations to evaluate the behavior of the methods in the third set while we vary the number of clients without changing the total sample size across all clients in the fourth simulation.

\subsection{Logistic Regression performance of federated methods}
\label{subsec:sim_logistics_regression}

\begin{figure}  
 \centering
    \begin{subfigure}{0.33\textwidth}   \includegraphics[width=\textwidth]{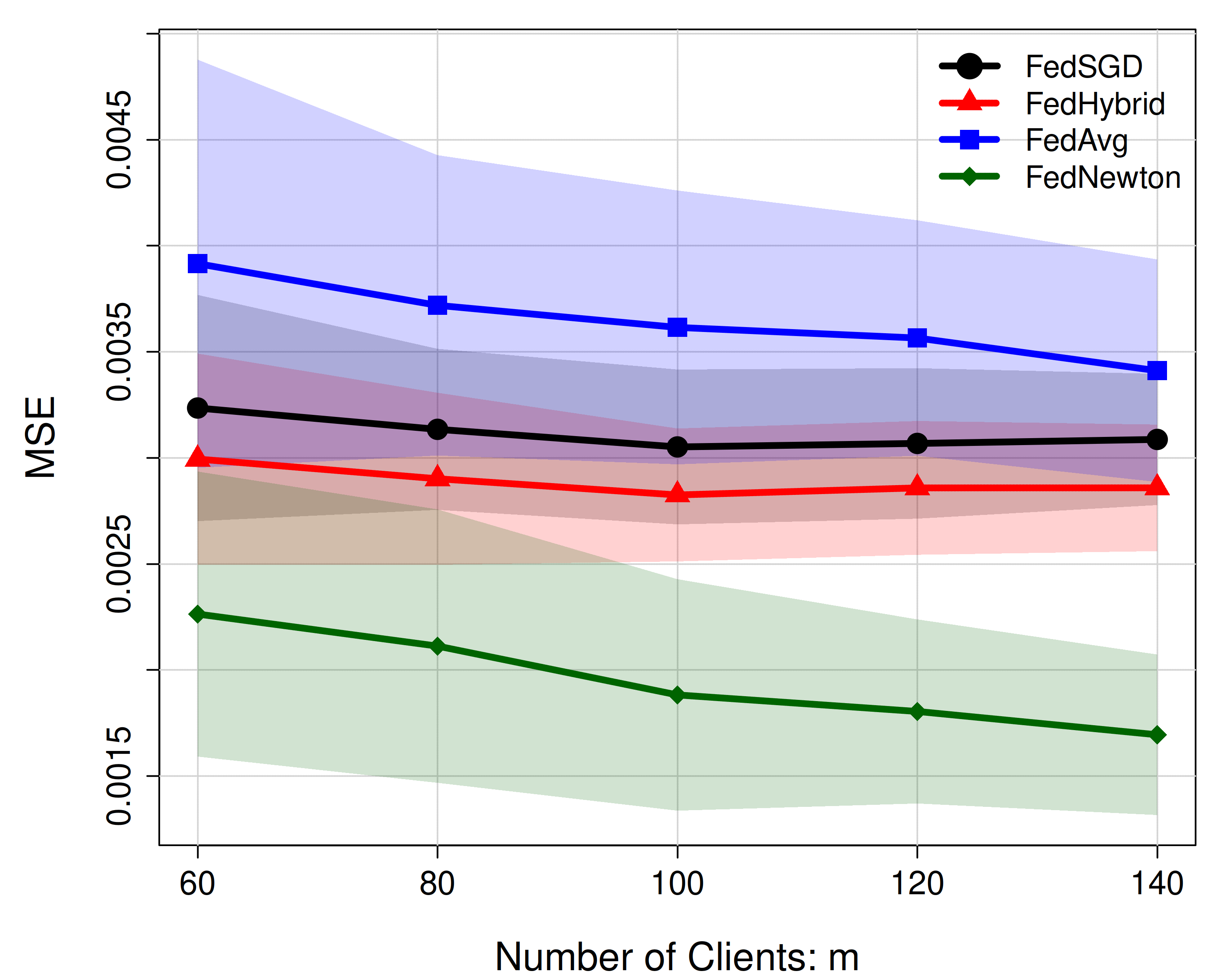}
        \caption{Random sample size}
    \end{subfigure}%
    \begin{subfigure}{0.33\textwidth}    \includegraphics[width=\textwidth]{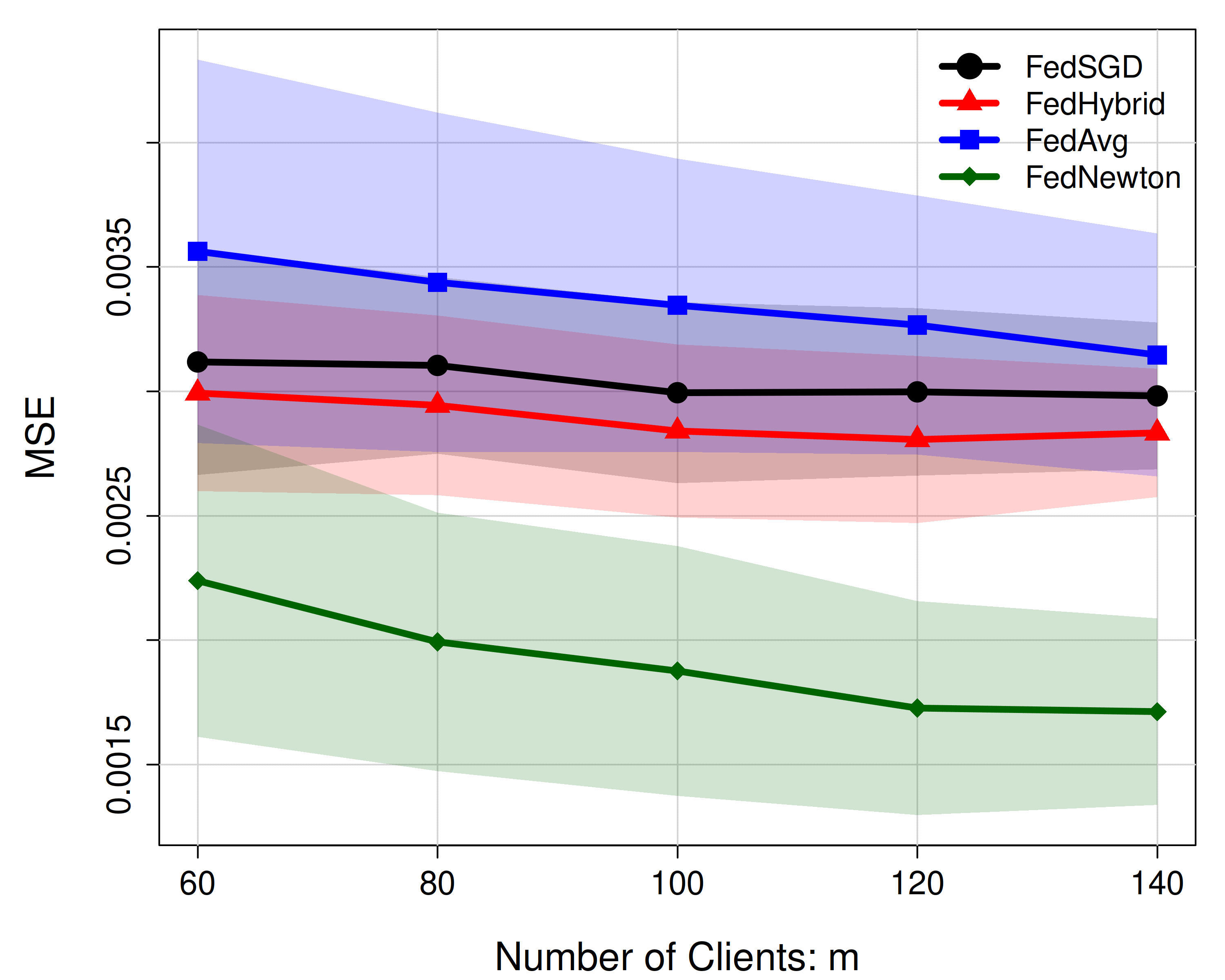}
        \caption{Equal sample size}
    \end{subfigure}%
    \begin{subfigure}{0.33\textwidth}
\includegraphics[width=\textwidth]{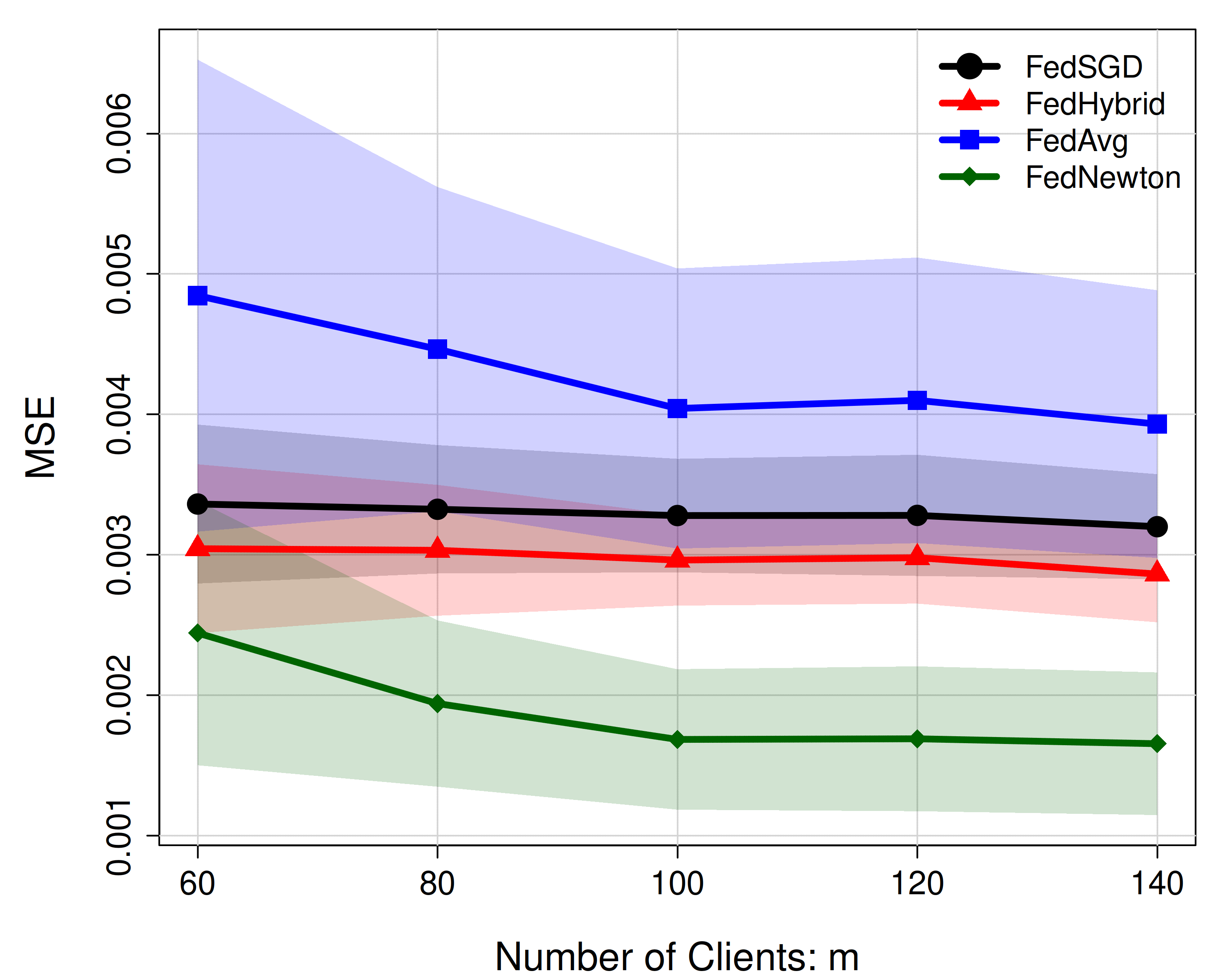}
        \caption{Heterogeneous sample sizes}
    \end{subfigure}
    \caption{MSE comparison of \FedSGD, \FedHybrid, \FedAvg, and \FedNewton\ for logistic regression with local sample sizes having mean around 400 under different distributions: (a) Uniformly distributed sample sizes, (b) Equal sample sizes, and (c) Lognormally distributed sample sizes.
    }
    \label{fig:log_ag1234_400}
\end{figure}

We remind the readers that for logistic regression setting, the probability of $y_{ij}=1$ is defined by the logit link function as $
\pi_{ij} = \frac{\exp(\mathbf{x}_{ij}^\top \boldsymbol{\beta}_{\text{true}})}{1 + \exp(\mathbf{x}_{ij}^\top \boldsymbol{\beta}_{\text{true}})}.$ The binary responses are then generated independently according to $
y_{ij} \mid \mathbf{x}_{ij} \sim \mathrm{Bernoulli}(\pi_{ij}),
\quad j = 1,\dots,n_i,\; i = 1,\dots,m.$
The performance of the federated $M$-estimator algorithms is evaluated by assessing how the MSE varies with the number of clients $m$ while keeping the local sample size $n_i$ fixed so that the total sample size increases. We display the results over two figures, clubbing similar methods together. In Figure \ref{fig:log_ag12_equal_ni}, we display the MSE of \FedSGD\ and \FedHybrid\ methods with increasing values of $m$ for different values of the common client sample size $n_i$. In Figure \ref{fig:log_ag34_equal_ni}, we plot the same results for the other two methods \FedAvg\ and \FedHybrid. In Appendix, we show similar figures for both cases with variable client sample sizes $n_i$.  The simulation results show that larger sample sizes lead to lower MSE, suggesting that estimation accuracy improves when either more clients are included or each client has access to more local data. Even after allowing for variation in local sample sizes, the overall MSE trend remains stable across all four algorithms, showing a decreasing trend as $m$ increases. These results suggest a baseline improvement in all federated methods with increasing sample size, either in terms of the number of clients or the number of samples per client, which is expected from the theoretical results.


\begin{figure}[h]
    \centering

    \begin{subfigure}[t]{0.47\textwidth}
        \centering
        \includegraphics[width=\textwidth]{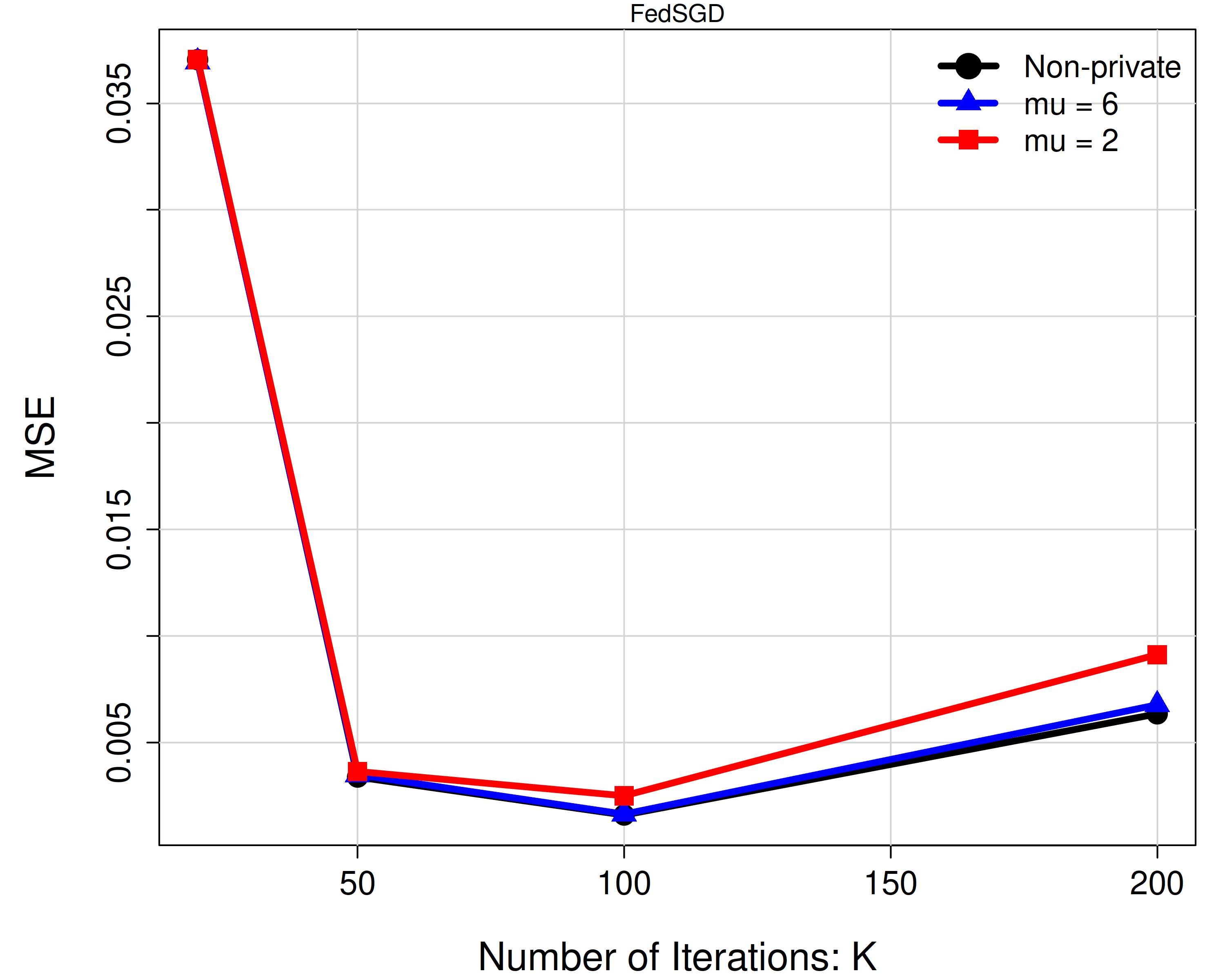}
        \caption{\FedSGD}
    \end{subfigure}
 \hfill
    \begin{subfigure}[t]{0.47\textwidth}
        \centering
        \includegraphics[width=\textwidth]{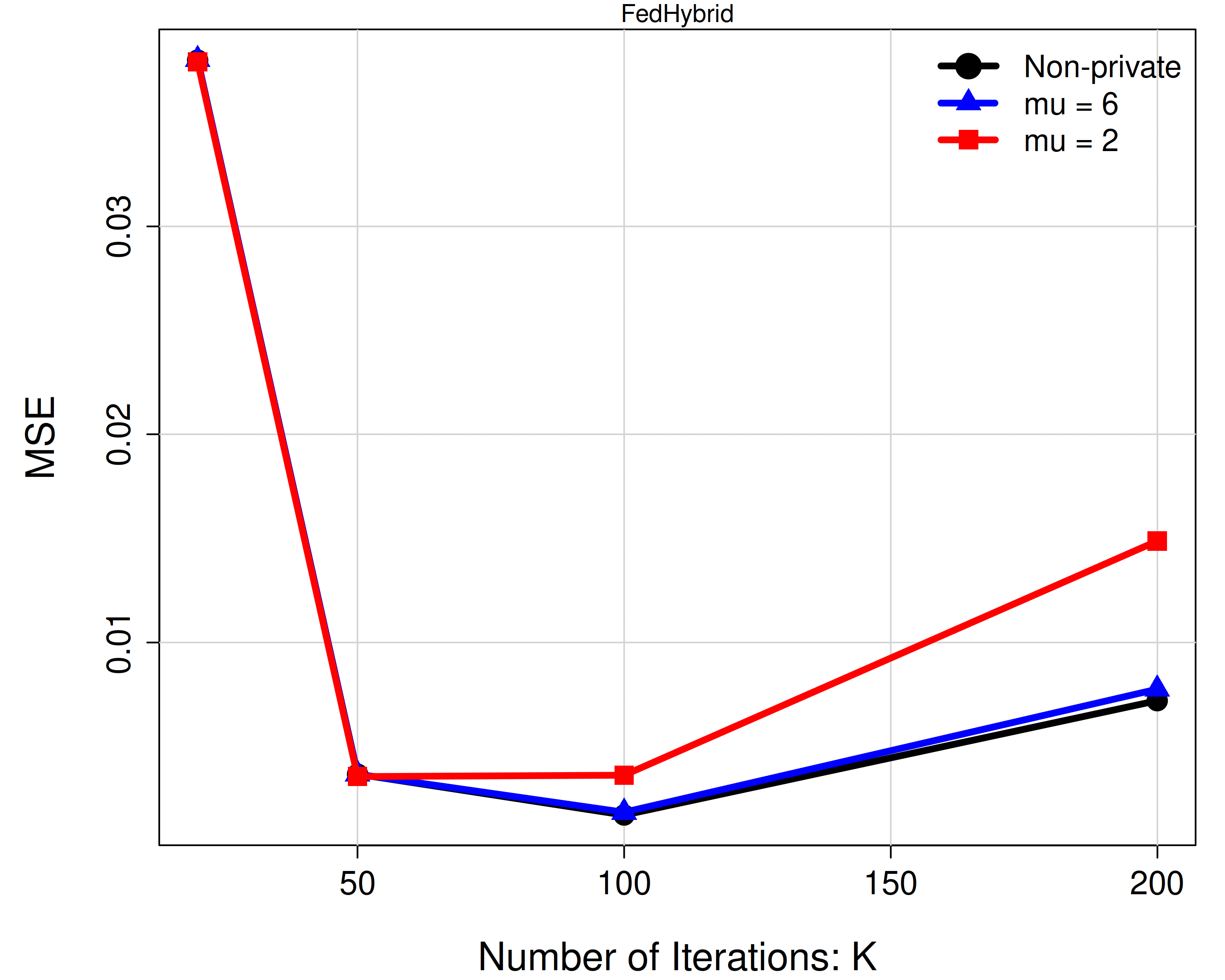}
        \caption{\FedHybrid}
    \end{subfigure}

    \begin{subfigure}[t]{0.47\textwidth}
        \centering
        \includegraphics[width=\textwidth]{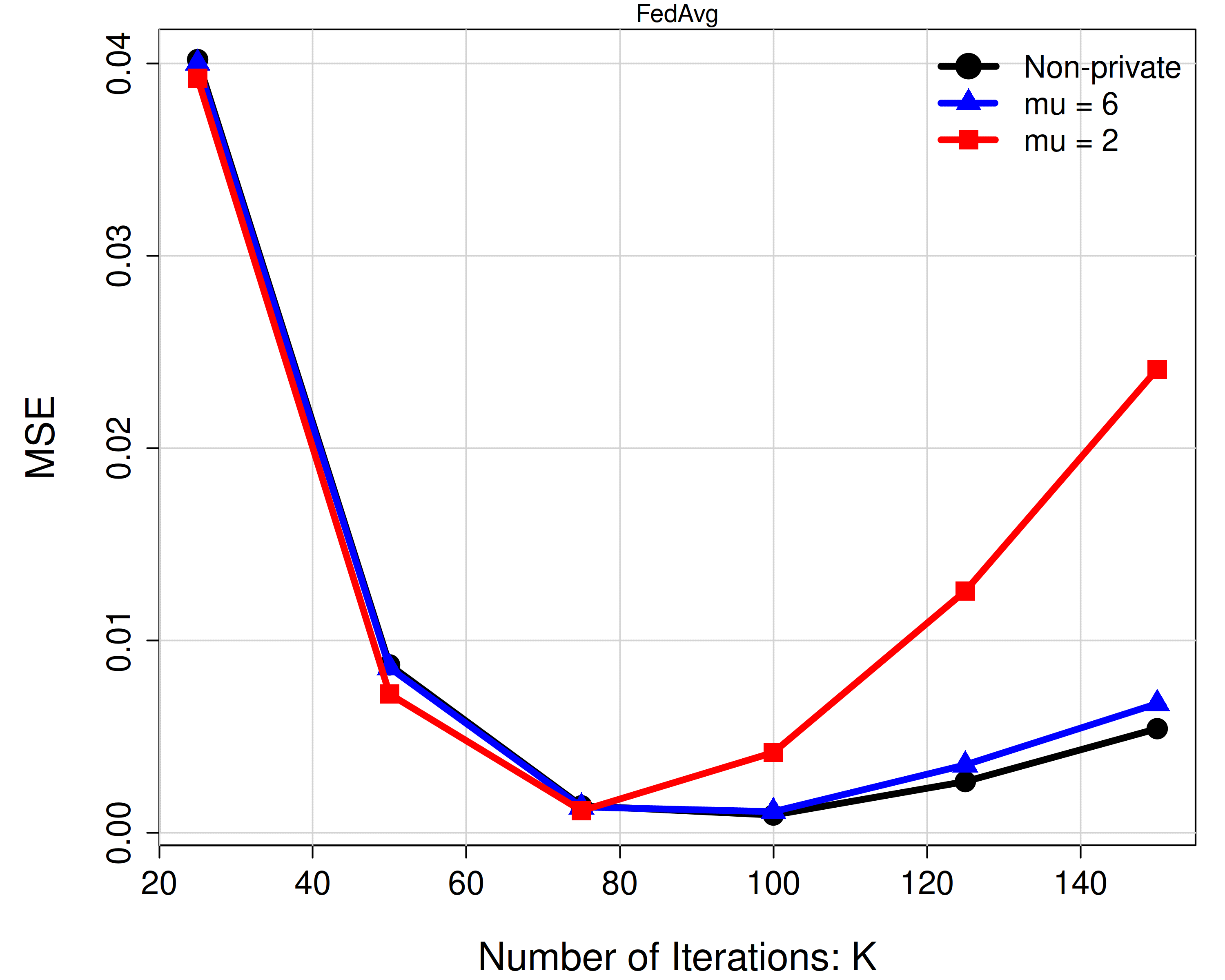}
        \caption{\FedAvg}
    \end{subfigure}
  \hfill
    \begin{subfigure}[t]{0.47\textwidth}
        \centering
        \includegraphics[width=\textwidth]{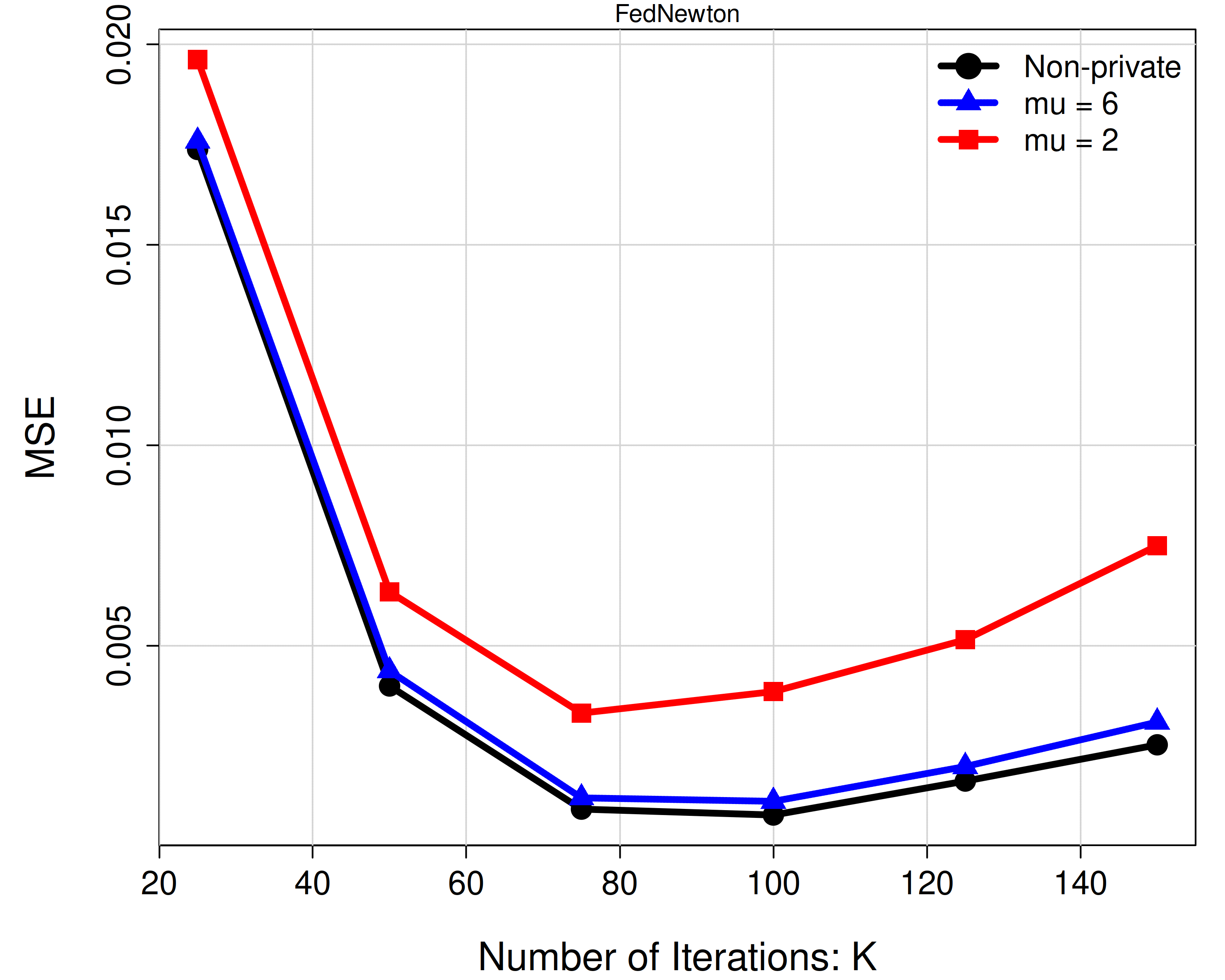}
        \caption{\FedNewton}
    \end{subfigure}

    \caption{The empirical MSE of the methods with increasing number of iterations illustrating the tradeoff between optimization quality and privacy.}
    \label{fig:log_accuracy_pribacy_tradeoff}
\end{figure}

\subsection{Comparison between \FedSGD, \FedHybrid, \FedAvg, and \FedNewton}

After verifying each method works well, we now compare the MSE of the estimators produced by \FedSGD, \FedHybrid, \FedAvg, and \FedNewton\ under different distributions of local client sample sizes. 
The results are presented in Figure~\ref{fig:log_ag1234_400} for logistic regression. 
In all three cases, the average local sample size is approximately $400$. 
For this setting, we consider three distributions of local sample sizes $n_i$: equal sample sizes, discrete uniformly distributed sample sizes, and log-normal distributed sample sizes. 
Specifically, we consider the following distributions for the local sample size $n_i$: 
(i) equal local sample sizes with $n_i = 400$; 
(ii) a discrete uniform distribution $[100, 700]$; and 
(iii) a log-normal distribution with mean $5.5$ and standard deviation $1$ on the log scale.

 Among the two communication heavy methods, \FedHybrid\ outperforms \FedSGD\ in terms of empirical MSE. We remind the readers that in \FedHybrid, clients first compute private local estimators, which are averaged at the server before further \FedSGD\ updates, which saves in communication cost, yet achieves a superior performance in our simulation.

Among the two communication-efficient methods, \FedNewton\ outperforms \FedAvg. Further, among the methods considered, \FedNewton\ achieves the best performance, while \FedAvg~ has the worst performance.  Recall that compared with \FedAvg, \FedNewton\ further refines the estimator by performing an additional Newton step after a \FedAvg\ initialization. This second-order update improves estimation accuracy and leads to lower MSE than \FedAvg\ in this setting. 

\begin{figure}[h]  
    \centering

    \begin{subfigure}{0.32\textwidth}
        \centering
        \includegraphics[width=\textwidth]{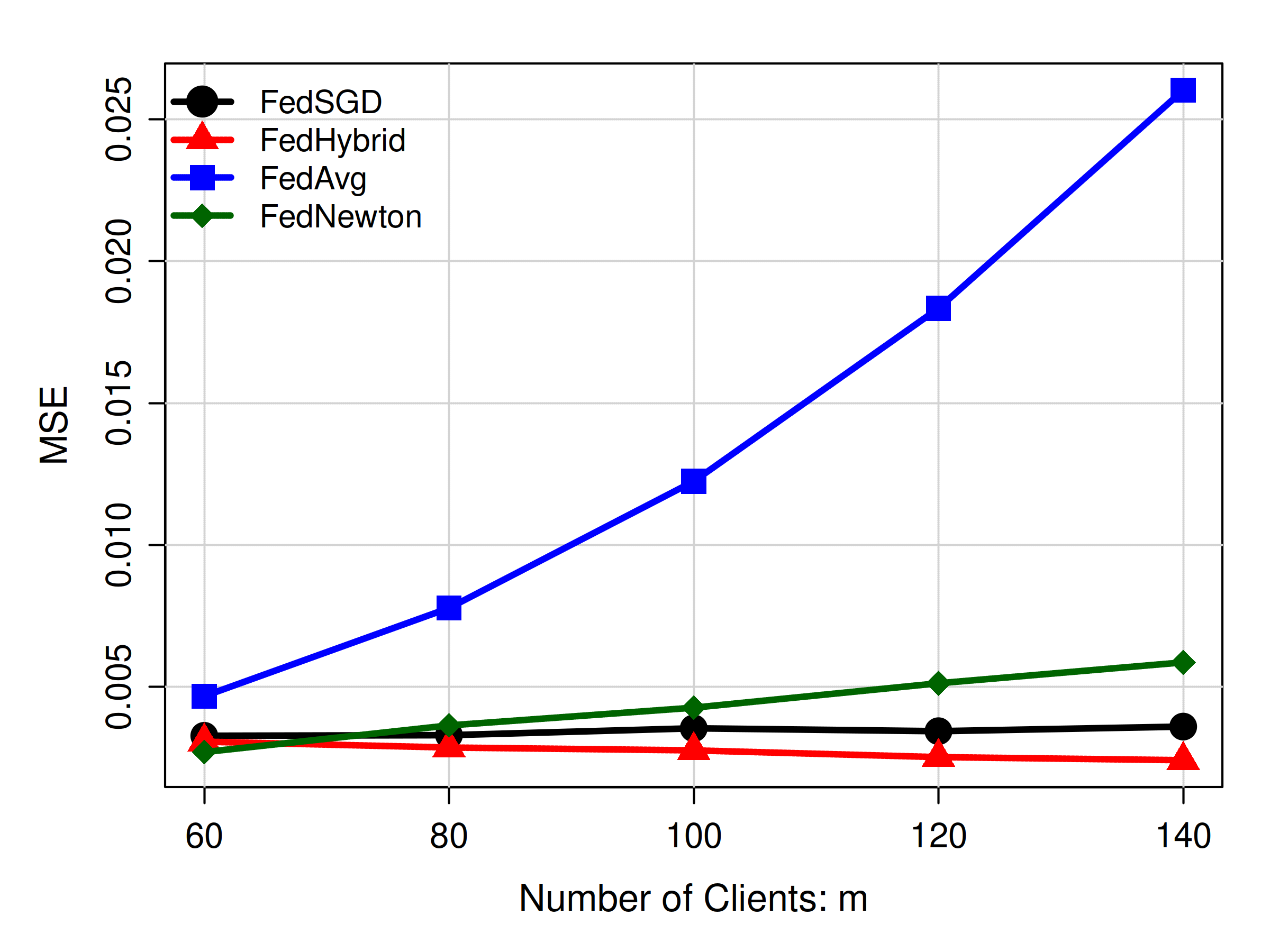}
        \caption{Equal $n_i$}
    \end{subfigure}%
    \begin{subfigure}{0.32\textwidth}
        \centering
        \includegraphics[width=\textwidth]{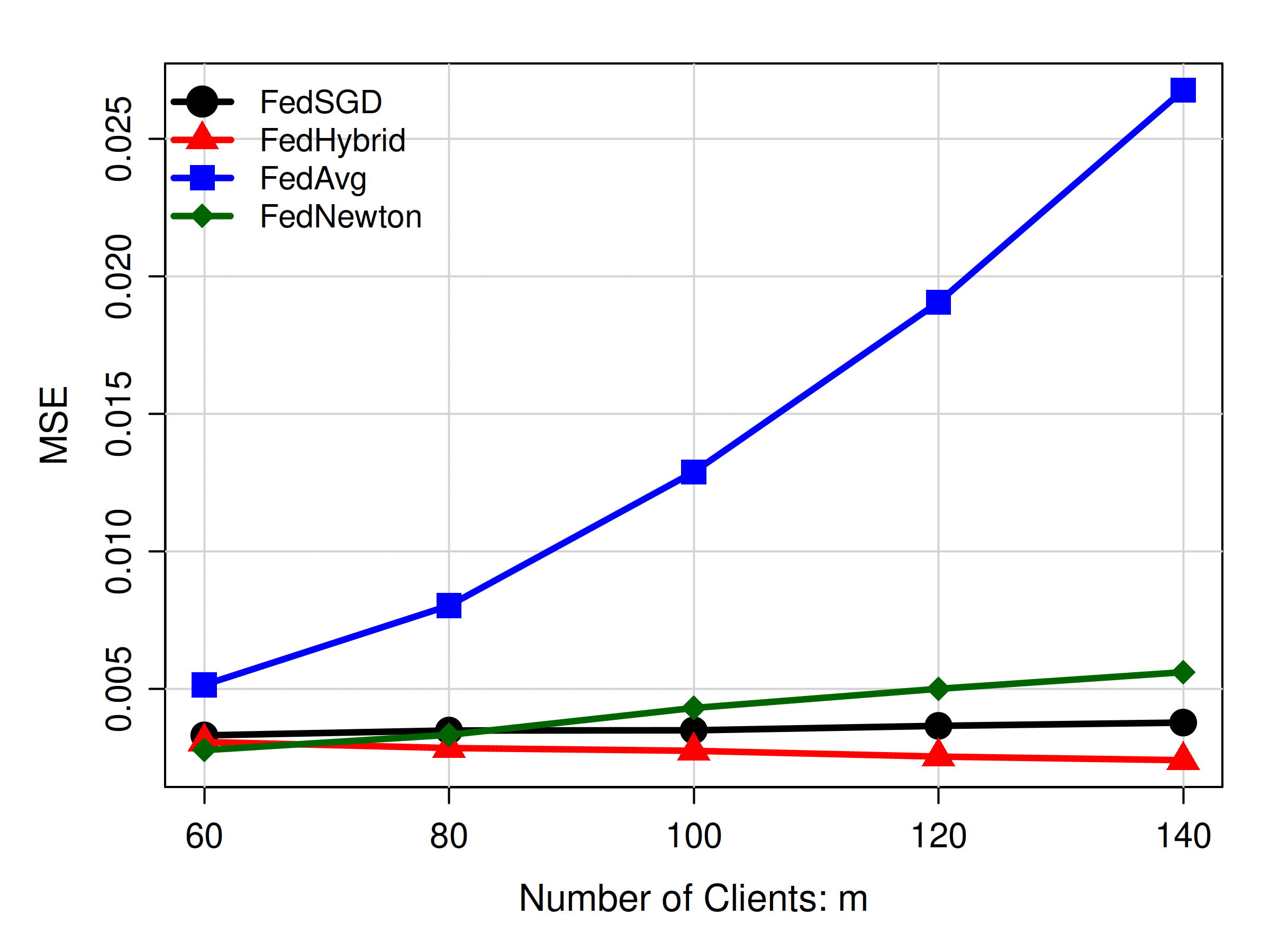}
        \caption{Uniformly distributed $n_i$}
    \end{subfigure}%
    \begin{subfigure}{0.32\textwidth}
        \centering
        \includegraphics[width=\textwidth]{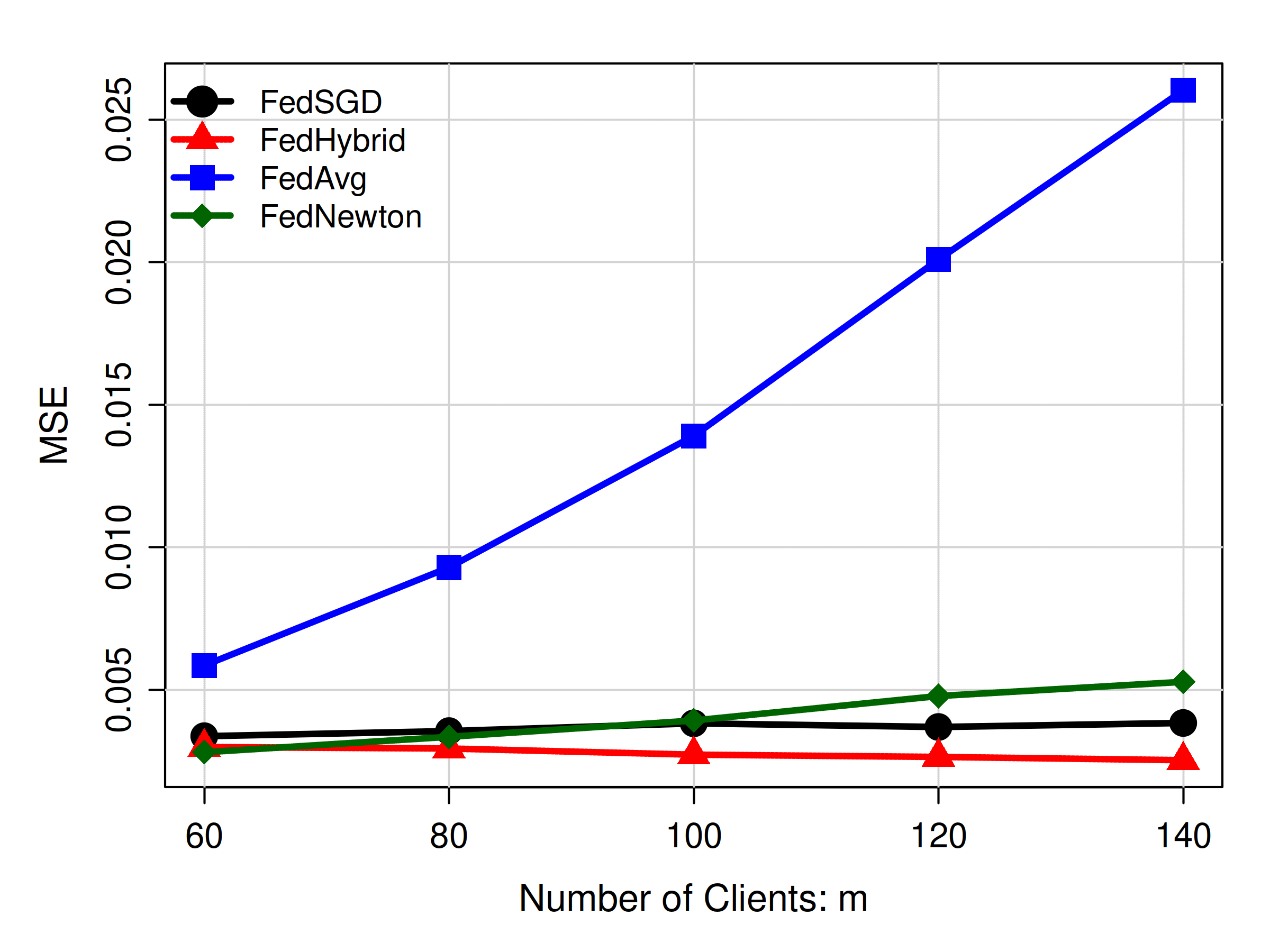}
        \caption{Lognormal distributed $n_i$}
    \end{subfigure}

\caption{MSE comparison of \FedSGD, \FedHybrid, \FedAvg, and \FedNewton\ under fixed total sample size \(N=20000\) with varying number of clients \(m\): (a) equal local sample sizes, (b) uniformly distributed local sample sizes, and (c) lognormally distributed local sample sizes.}
    \label{fig:log_fixedN}
\end{figure}

\subsection{Accuracy and Privacy Trade-off with Number of Iterations} 
We next validate the theoretical results showing that increasing the number of iterations leads to a phenomenon of conflict between increased accuracy due to better optimization and decreased accuracy due to higher privacy noise. In Figure \ref{fig:log_accuracy_pribacy_tradeoff} we show this trade-off between optimization performance and privacy noise for the logistic regression setting. When the number of iteration $K$ is small, the MSE decreases as $K$ increases. This is expected because additional steps help improve the estimator. However, after a certain point, further increasing $K$ can increase the MSE. This is because a larger number of iterations also leads to more accumulated privacy noise.

A smaller value of $\mu$ gives stronger privacy and requires more Gaussian noise. As shown in the figure, the curve for $\mu=2$, while remaining similar to the non-private and $\mu=6$ curves for smaller $K$, sharply goes above those curves for larger $K$, especially for $K>100$. This shows that although a larger $K$ can improve optimization, too many iterations may add too much privacy noise and harm the final estimator. Overall, the results illustrate the trade-off among accuracy, privacy, and $K$.

\subsection{Effect of the Number of Clients under Fixed Total Sample Size}

To illustrate the effect of federation across clients, we conduct a simulation study with a fixed total sample size $N$, while allowing the number of clients $m$ to increase. The total sample size is fixed at \(N=20000\). For each value of \(m\), we consider three local sample-size allocation schemes: equal sample sizes, uniformly distributed sample sizes, and lognormally distributed sample sizes. In this experiment, increasing \(m\) redistributes the same total amount of data across more clients, instead of increasing the total sample size as was the case in section 4.3. This is a very relevant scenario for modern Federated Learning applications, where we typically have many small client devices, each with a small amount of data.

The results are consistent with the theoretical comparison. \FedAvg\ shows the sharpest increase in MSE as \(m\) grows, which agrees with the \(O(m^2d^2/N^2)\) term in its error bound. When \(N\) is fixed, this term increases quadratically with \(m\), indicating that \FedAvg\ is more sensitive to splitting the same data across more clients. In contrast, \FedSGD, \FedHybrid, and \FedNewton\ remain relatively stable as \(m\) increases. This is consistent with their \(m\)-dependent terms in the error bound, which are \(O(md^2\log(Nd)/(\mu^2N^2))\) for \FedSGD, \(O(md^2\log m/(\mu^2N^2))\) for \FedHybrid, and \(O(md^2/(\mu^2N^2))\) for \FedNewton. Overall, this experiment shows that redistributing the same total sample size across more clients affects the four private methods differently. The performance of \FedAvg~deteriorates quickly, whereas our new communication-efficient solution \FedNewton~significantly mitigates this decline.

\subsection{Poisson GLM}

\begin{figure}[h]  
    \centering

    \begin{subfigure}{0.32\textwidth}
        \centering
        \includegraphics[width=\textwidth]{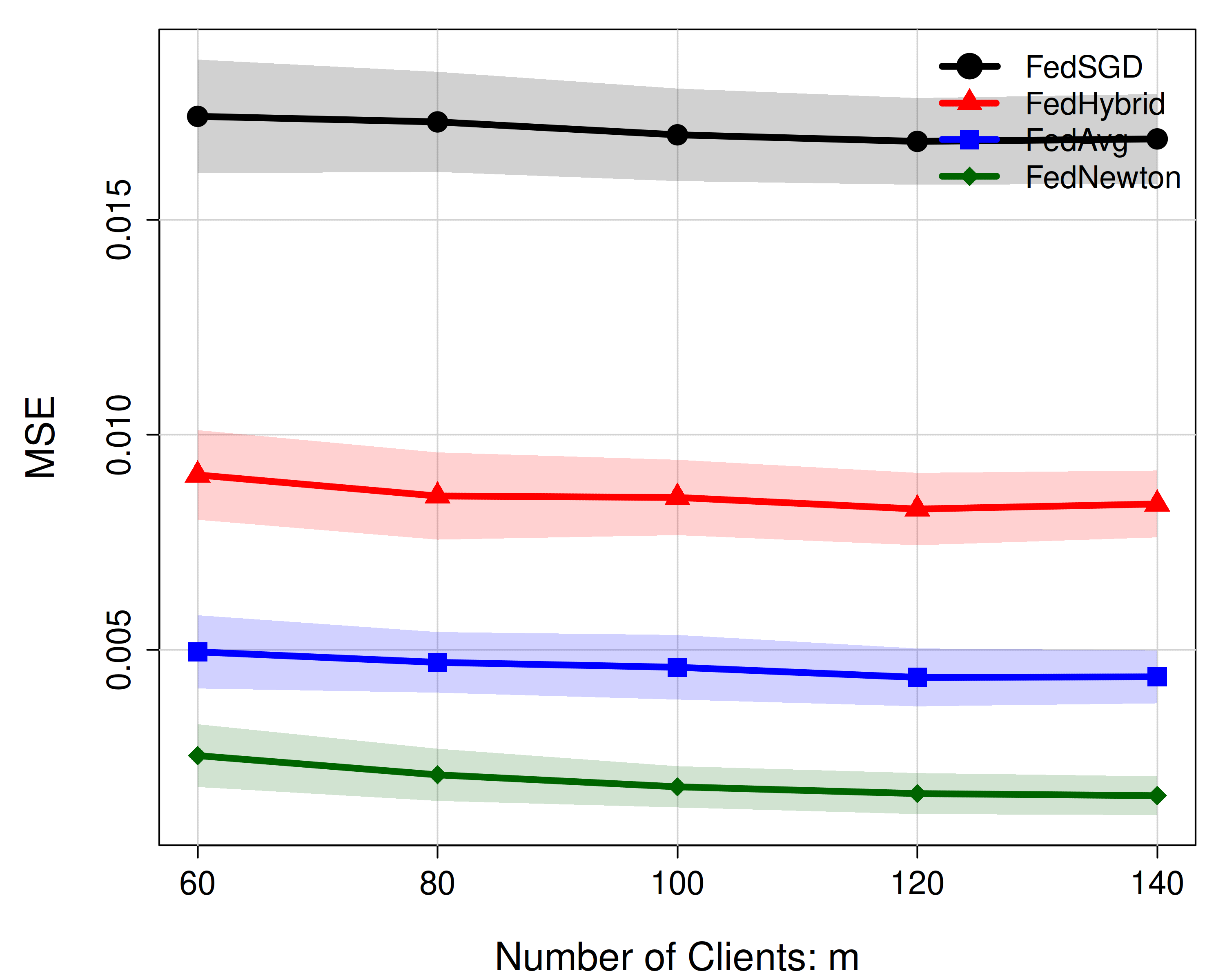}
        \caption{Uniformly distributed $n_i$}
    \end{subfigure}%
    \begin{subfigure}{0.32\textwidth}
        \centering
        \includegraphics[width=\textwidth]{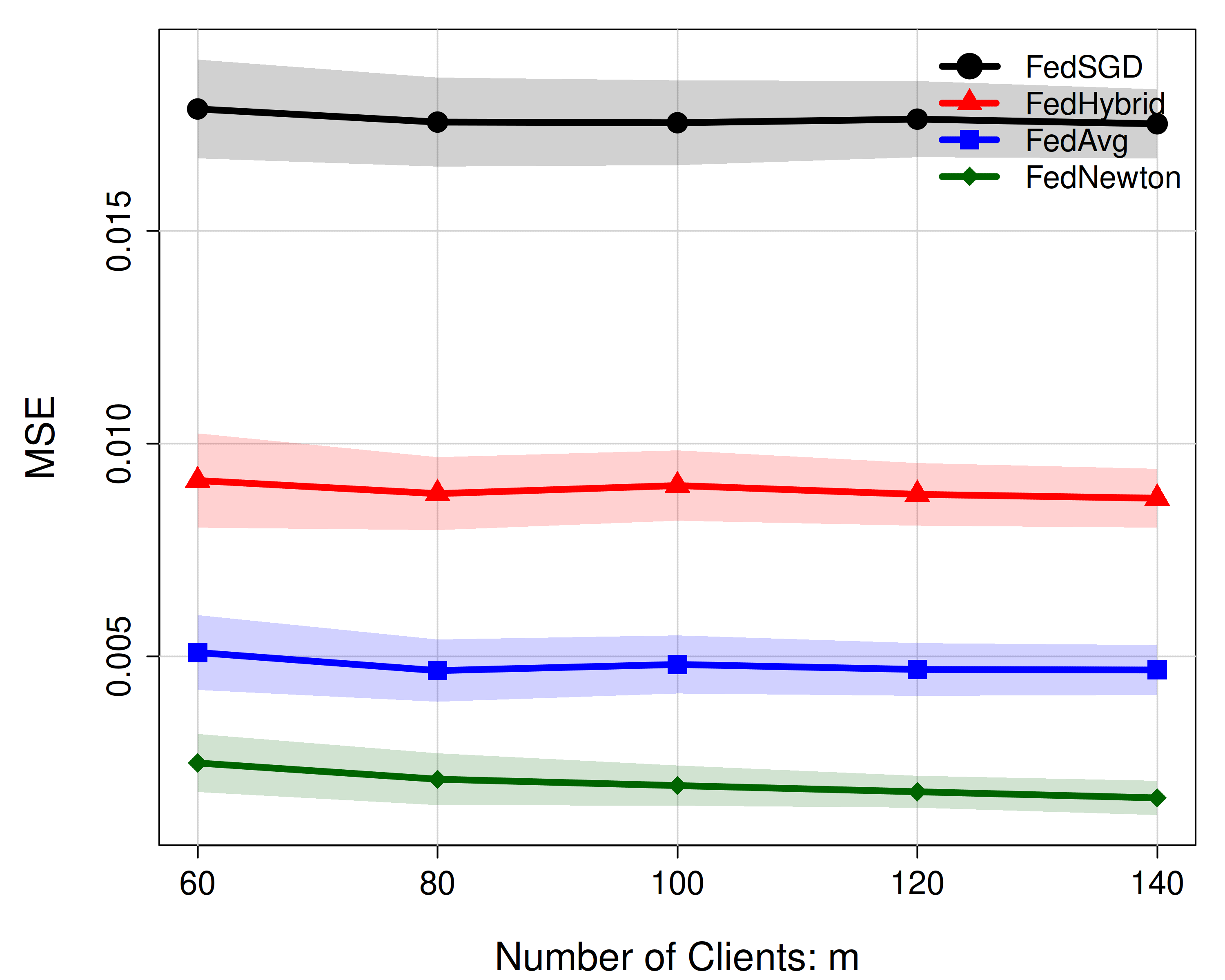}
        \caption{Equal $n_i$}
    \end{subfigure}%
    \begin{subfigure}{0.32\textwidth}
        \centering
        \includegraphics[width=\textwidth]{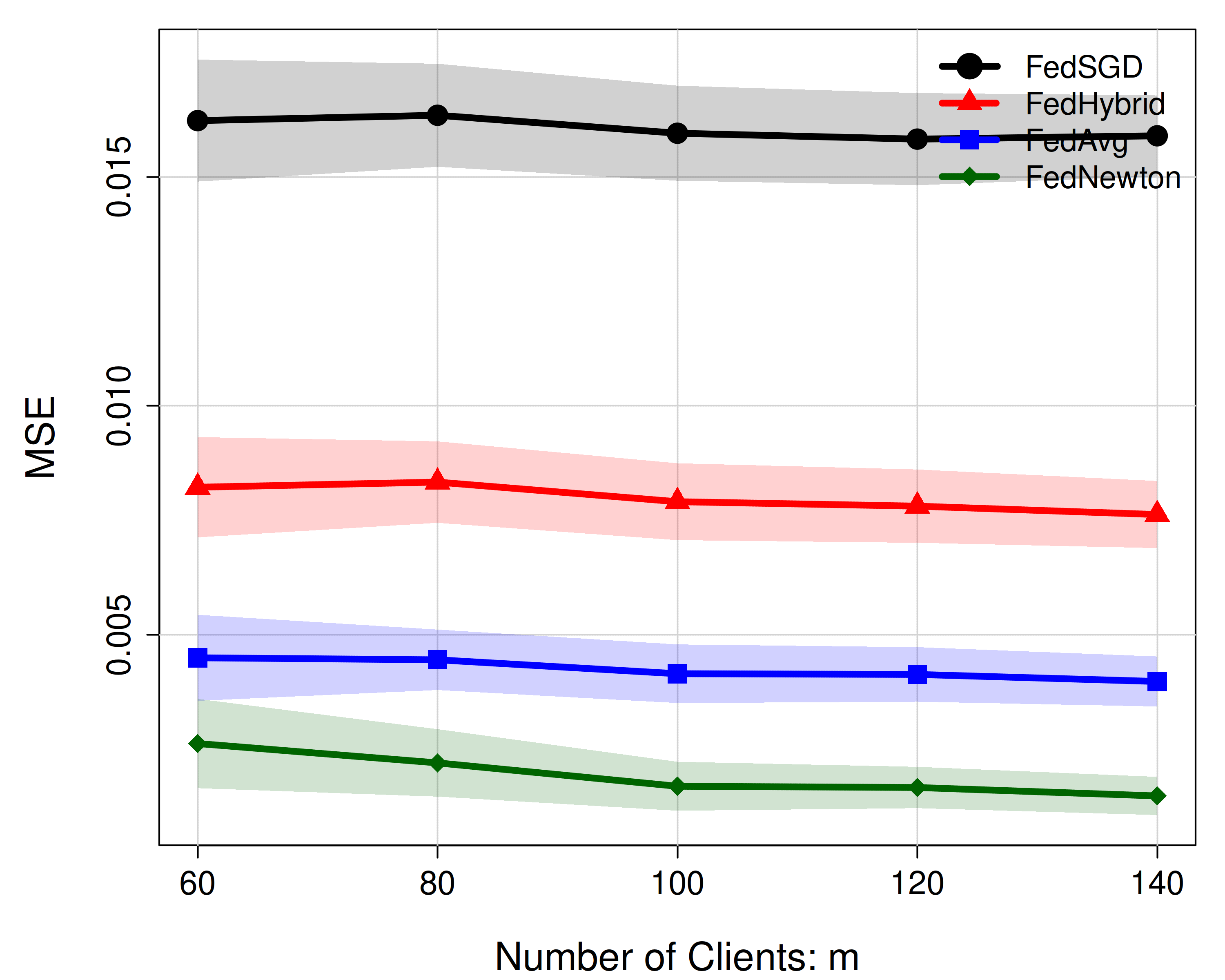}
        \caption{Lognormal distributed $n_i$}
    \end{subfigure}

    \caption{
    Poisson GLM: Empirical MSE comparison of \FedSGD, \FedHybrid, \FedAvg, and \FedNewton\ for local sample sizes with mean around 400 under different local sample size distributions: equal sample sizes, uniformly distributed sample sizes, and lognormally distributed sample sizes.
    }
    \label{fig:poi_ag1234_400}
\end{figure}

In the Poisson GLM the link function is log. Accordingly, the response variables are generated independently according to $
y_{ij} \mid \mathbf{x}_{ij} \sim \mathrm{Poisson}(\lambda_{ij}),
\quad j = 1,\dots,n_i,\; i = 1,\dots,m$, with the conditional mean of the Poisson distribution given by $
\lambda_{ij} = \exp(\mathbf{x}_{ij}^\top \boldsymbol{\beta}_{\text{true}}).$

To evaluate the performance of the proposed federated M-estimator algorithms, we compute the empirical MSE as the number of clients $m$ changes both for homogeneous sample sizes and different sample sizes. For the four proposed algorithms, \FedSGD, \FedHybrid, \FedAvg, and \FedNewton, the empirical MSE of the corresponding final estimates are plotted against increasing values of $m$ under different values of $n_i$ in Figure \ref{fig:poi_ag12_equal_different_ni} and Figure \ref{fig:poi_ag34_equal_different_ni} in the Appendix \ref{appsim}. The results indicate that larger local sample sizes lead to lower MSE across all values of m, suggesting that the estimation accuracy improves when there are more clients or when each client has more local data. Moreover, variation in local sample sizes does not substantially affect the overall MSE trend. Regardless of the variation in local sample sizes, the MSE decreases as the number of clients increases for all four algorithms.

The four methods are compared against each other in Figure \ref{fig:poi_ag1234_400}. We once again see that among communication inefficient methods, \FedHybrid~outperforms \FedSGD. Among communication-efficientient methods, \FedNewton\ clearly performs better than \FedAvg. The comparison of the methods with increasing number of iterations $K$ is presented in the Appendix Figure \ref{fig:poi_accuracy_pribacy_tradeoff}, while the comparison in the case of increasing number of clients yet fixed total sample size is presented in Appendix Figure \ref{fig:poi_fixedN}. The conclusions from these figures are similar to what we had in the logistic regression case.

\section{Real Data Applications}

\subsection{Binary and Multiclass Logistic regression on MNIST}

The first real data application  we consider is binary and multiclass logistic regression on the popular MNIST image classification dataset \citep{lecun1998mnistdataset}. The MNIST consists of grayscale handwritten digit images from ten classes,
corresponding to the digits $0$ through $9$.  Each image is available as a
$28 \times 28$ matrix with entries being the pixel intensities taking integer values between $0$ and $255$. The MNIST dataset consists of two components: a training set containing 60000 images and a test set containing 10000 images, each with corresponding class labels. For this application, we vectorize the matrices and store them as feature vectors of length $784$. The associated labels indicating the handwritten digit ($0$--$9$) in each image is our response variable. Moreover, all image vectors are normalized by dividing pixel intensities by 255, so that their values lie in [0,1]. 

In this study, we combine the original training and testing sets to form a single dataset of size $N = 70000$. This combined dataset is then randomly partitioned across $m = 80$ clients, while ensuring a minimum local sample size of 800 images for each client. This is done so that the number of samples in each client is higher than the total number of parameters, which is $d=784$. After allocating this minimum number of samples to all clients, the remaining images are distributed among the clients according to randomly generated proportions.

Within each client, the local MNIST data are split into $K=5$ folds for cross-validation. 
Fold assignment is performed independently within each client by first applying a random permutation to the observations and then allocating them approximately evenly across the five folds. In each round, one fold serves as the testing set and the remaining four folds are used for training. Then in every round, the out of sample AUC values are computed for the data in the test fold for each method and each client. To reduce variability induced by the partition into clients, we repeat the client splitting procedure $S=3$ times. 
Therefore, within each split, each method yields one AUC per client and per fold. We then average these client-level AUCs across the $5$ folds to obtain a fold-averaged AUC for each client. 
Therefore, for each fixed $(\mu,\text{method})$ combination, we obtain $m\times S = 80\times 3 = 240$ fold-averaged client-level AUC values, which are summarized using boxplots. 

To investigate how predictive performance varies with the privacy budget, we consider two levels of privacy, with larger values of the privacy parameter corresponding to weaker privacy requirement. Specifically, we take
\(
\mu \in \{6, 2\}
\). We compare nine estimators using the AUC values:
\FedSGD, \FedAvg, \FedHybrid, \FedNewton, \texttt{DP-SCAFFOLD}, \texttt{NP-SCAFFOLD},
\texttt{NP-FedAvg}, \texttt{NP-LocalFit}, and \texttt{NP-Pooled}. The methods with a ``\texttt{NP}'' are non-private methods and are used as baselines.

\begin{figure}[h]
    \centering
    \includegraphics[width=0.49\linewidth]{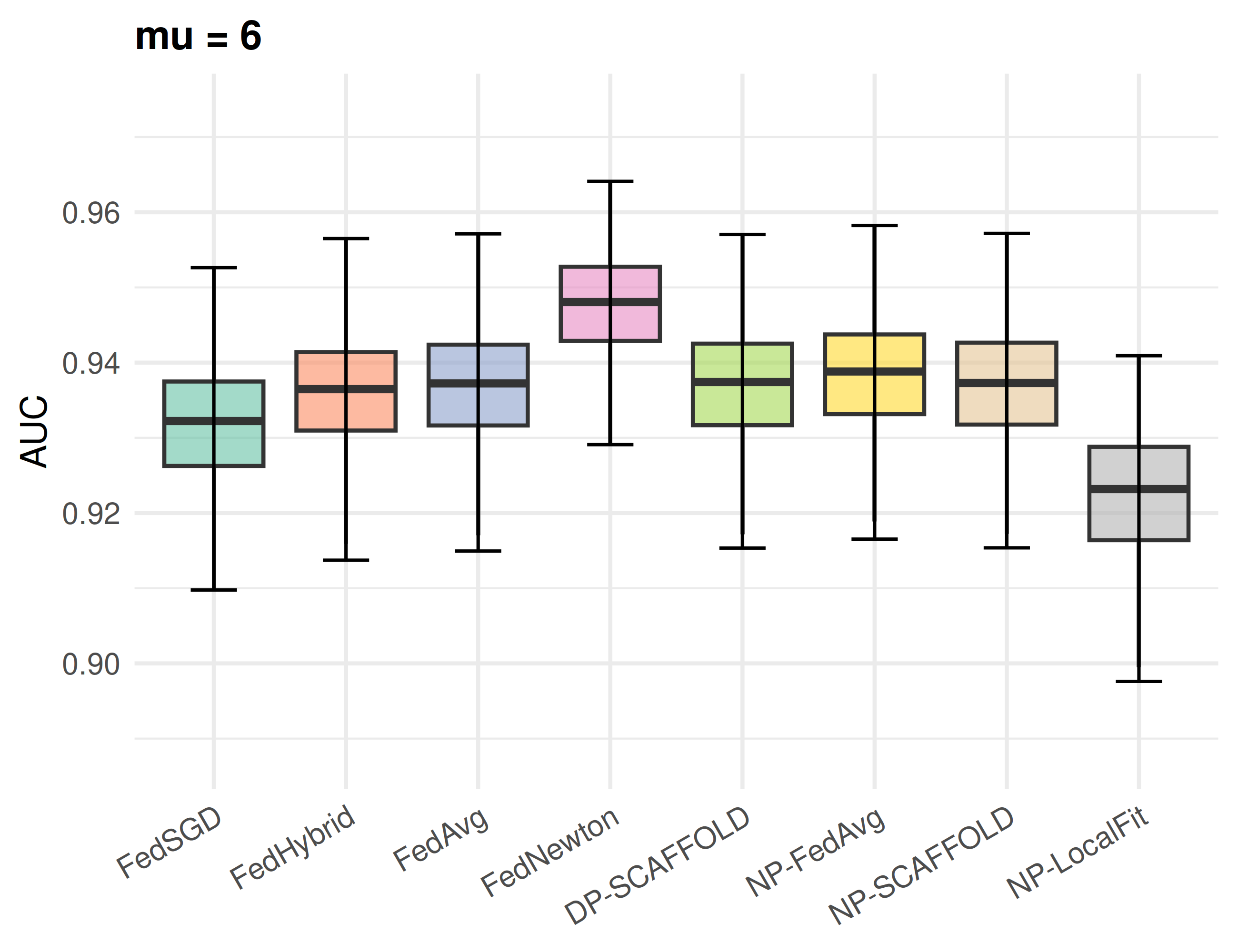}
        \includegraphics[width=0.49\linewidth]{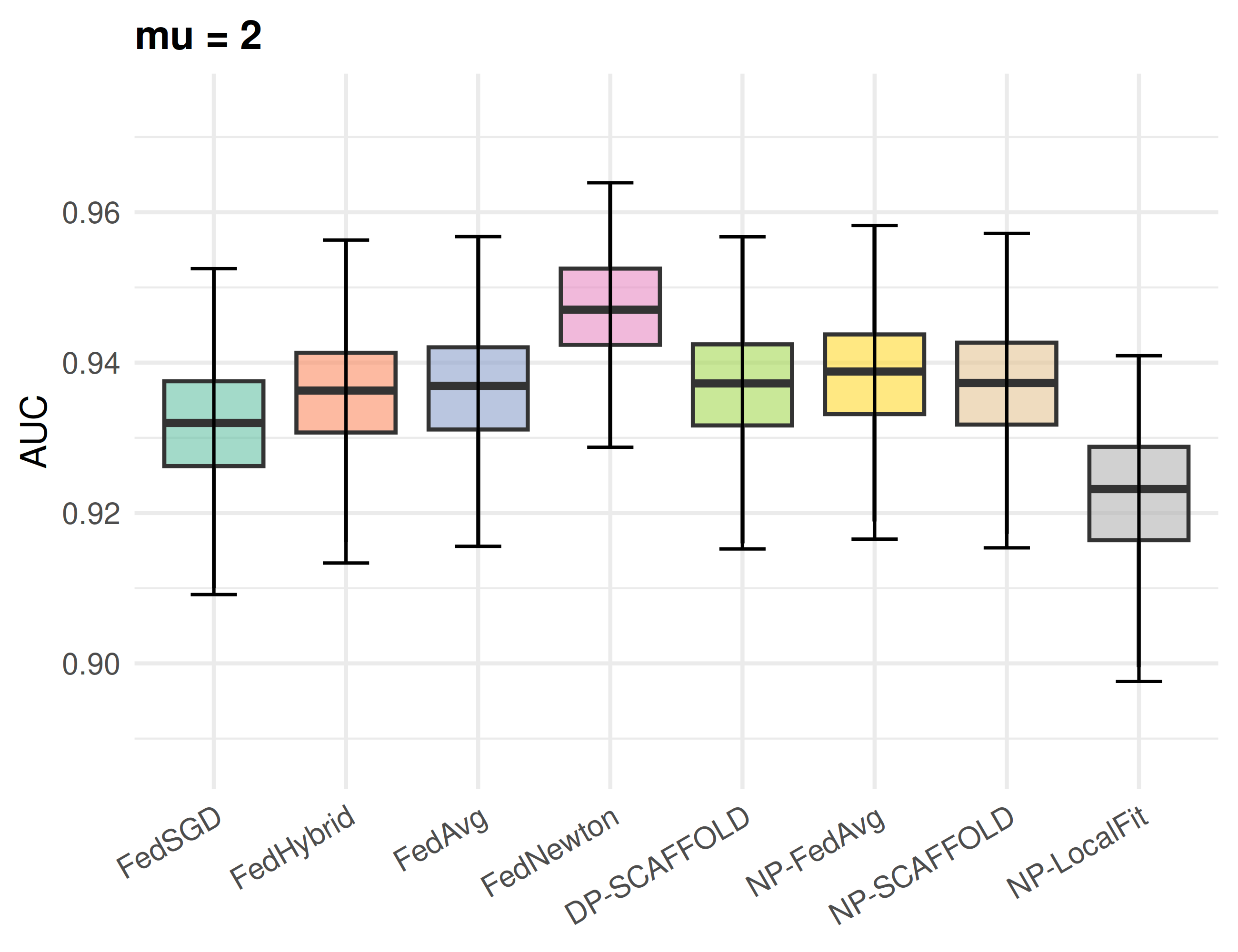}
    \caption{Binary class logistic regression MNIST results under different privacy budgets.}
    \label{MNIST_binary}
\end{figure}

\paragraph{Binary Classification} 
We consider a binary classification task for MNIST. Specifically, images whose original labels correspond to odd digits \((1,3,5,7,9)\) are relabeled as \(1\), while images corresponding to even digits \((0,2,4,6,8)\) are relabeled as \(0\). The resulting binary response variable indicates whether an image represents an odd digit. A logistic regression model is then fitted to estimate the probability that an image corresponds to an odd digit.

Figure~\ref{MNIST_binary} reports the boxplots of the client-level AUC values under \(\mu=6\) (left) and \(\mu=2\) (right) respectively. The private federated estimators studied in this paper achieve strong predictive performance in both settings. The results are also stable when the privacy constraint becomes stronger, from \(\mu=6\) to \(\mu=2\). This suggests that the additional privacy noise does not lead to a substantial loss in prediction accuracy in this application.

Among all the private methods, \FedNewton\ has the highest median AUC. In fact it has higher median AUC compared to the non-private methods as well. In particular it outperforms both DP version of \FedAvg~and \texttt{SCAFFOLD}. The \FedHybrid\ and \FedAvg\ also perform well, while \FedSGD\ is slightly lower but still comparable. They methods are also competitive Compared with the non-private federated baselines, the proposed private methods remain competitive. Their performance is close to the performance of NP-FedAvg and NP-SCAFFOLD. In contrast, NP-LocalFit has the lowest AUC among the methods considered, which suggests that local fitting alone is less effective than federated aggregation in this setting. These results indicate that the proposed private federated estimators achieve a good balance between data privacy and prediction accuracy under reasonable privacy budgets. 

\paragraph{Multi-Class Classification}
Now we consider the more natural target of multi-class classification using a multiclass logistic regression model. More specifically, we fit a logistic (often also called ``softmax") regression model to estimate the class probabilities for the ten digit classes. Figure~\ref{fig:MNIST_multiclass} shows the client-level classification accuracy under \(\mu=6\) and \(\mu=2\). The accuracy of classification is a more commonly used metric for multi-class classification problems.

\begin{figure}
    \centering
    \begin{minipage}{0.49\textwidth}
        \centering
        \includegraphics[width=\linewidth]{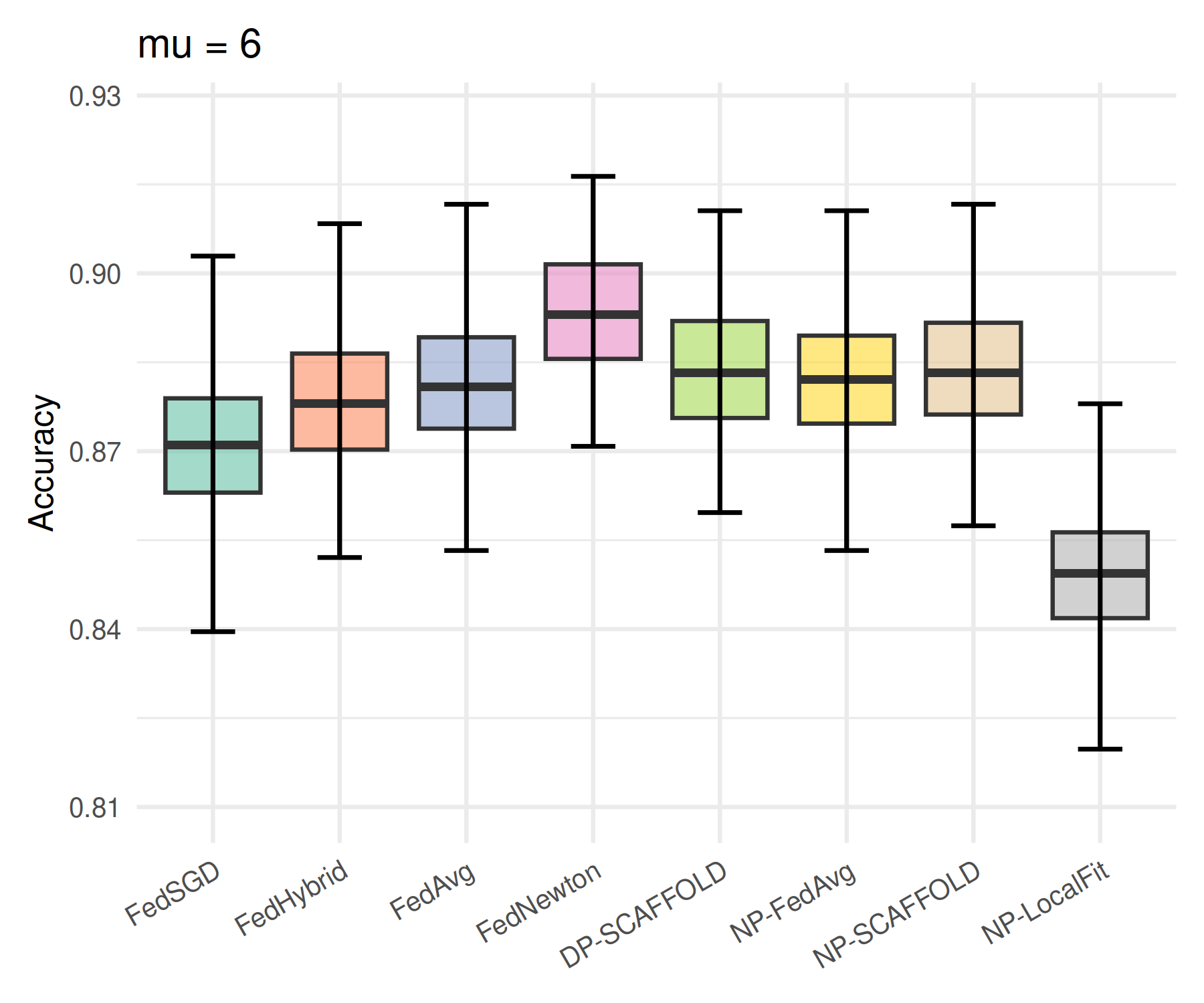}
        
    \end{minipage}
    \hfill
    \begin{minipage}{0.49\textwidth}
        \centering
        \includegraphics[width=\linewidth]{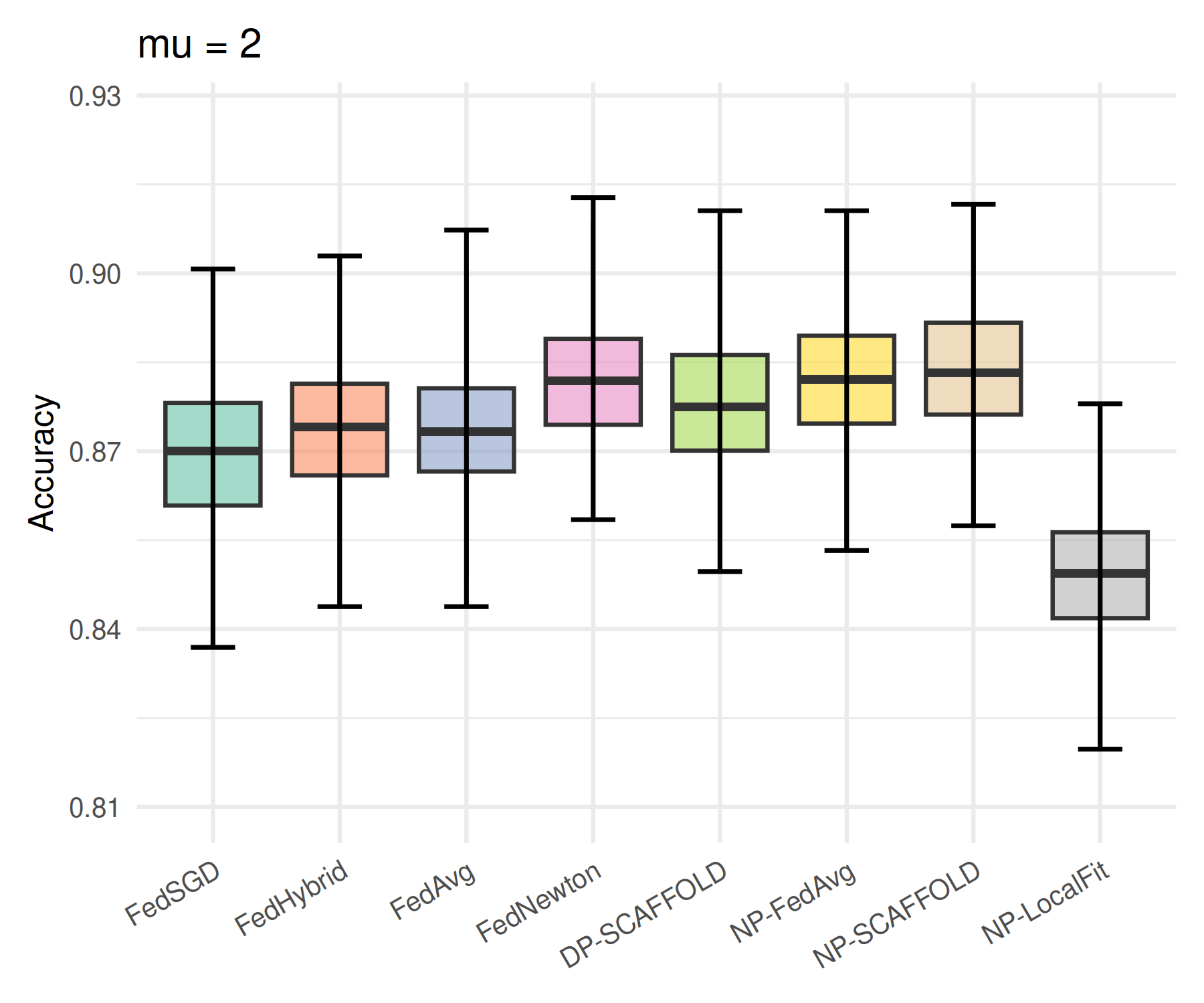}
        
    \end{minipage}
    \caption{Multiclass MNIST results under different privacy budgets.}
    \label{fig:MNIST_multiclass}
\end{figure}

As shown in Figure~\ref{fig:MNIST_multiclass}, all methods achieve relatively high accuracy in this setting. The results are also stable when the privacy constraint becomes stronger, from \(\mu=6\) to \(\mu=2\). The private federated methods perform closely to the non-private federated methods. Among the private methods, \FedNewton~continues to have the highest median accuracy, followed by DP-SCAFFOLD, while \FedSGD, \FedHybrid, and \FedAvg~have slightly lower but still comparable performance. Among the non-private methods, NP-SCAFFOLD and NP-FedAvg perform well, whereas NP-LocalFit has the lowest accuracy. 

These results suggest that the proposed private federated estimators can preserve  predictive performance well in the multiclass MNIST problem. Among the private methods, \FedNewton\ appears to benefit from the additional Newton refinement step. The similar  performance under \(\mu=6\) and \(\mu=2\) further suggests that the private methods are not very sensitive to the moderate increase in privacy protection in this data application.

\subsection{Federated training of neural networks on MNIST and CIFAR10}

We apply the private federated learning methods to multi-class image classification on MNIST and CIFAR-10 using convolutional neural networks (CNNs).  For MNIST, the CNN has two convolutional layers followed by two fully connected layers.  For CIFAR-10, the CNN has three convolutional layers followed by three fully connected layers.  The global accuracy reported in all figures is the accuracy on the standard held-out test set.

We consider two federated learning scenarios.  In the first scenario, shown in Figure~\ref{fig:accuracy-samples}, we fix the number of clients at $m=100$ and increase the number of samples available to each client.  In the second scenario, shown in Figure~\ref{fig:accuracy-clients}, we use the full training splits and increase the number of clients while keeping the total training sample size fixed.  Thus, as $m$ increases, the average number of samples per client decreases.

\begin{figure}[h]
    \centering
    \includegraphics[width=0.48\linewidth]{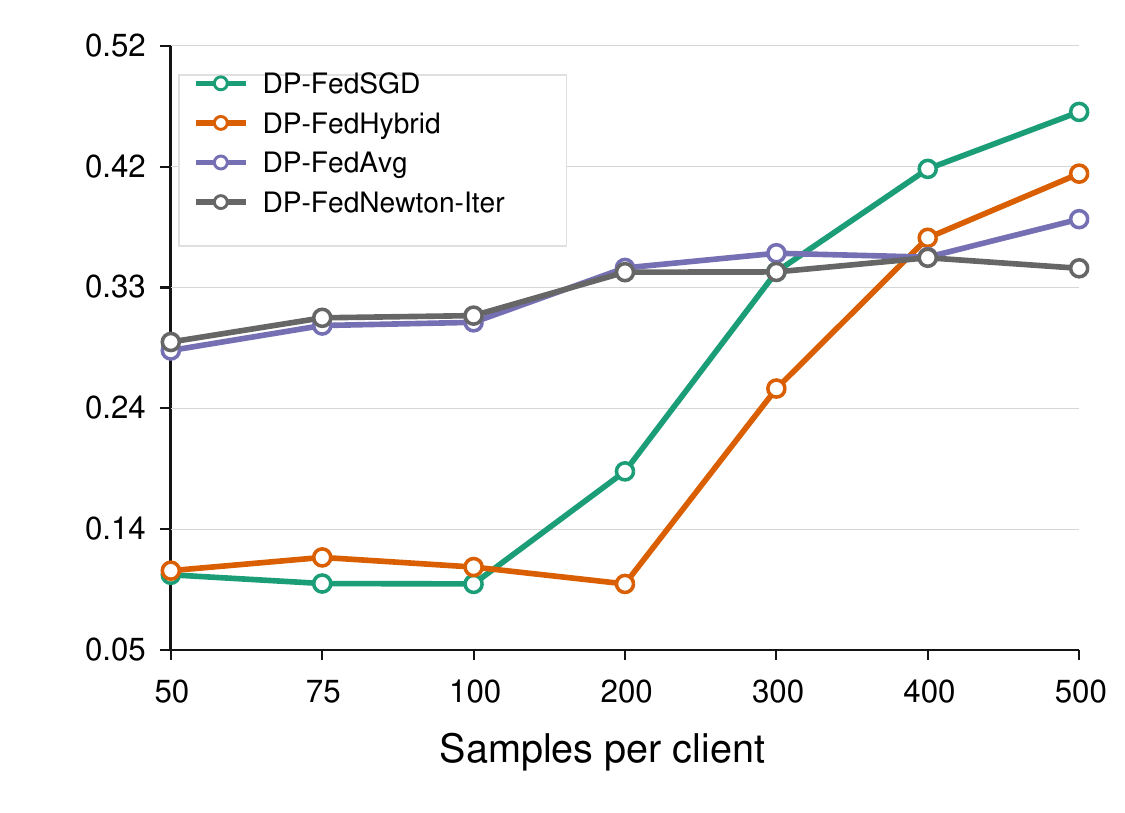}
    \hfill
    \includegraphics[width=0.48\linewidth]{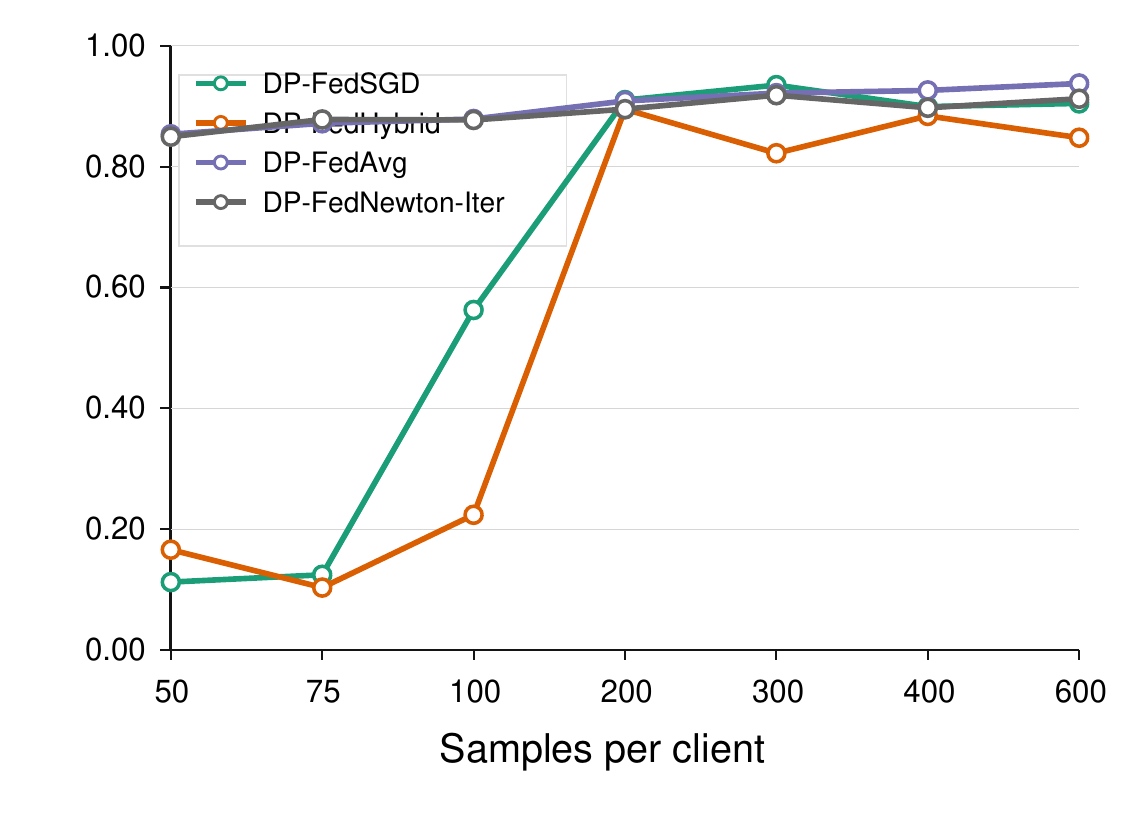}
    \caption{CNN results with increasing samples per client: median test
    accuracy on (left) CIFAR-10 and (right) MNIST.  The number of clients is
    fixed at $m=100$. The medians are computed over five independent runs.}
    \label{fig:accuracy-samples}
\end{figure}

\begin{figure}[h]
    \centering
    \includegraphics[width=0.48\linewidth]{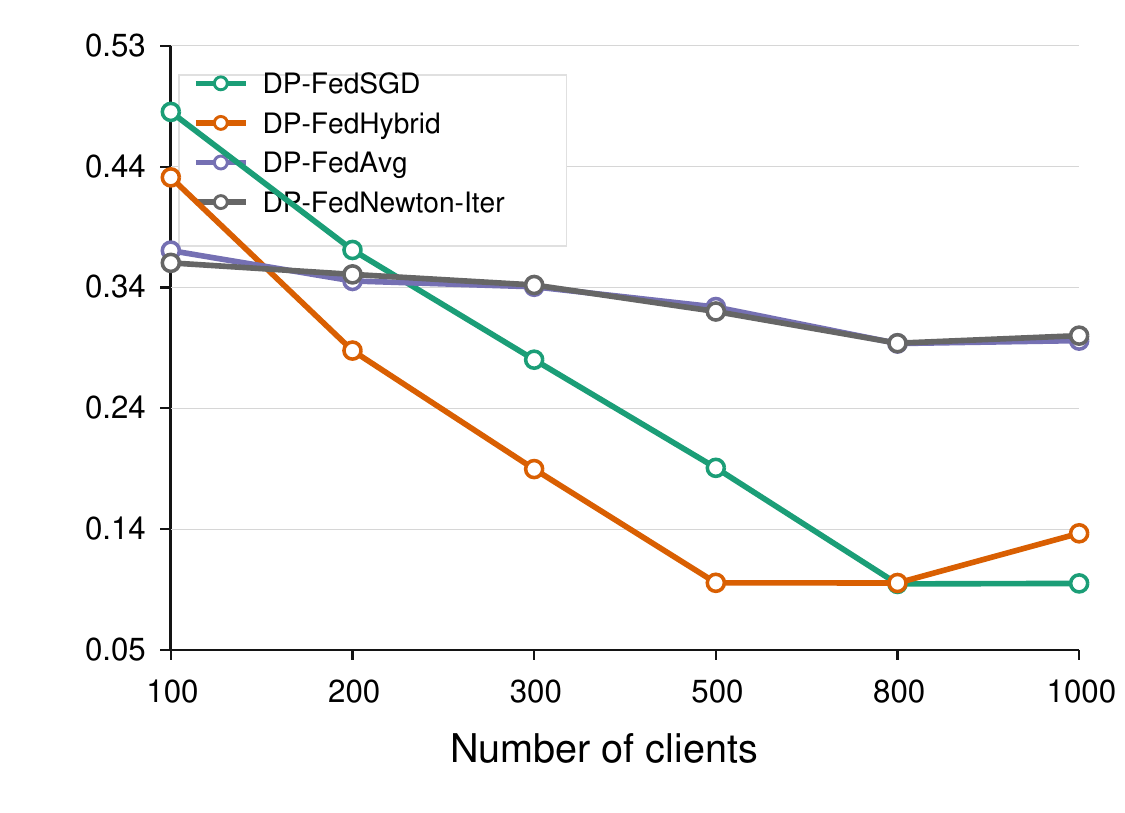}
    \hfill
    \includegraphics[width=0.48\linewidth]{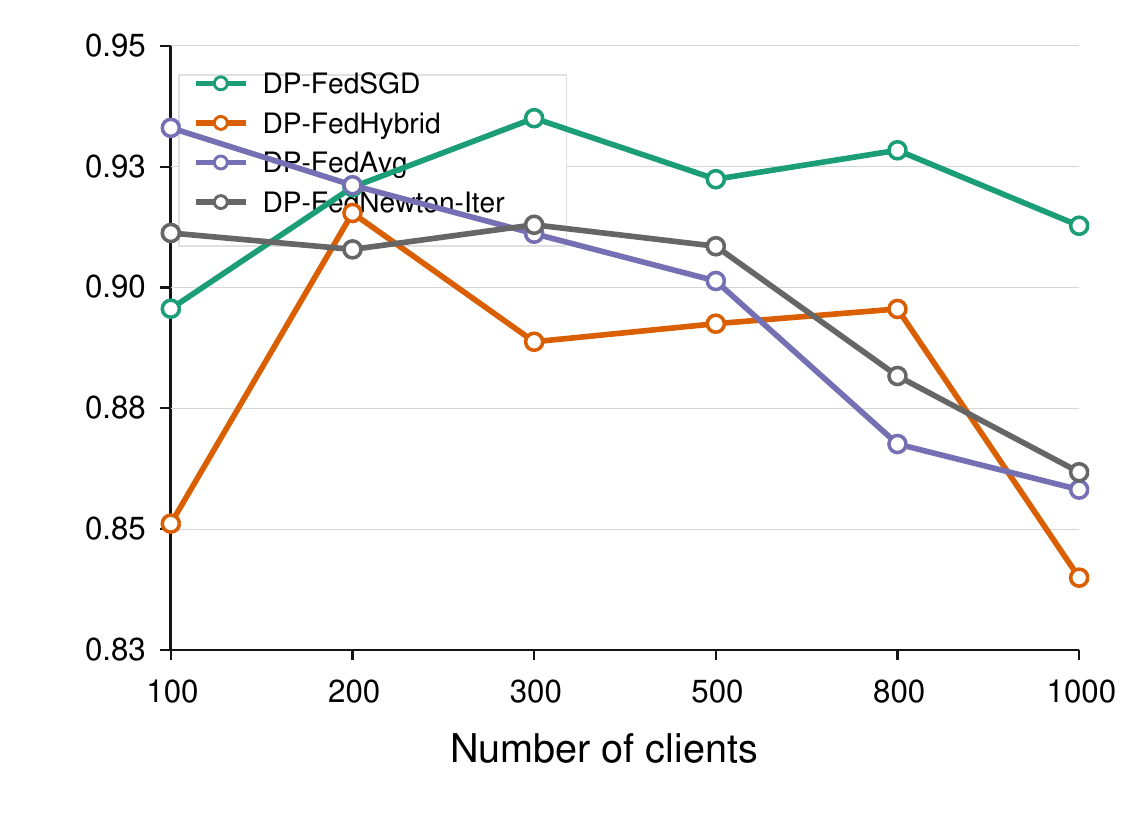}
    \caption{CNN results with increasing number of clients fixing the total sample size: median final test
    accuracy on (left) CIFAR-10 and (right) MNIST.  The medians are computed over four
    independent runs.}
    \label{fig:accuracy-clients}
\end{figure}
\subsubsection{Data and Client Partitioning}

We compare the performance of federated training of CNNs on two popular image classification benchmarks: MNIST and CIFAR-10.   The MNIST dataset has been discussed before. The standard dataset contains 60,000 training images and 10,000 test images.  Because the images are low-dimensional, centered, and visually simple, MNIST is a relatively easy benchmark for convolutional neural networks.

CIFAR-10 \cite{krizhevsky2010cifar} is a more challenging natural image classification benchmark with ten object classes: airplane, automobile, bird, cat, deer, dog, frog, horse, ship, and truck.  Each image has spatial dimension $32 \times 32$ with three RGB color channels making those $32 \times 32 \times 3$ tensors.  The standard split contains 50,000 training images and 10,000 test images.  Compared with MNIST, CIFAR-10 has substantially greater visual variation due to color, texture, background clutter, and changes in orientation, making it a harder problem to classify accurately. 

For the increasing samples in the clients scenario, the training data are split IID across $m=100$ clients and each client's local data are capped at $n \in \{50,75,100,200,300,400,600\}$ samples per client.  Since CIFAR-10 has 50,000 training examples, the maximum sample per client for this dataset is 500.  For the increasing client scenario, we instead use the full training split and vary $
  m \in \{100,200,300,500,800,1000\}$. This keeps the total sample size fixed while changing the average local sample size per client.  In both cases, each client's data is further divided locally into training, validation, and local test portions using 10\% validation and 10\% local test splits.  Local gradient descent updates are full-batch updates, so one local epoch corresponds to one gradient step over the client's local training split. For both scenarios, we repeat the runs a few times and plot the median test accuracies in each case in the figures.

\subsubsection{Models and optimization setup}

We use convolutional neural networks with log-softmax outputs and negative log-likelihood loss for optimization.  The CNN model for MNIST has two convolutional blocks with 32 and 64 channels, followed by a 128-unit fully connected layer and a 10-class classifier layer.  The CIFAR-10 network uses three convolutional blocks with 32, 64, and 128 channels, followed by a 256-unit fully connected layer and a 10-class classifier.

In both scenarios, a fraction $q=0.1$ of clients is sampled per global round. The MNIST runs use 10 base global rounds and 10 local epochs per round, with learning rate $0.01$, server learning rate $1.0$, and clipping bound $B=1.0$. The CIFAR-10 runs use 10 base global rounds and 30 local epochs per round, with learning rate $0.005$, server learning rate $0.75$, and clipping bound $B=0.75$.  The privacy parameter is $\mu=2.0$ in all runs.

The figures compare four DP federated learning methods. These are \FedAvg, a compute-matched \FedSGD, compute-matched \FedHybrid, and \FedNewton-Iter.  The \FedAvg
performs private full-batch local SGD on selected clients.  The \FedSGD~and  \FedHybrid~uses additional communication rounds to match the local computation budget of
\FedAvg.  The \FedNewton-Iter alternates \FedAvg~blocks with private head-only Newton refinements. It uses two cycles, each with five \FedAvg~rounds followed by one Newton refinement. All methods are privatized with the Gaussian noise added as described in the methods and theory.

\subsubsection{Results}

Figure~\ref{fig:accuracy-samples} shows that increasing the number of local
examples generally improves final accuracy, as expected, because both the cost of federation and the effect of
client-level privacy noise decrease with larger local sample sizes.  On MNIST,
\FedAvg~and \FedNewton-Iter are strongest in the small-sample regime, while
compute-matched \FedSGD becomes competitive once clients have a few hundred
examples.  On CIFAR-10, \FedAvg~and \FedNewton-Iter perform best at small to
moderate sample sizes, but compute-matched \FedSGD~improves sharply with more
local data and becomes the strongest method at the largest sample sizes.

Figure~\ref{fig:accuracy-clients} shows the effect of increasing
the number of clients while keeping the total training set fixed.  We note that the performance of all methods deteriorates as the same data is split among more clients due to the cost of federation. We also note that \FedNewton~is either competitive or is better than \FedAvg~at a higher number of clients.

\section{Conclusion}

In this article, we obtained finite sample upper bounds on the estimation error of differentially private federated learning methods for M-estimation. These results allowed us to theoretically study the tradeoffs among accuracy gain and privacy loss as the number of iterations increases, as well as the cost of federation and privacy in federated learning. We developed a minimax lower bound on the error rate of any private federated estimator, which allowed us to further study the optimality gap of various federated learning methods.

Based on these theoretical considerations, we further proposed two methods as improvements on the existing \FedSGD~and \FedAvg~respectively. Our new \FedHybrid~achieves the same level of accuracy as \FedSGD~with fewer communication rounds and hence reduces the communication cost. Our new \FedNewton~retains the communication efficiency of \FedAvg~while improving upon the accuracy of \FedAvg~especially when the number of clients is high.

\section{Acknowledgement}
This research was partially supported by a grant from the NSF (DMS grant 2529302) and a grant from the OSU College of Arts and Sciences. We also gratefully acknowledge the computing credits provided by the Ohio Supercomputer Center.

\bibliographystyle{plainnat}  
\bibliography{references,references_cand} 
\clearpage

\appendix

\section{Proofs of Lemmas and Theorems}

\subsection{Additional definitions and lemmas}

We will require the following definition for privacy accounting.

\begin{definition}[Global Sensitivity]
Let $\mathcal{D}^*$ denote the space of all datasets. Consider two datasets $D, D' \in \mathcal{D}^*$ that differ in exactly one datum. The global sensitivity of a function $f: \mathbb{R}^{n \times m} \rightarrow \mathbb{R}^p$ with respect to a norm $\|\cdot\|$ is defined as
\[
\mathrm{GS}_f
=
\sup_{D, D'}
\| f(D) - f(D') \|.
\]
\end{definition}

\begin{proposition} [Theorem 1 in \cite{dong2022gaussian}]
\label{mugdpdong}
Let $f : \mathbb{R}^{n \times m} \to \mathbb{R}^p$ be a function and assume its global sensitivity $\mathrm{GS}_f < \infty$. Let $Z \sim \mathcal{N}(0, I_p)$ be a $p$-dimensional standard Gaussian random vector. For any $\mu > 0$, define 
\[
h(x) = f(x) + \frac{\mathrm{GS}_f}{\mu} Z,
\quad x \in \mathbb{R}^{n \times m}.
\]
Then h(x) is $\mu$-Gaussian Differentially Private ($\mu$-GDP).
\end{proposition}

\begin{lemma}\label{lem:high-prob-event}
    Let \(\mathcal{E}\) be the event defined as follows:
    \[
    \mathcal{E}
    :=\left\{
    \max_{1\le i\le m}\max_{1\le j\le n_i}
    \left\Vert \nabla\rho(x_j^{(i)},\theta_0)\right\Vert\le B,\,
    \sup_{\theta\in \RR^d}    
    \left\Vert 
    \ddot{L}_i(\theta)-\ddot{L}_0(\theta)
    \right\Vert\le \frac{CH}{\sqrt{n_i}}
    \,\,
    \text{for all }i,
    \,\,
    \sup_{\theta\in \RR^d}    
    \left\Vert 
    \ddot{L}(\theta)-\ddot{L}_0(\theta)
    \right\Vert\le \frac{CH}{\sqrt{N}}
    \right\}
    \]
    for $H=C_1\sqrt{d\vee \log N}$ where $C_1$ is the sub-Gaussian constant in Assumption~\ref{ass:grad}. Then $\P(\mathcal{E}^c)\le c(\exp(-cd)\wedge N^{-4})$ whenever $B\ge C(\sqrt{d}\vee \sqrt{\log N})$.
\end{lemma}

\begin{proof}[Proof of Lemma~\ref{lem:high-prob-event}]
We prove the result for the case where $d>C\log N$ so that we take $B\ge C\sqrt{d}$. The proof for the first part follows by sub-Gaussian tail assumptions on the gradient $\nabla\rho(x_j^{(i)},\theta_0)$. Similarly for the Hessian $\ddot{L}_i(\theta)$, we use a $1/3$-net $\mathcal{S}$ on the unit sphere $\SS^{d-1}$, which has a cardinality at most $4^d$. For each $\bu\in\mathcal{S}$, we again use sub-Gaussian tail assumptions to show
    \[
    \P\left(\bu^{\top}(\ddot{L}_i(\theta)-\ddot{L}(\theta))\bu
    \ge \frac{CH}{\sqrt{n_i}}
    \quad
    \text{for all }i
    \right)\le \exp(-Cd)
    \]
    for a constant $C>0$. The result follows by taking a union bound over all $\bu\in \calS$ and then using a standard $\eps$-net argument, after which we take another union bound using a second $\eps$-net argument over $\{\theta:\|\theta\|\le C\}$ for a sufficiently large constant $C>0$. The proof for the case of small dimensions, i.e., $d\le C\log N$ follows similarly by allowing $B=C\sqrt{\log N}$.
\end{proof}

\subsection{Proof of results in the main paper}

\begin{proof}[Proof of Lemma \ref{lem:AG1muGDP}]
Suppose we run Algorithm \ref{alg:AG1} (AG1) with $K$ gradient iterations at the server. The total privacy budget for each client is assumed to be $\mu$ (i.e., the clients are $\mu$-GDP). Then for each iteration, by the composition theorem of the Gaussian mechanism \cite{dong2022gaussian}, the privacy budget is $\frac{\mu}{\sqrt{K}}$. Also, the server takes a weighted average of the clients' gradients in each of the $K$ iterations. Let us denote the scaling factor of the Gaussian noise as $\sigma/n_i$ for the $i$th client. 

To determine the differential privacy of the client towards the server, we note that the global sensitivity of the function $g_i^{(k)} = \frac{1}{n_i}\sum^{n_i}_{j=1}\nabla\rho\left(x_j^{(i)},\theta^{(k)}\right)$ is $\frac{2B}{n_i}$ from the bounded gradient assumption. Then 
\[
\frac{2B \sqrt{K}}{\mu n_i} = \frac{\sigma}{n_i}.
\]
This gives us $\sigma =\frac{2B \sqrt{K}}{\mu}$ for all clients to be $\mu$-fed-GDP towards the server.

To determine the differential privacy of the server to a third party, we make the following calculation.

\begin{align*}
    \theta^{(k+1)}&= \theta^{(k)}-\eta \tilde{g}^{(k)} \\ &= \theta^{(k)}- \eta\sum^{m}_{i=1}\frac{n_i}{N}\tilde{g}_{i}^{(k)} \\ &= \theta^{(k)}-\eta\sum^{m}_{i=1}\frac{n_i}{N}\left(g_i^{(k)}+\text{noise}_i^{(k)}\right)\\ &= \theta^{(k)}-\eta\sum^{m}_{i=1}\frac{n_i}{N}\left[\frac{1}{n_i}\sum^{n_i}_{j=1}\nabla\rho\left(x_j^{(i)},\theta^{(k)}\right)+\frac{\sigma}{n_i} Z_k^{(i)}\right] \\ &= \theta^{(k)}-\frac{\eta}{N}\sum^{m}_{i=1}\sum^{n_i}_{j=1}\nabla\rho\left(x_j^{(i)},\theta^{(k)}\right)-\frac{\eta}{N}\sum^{m}_{i=1} \sigma\times Z_k^{(i)} \\ &= \theta^{(k)}-\frac{\eta}{N}\sum^{m}_{i=1}\sum^{n_i}_{j=1}\nabla\rho\left(x_j^{(i)},\theta^{(k)}\right)-\frac{\eta \sigma}{N}\sum^{m}_{i=1} Z_k^{(i)},
\end{align*}

Now in the notation of Proposition \ref{mugdpdong}, we denote  $f(x)=\theta^{(k)}-\frac{\eta}{N}\sum^{m}_{i=1}\sum^{n_i}_{j=1}\nabla\rho\left(x_j^{(i)},\theta^{(k)}\right)$. Then,
\begin{align*}
    \text{GS}_f &=\frac{\eta}{N}\left\|\sum^{m}_{i=1}\sum^{n_i}_{j=1}\nabla\rho\left(x_{j}^{(i)},\theta^{(k)}\right)-\sum^{m}_{i=1}\sum^{n_i}_{j=1}\nabla\rho\left(x_{j}^{(i)\prime},\theta^{(k)}\right)\right\| \\ &\le\frac{\eta}{N}\cdot  2B =\frac{2B\eta}{N}
\end{align*}

Note in the above calculation, the global sensitivity is measured by considering two datasets differ by just one data point overall in the global sample (i.e., there is one client where one data point is different).

Also note that $Z_k^{(i)}\sim\mathcal{N}(0,I_d)$ are independent across $i$ and $k$. Then $\sum_{i=1}^m Z_k^{(i)} \sim N(0,mI_d)$.
Let the server be $\mu_1-$GDP. Then, we have the following equality
\begin{align*}
    \frac{\eta \sqrt{m} 2B\sqrt{K}}{\mu N}=\frac{2B\eta\sqrt{K}}{N \mu_1} 
\end{align*}
This implies $\mu_1 = \mu/\sqrt{m}$.
Therefore, the full algorithm is $\mu/\sqrt{m}-$GDP to a third-party with the noise scaling of $\frac{2B\sqrt{K}}{\mu n_i}$ for client $i$.
\end{proof}

\medskip

\medskip

\begin{proof}[Proof of Theorem~\ref{th:alg1}]
  We first note that
\begin{align*}
    \tilde{g}^{(k)} = \sum_{i=1}^{m} w_i \tilde{g}^{(k)}_i = \sum_{i=1}^{m} w_i \left( \nabla L_i(\theta^{(k)}) + 
\frac{2\eta B\sqrt{K}}{\mu n_i} Z^{(i)}_k \right).
\end{align*}  
  
Therefore, writing $\dot{L}_w(\cdot)=\sum w_i \dot{L}_i(\cdot)$ and $\ddot{L}_w(\cdot)=\sum w_i \ddot{L}_i(\cdot)$
\begin{align*}\label{eq:theta-k-taylor}
    &~\theta^{(k+1)}-\theta_0 \\
&= \theta^{(k)}-\theta_0 -\eta \dot{L}_w(\theta^{(k)})
+
\sum_{i=1}^m
\left(w_i
\frac{2\eta B\sqrt{K}}{\mu n_i}\right) Z^{(i)}_k\\
&=\theta^{(k)}-\theta_0 -\eta\left[\dot{L}_w(\theta_0)
+\int_{0}^{1}\ddot{L}_w((1-t)\theta_0+t\theta^{(k)})dt\,(\theta^{(k)}-\theta_0)\right]
+\frac{2\eta B\sqrt{K}}{\mu }
\sum_{i=1}^m
\frac{w_i}{n_i} Z^{(i)}_k\\
&=\left(I-\eta\int_{0}^{1}\ddot{L}_w((1-t)\theta_0+t\theta^{(k)})dt\right)(\theta^{(k)}-\theta_0)-\eta \dot{L}_w(\theta_0)
+\frac{2\eta B\sqrt{K}}{\mu }\sum_{i=1}^m 
\frac{w_i}{n_i} Z^{(i)}_k.\numberthis
\end{align*}
By the $\tau_1$-strong convexity of $\rho$, on the event $\mathcal{E}$ we have
\[
\lambda_{\min}\left(
\int_{0}^{1}\ddot{L}_w((1-t)\theta_0+t\theta^{(k)})dt
\right)
\ge \tau_1-C\sqrt{\frac{H^2}{n_{\min}}}
\]
and hence again on $\mathcal{E}$ we have from \eqref{eq:theta-k-taylor} that 
\[
\|\theta^{(k+1)}-\theta_0\|
\le 
\left(1+C\sqrt{\frac{H^2}{n_{\min}}}-\eta\tau_1\right)\|\theta^{(k)}-\theta_0\|
+
\eta
\left\Vert\dot{L}_w(\theta_0)\right\Vert
+\left\Vert
\frac{2\eta B\sqrt{K}}{\mu  } 
\sum_{i=1}^m
\frac{w_i}{n_i}
Z^{(i)}_k
\right\Vert
\]
Unrolling the above recursion for $k=0,1,\dots,K-1$ we have
\begin{align*}
\|\theta^{(K)}-\theta_0\|
\le &~
\gamma_1^{K}\|\theta^{(0)}-\theta_0\|\\
&~+
\left(
\gamma_1^{K-1}+\gamma_1^{K-2}
+\dots+
1
\right)
\left\{
\eta
\left\Vert\dot{L}_w(\theta_0)\right\Vert
+\left\Vert
\frac{2\eta B\sqrt{K}}{\mu  } 
\sum_{s=1}^m
\frac{w_i}{n_i}
Z^{(s)}_k
\right\Vert\right\}
\end{align*}
where $\gamma_1=\left(1-\eta\left(\tau_1-
C\sqrt{\frac{H^2}{n_{\min}}}
\right)\right)$. Squaring and moving to expectations one obtains:
\begin{align*}
    \EE(\|\theta^{(K)}-\theta_0\|^2\mathbf{1}_{\mathcal{E}})
    \le &~
    3\gamma_1^{2K}\EE\|\theta^{(0)}-\theta_0\|^2\\
    &~+
    \frac{3(1-\gamma_1^K)^2}{(1-\gamma_1)^2}
    \left(
    \eta^2
\EE\left\Vert\dot{L}_w(\theta_0)\right\Vert^2
+\EE\left\Vert
\frac{2\eta B\sqrt{K}}{\mu N} 
\sum_{s=1}^m 
\frac{w_s}{n_s} Z^{(s)}_k
\right\Vert^2
    \right)
\\
\le&~
3\gamma_1^{2K}\EE\|\theta^{(0)}-\theta_0\|^2
+
\frac{3}{(1-\gamma_1)^2}
\left(
\eta^2
\EE\left\{\left\Vert\dot{L}(\theta_0)\right\Vert^2\mathbf{1}_{\mathcal{E}}\right\}
+\frac{\eta^2B^2Kd}{\mu^2}
\sum_{s=1}^m
\frac{w_s^2}{n_s^2}
\right)
\end{align*}

Since $\eta\ge \frac{1}{2\tau_2}$, for sufficiently large $N$, we have $\gamma_1\le 1-\tau_1/3\tau_2$, so that we have the rate:
\begin{align*}
        \mathbb{E}\left[\left\|\theta^{(K)}-\theta_0\right\|^2
        \mathbf{1}_{\mathcal{E}}
        \right]
    \le&~
    3\left(1-\frac{\tau_1}{3\tau_2}\right)^{2K}\EE\|\theta^{(0)}-\theta_0\|^2
+
\frac{3}{(1-\gamma_1)^2}
\left(
\eta^2
\EE\left\Vert\dot{L}_w(\theta_0)\right\Vert^2
+\frac{\eta^2B^2Kd}{\mu^2}
\sum_{s=1}^m
\frac{w_s^2}{n_s^2}
\right).
\end{align*}
To get the final rates, we choose $K=\frac{2\log(d N)}{\log((1-\tau_1/3\tau_2)^{-1})}$ so that the first term is $O(N^{-2})$, assuming $\mathbb{E}\left[\left\|\theta^{(0)}-\theta_0\right\|^2\right]=O(d)$.

For the second term, note that \begin{align*}
\mathbb{E}\!\left[\left\|\eta \dot{L}_w(\theta_0)\right\|^2
\mathbf{1}_{\mathcal{E}}
\right]
&\le \eta^2\,\mathbb{E}\!\left[\left\|
\sum_{i=1}^{m}
\frac{w_i}{n_i}
\sum_{j=1}^{n_i}\nabla \rho\!\left(x_j^{(i)},\theta_0\right)\right\|^2\right] \\
&= \eta^2\sum_{i=1}^{m}
\frac{w_i^2}{n_i^2}
\sum_{j=1}^{n_i}\mathbb{E}\!\left[\left\|\nabla \rho\!\left(x_j^{(i)},\theta_0\right)\right\|^2\right]
\quad \text{(since i.i.d.)} \\
&= \eta^2 
\sum_{i=1}^{m}
\frac{w_i^2}{n_i}
\cdot {\rm trace}
\left(\Var\left(\nabla \rho\!\left(x_j^{(i)},\theta_0\right)
\right)\right) \\
&= \eta^2 {\rm trace}(\Sigma)
\sum_{i=1}^{m}
\frac{w_i^2}{n_i}.
\end{align*}
Thus the overall rate becomes:
\begin{align*}
        \mathbb{E}\left[\left\|\theta^{(K)}-\theta_0\right\|^2\mathbf{1}_{\mathcal{E}}\right]
    \le&~
    \frac{4}{(1-\gamma_1)^2}
    \left(
    \eta^2 {\rm trace}(\Sigma)
    \sum_{i=1}^{m}
    \frac{w_i^2}{n_i}
    +\frac{\eta^2B^2Kd}{\mu^2}
    \sum_{i=1}^m
    \frac{w_i^2}{n_i^2}\right)\\
    \le&~ 
    \frac{16\eta^2}{\tau_1^2}
    \left(
    {\rm trace}(\Sigma)
    \sum_{i=1}^{m}
    \frac{w_i^2}{n_i}
    +\frac{ B^2Kd}{\mu^2}
    \sum_{i=1}^m
    \frac{w_i^2}{n_i^2}\right)\\
    =&~
    \frac{16\eta^2}{\tau_1^2}
    \sum_{i=1}^m w_i^2
    \left(
    \frac{{\rm trace}(\Sigma)}{n_i}
    +\frac{ B^2Kd}{\mu^2 n_i^2}\right)
\end{align*}
where the second last line uses the assumption that $1-\gamma_1=\eta\tau_1-C\eta\sqrt{\frac{H^2}{N}}\ge\frac{1}{2}\eta\tau_1$.

Note that the above MSE bound holds for any $\textbf{w}
:=\{(w_1,\dots,w_m)\in [0,1]^m:\sum w_i=1\}$. We can thus minimize over all possible $\textbf{w}$ to find
\[
\hat{w}_i\propto \left(
    \frac{{\rm trace}(\Sigma)}{n_i}
    +\frac{ B^2Kd}{\mu^2 n_i^2}\right)^{-1}
\]
and the corresponding MSE
\[
 \mathbb{E}\left[\left\|\theta^{(K)}-\theta_0\right\|^2\mathbf{1}_{\mathcal{E}}\right]
 \le 
 \frac{16\eta^2}{\tau_1^2}
 \left(
 \sum_{i=1}^m
 \left(
    \frac{{\rm trace}(\Sigma)}{n_i}
    +\frac{ B^2Kd}{\mu^2 n_i^2}\right)^{-1}
 \right)^{-1}.
\]

To finish the proof we write
\begin{align*}
    \mathbb{E}\left[\left\|\theta^{(K)}-\theta_0\right\|^2\mathbf{1}_{\mathcal{E}^c}\right]
    \le&~ 2\mathbb{E}\left[\left\|\theta^{(K)}
    -\theta^{(0)}
    \right\|^2\mathbf{1}_{\mathcal{E}^c}\right]
    +2\EE\|\theta^{(0)}-\theta_0\|^2\P(\mathcal{E}^c)\\
     \le&~ 2\left(\mathbb{E}\left\|
     \theta^{(0)}
     +\sum_{k=1}^K(\theta^{(k)}-\theta^{(k-1)})\right\|^4\right)^{1/2}
    \sqrt{\P({\mathcal{E}^c})}
    +2\EE\|\theta_0-\theta^{(0)}\|^2\P(\mathcal{E}^c)\\
     \le&~2CK\left(\frac{d}{N}+
     \frac{mB^2dK}{N^2\mu^2}
     \right)\sqrt{\frac{C}{N^4}}
     +\frac{Cd}{N^4}
\end{align*}
for some constant $C>0$. Here we use the triangle inequality at each gradient descent iteration and bound the fourth moment of the gradient sum and noise added due to privacy. Thus the MSE contribution from $\mathbb{E}\left[\left\|\theta^{(K)}-\theta_0\right\|^2\mathbf{1}_{\mathcal{E}^c}\right]$ contributes a smaller order term, so that the final rate is given by the term reported in the theorem statement.
\end{proof}

\medskip

\begin{proof}[
Proof of Lemma \ref{lem:AG2muGDP} ]

In stage 1, the clients perform the following iterations
\[\theta_i^{(t+1)}=\theta_i^{(t)}-\dfrac{\eta_1}{n_i}\sum_{j=1}^{n_i} g(x_j^{(i)},\theta_i^{(t)}) + a_i Z_{it},\quad
Z_{it}\sim\mathcal{N}(0,I_d) ,\quad
a_i=\dfrac{2B\eta_1\sqrt{2K_{1i}}}{\mu n_i}.
\]
As discussed before $GS_{g_i}=\frac{2B}{n_i}$.

Let $\mu_i$ be the DP for client $i$ in stage 1. Then, 
\[
\frac{2B\eta_1\sqrt{K_{1i}}}{n_i \mu_i} = \frac{2B\eta_1\sqrt{2K_{1i}}}{\mu n_i},
\]
implies $\mu_i = \frac{\mu}{\sqrt{2}}$.
In stage 2, we run $K_2$ steps of AG1 with $\bar\theta$ as initializer and noise multiplier
$b'=\frac{2B\sqrt{2K_2}}{\mu n_i}$.
From the proof of Lemma~\ref{lem:AG1muGDP}, we have GS as $\frac{2B}{n_i}$. Hence
\[
\frac{2B\sqrt{K_2}}{n_i \mu_i}=\frac{2B\sqrt{2K_2}}{\mu  n_i} \Rightarrow \mu_i = \frac{\mu}{\sqrt{2}}.
\]
Therefore, by the composition theorem in \cite{dong2022gaussian}, the clients are $\mu$-GDP to the server.

It follows immediately from the Proof of Lemma~\ref{lem:AG1muGDP}, that for stage 2, the server's DP to a third party of $\mu/\sqrt{2m}$. For stage 1, following the proof of Lemma \ref{lem:AG1muGDP}, we obtain a DP bound of $\mu/\sqrt{2m}$. Therefore, the total GDP privacy for the server to a third party can be bounded by
\[
\frac{\mu}{\sqrt{m}}.
\]
\end{proof}

\medskip
\begin{proof}[Proof of Theorem~\ref{thm:one-step-fedavg}]
The proof follows the technique for bounding the MSE of $\theta_{\rm AG1}$ by replacing $(L,\dot{L},\ddot{L})$ by $(L_i,\dot{L}_i,\ddot{L}_i)$ to obtain:
\begin{equation}\label{eq:MSE-theta-i}
        \mathbb{E}\left[\left\|\theta^{(k)}_i-\theta_0\right\|^2\right]
    \le ~ 
    \frac{16}{\tau_1^2}
    \left(
    \frac{{\rm trace}(\Sigma)}{n_i}
    +\frac{ dB^2K_1}{\mu^2n_i^2}\right)
    =: r_i
    \,\,
    \text{for all }
    \,\,
    k\ge K_{1i}/2
    \,\,
    \text{and}
    \,\,
    i=1,\dots,m
    .
\end{equation}
The next step will be to follow the proof of Lemma 23 from \cite{huang2019distributed} to bound the bias of $\theta^{(K_1)}_i$. Modifying \eqref{eq:theta-k-taylor} for the $i$-th client we have:
\begin{align*}
    \theta^{(K_1)}_i-\theta_0
    =&~
    \left(I-\eta\int_{0}^{1}\ddot{L}_i((1-t)\theta_0+t\theta^{(K_1-1)})dt\right)(\theta^{(K_1-1)}-\theta_0)\\
    &~-\eta \dot{L}_i(\theta_0)
    +\frac{2\eta B\sqrt{K_1}}{\mu n_i} Z^{(s)}_k
\end{align*}
so that
\begin{align*}
    \EE(\theta^{(K_1)}_i-\theta_0)
    =&~
    \EE\left(
    \left(I-\eta\int_{0}^{1}\ddot{L}_i((1-t)\theta_0+t\theta^{(K_1-1)}_i)dt\right)(\theta^{(K_1-1)}_i-\theta_0)
    \right)\\
    =&~
    \EE\left(
    \left(I-\eta \ddot{L}_i(\theta_0) \right)(\theta^{(K_1-1)}_i-\theta_0)
    \right)\\
    &~+
    \eta 
    \EE\left(
    \left(\int_{0}^{1}
    [\ddot{L}_i(\theta_0)
    -
    \ddot{L}_i((1-t)\theta_0+t\theta^{(K_1-1)}_i)]dt\right)(\theta^{(K_1-1)}_i-\theta_0)
    \right)\\
    =&~
    \left(I-\eta \ddot{L}_0(\theta_0) \right)
    \EE(\theta^{(K_1-1)}_i-\theta_0)
    +
    \eta 
    \EE\left(
    \left(\ddot{L}_0(\theta_0)- \ddot{L}_i(\theta_0) \right)
    (\theta^{(K_1-1)}_i-\theta_0)
    \right)
    \\
    &~+
    \eta 
    \EE\left(
    \left(\int_{0}^{1}
    [\ddot{L}_i(\theta_0)
    -
    \ddot{L}_i((1-t)\theta_0+t\theta^{(K_1-1)}_i)]dt\right)(\theta^{(K_1-1)}_i-\theta_0)
    \right)\\
    =:&~
    \left(I-\eta \ddot{L}_0(\theta_0) \right)
    \EE(\theta^{(K_1-1)}_i-\theta_0)
    +\eta T_1+\eta T_2.
\end{align*}
This implies
\begin{align*}
    (I-I+\eta \ddot{L}_0(\theta_0) )\EE(\theta^{(K_1)}_i-\theta_0)
    =&~
    \left(I-\eta \ddot{L}_0(\theta_0) \right)
    \EE(\theta^{(K_1-1)}_i-\theta^{(K_1)}_i)
    +\eta T_1+\eta T_2
\end{align*}
which by the definition of privatized gradient descent implies:
\begin{align*}
    \ddot{L}_0(\theta_0) \EE(\theta^{(K_1)}_i-\theta_0)
    =&~
    \left(I-\eta \ddot{L}_0(\theta_0) \right)
    \EE(\dot{L}_i(\theta^{(K_1-1)}_i))
    + T_1+ T_2\\
    =&~
    \left(I-\eta \ddot{L}_0(\theta_0) \right)
    \left(
    \EE(\dot{L}_i(\theta_0))
    +
    \EE \ddot{L}_i(\theta_0)(\theta^{(K_1-1)}_i-\theta_0)
    \right)
    +\Xi
    + T_1+ T_2\\
    =&~
    \left(I-\eta \ddot{L}_0(\theta_0) \right)
    \left(
    \mathbf{0}
    +
    \EE \ddot{L}_0(\theta_0)(\theta^{(K_1-1)}_i-\theta_0)
    \right)
    +\Xi\\
    &~- \left(I-\eta \ddot{L}_0(\theta_0) \right) T_1
    + T_1+ T_2\\
    =&~
    \left(I-\eta \ddot{L}_0(\theta_0) \right)
    \left(
     \ddot{L}_0(\theta_0)
     \EE(\theta^{(K_1-1)}_i-\theta_0)
    \right)
    +\Xi\\
    &~+\eta \ddot{L}_0(\theta_0) T_1
    + T_2
\end{align*}
where the Lipschitz condition on $\ddot{L}_0(\theta_i^{(K_1-1)})$ and convergence of $\ddot{L}_i(\theta)$ to $\ddot{L}_0(\theta)$ for $\|\theta-\theta_0\|\le c_0$ yields
\begin{align*}\label{eq:xi-bound}
    \|\Xi\|
    =&~\|\EE\ddot{L}_i\left(
    \int_0^1 t\theta_0+(1-t)(\theta_i^{(K_1-1)}))-\ddot{L}_i(\theta_0)
    dt
    \right)(\theta_i^{(K_1-1)}-\theta_0)\|\\
    \le&~
    2\max_{t\in [0,1]}
    \|\EE
    (\ddot{L}_i(t\theta_0+(1-t)(\theta_i^{(K_1-1)}))
    -
    \ddot{L}_0(t\theta_0+(1-t)(\theta_i^{(K_1-1)}))
    )
    (\theta_i^{(K_1-1)}-\theta_0)
    \|+\\
    &~+
    \|\EE(\ddot{L}_0(t\theta_0+(1-t)(\theta_i^{(K_1-1)}))-\ddot{L}_0(\theta_0))(\theta_i^{(K_1-1)}-\theta_0)\|\\
    \le
    &~\frac{2CH}{\sqrt{n_i}}\times \sqrt{r_i}
    +
    C_{\rm Lip}(\EE\|\theta^{(K_1-1)}_i-\theta_0\|^2)\\
    \le&~ C_1 r_i.\numberthis
\end{align*}
for a constant $C_1>0$, where in the last step we used the assumption that $H^2\le C_*{\rm trace}(\Sigma)$ for some constant $C_*>0$.

Note that by $\tau_1$ strong convexity of $\ddot{L}_0(\cdot)$,
\[
\ddot{L}_0(\theta_0)\succ \tau_1 I_d,
\,\,
\left\Vert I-\eta \ddot{L}_0(\theta_0) \right\Vert
\le 1-\eta\tau_1.
\]
Then following the recursion argument of the proof for $\theta_{\rm AG1}$ we have
\begin{align*}\label{eq:bias-local-est}
    \left\Vert \EE(\theta^{(K_1)}_i-\theta_0)\right\Vert
    \le&~
    \frac{\tau_2}{\tau_1}
    (1-\eta\tau_1)^{K_1/2} 
    \left\Vert \EE(\theta^{(K_1/2)}_i-\theta_0)\right\Vert
    +\frac{3}{\eta\tau_1^2}
    \left(
     C_{\rm Lip} r_i
     +\eta\tau_2\|T_1\|+\|T_2\|
    \right)\\
    \le&~
    \frac{\tau_2}{\tau_1}
    \times \frac{\sqrt{r_i}}{N^2}
    +\frac{3}{\eta\tau_1^2}
    \left(
     C_{\rm Lip} r_i
     +\eta\tau_2\|T_1\|+\|T_2\|
    \right)\numberthis
\end{align*}
To bound $\|T_1\|$ note that $\ddot{L}_0(\theta_0)-\ddot{L}_i(\theta_0)$ is independent of $\{Z_k^{(i)}:1\le k\le K_1\}$. Then we can unroll the recursion from \eqref{eq:theta-k-taylor} to write:
\begin{align*}
    &~\|T_1\|\\
    =&~
    \left\Vert
    \EE\left(
    \left(
    \ddot{L}_0(\theta_0)-\ddot{L}_i(\theta_0)
    \right)
    \left(
    \theta_i^{(K_1-1)}-\theta_0
    \right)
    \right)
    \right\Vert\\
    \le&~
    \left\Vert
    \EE\left(
    \left(
    \ddot{L}_0(\theta_0)-\ddot{L}_i(\theta_0)
    \right)
    \left(
    \sum_{k={K_1/2}}^{K_1-2}
    (I-\eta\ddot{L}(\theta_0))^{K_1-1-k}
    (\eta \dot{L}(\theta_0)+4(\eta B\sqrt{K_1}/\mu n_i)Z_k^{(i)})
    \right)
    \right)
    \right\Vert\\
    &~+O(\EE\|\ddot{L}_0(\theta_0)-\ddot{L}_i(\theta_0)\|\cdot \EE \|\theta^{(K_1/2)}-\theta_0\|^2)\\
    \le&~
    \left\Vert
    \EE\left(
    \left(
    \ddot{L}_0(\theta_0)-\ddot{L}_i(\theta_0)
    \right)
    \left(
    \sum_{k={K_1/2}}^{K_1-2}
    (I-\eta\ddot{L}_i(\theta_0))^{K_1-1-k}
    (\eta \dot{L}_i(\theta_0)
    \right)
    \right)
    \right\Vert\\
    &~+O((\EE\|\ddot{L}_0(\theta_0)-\ddot{L}_i(\theta_0)\|^2\cdot \EE \|\theta^{(K_1/2)}-\theta_0\|^4)^{1/2})\\
    \le&~
    \eta 
    \left(
    \EE 
    \left\Vert
    \ddot{L}_0(\theta_0)-\ddot{L}_i(\theta_0)
    \right\Vert^2
    \right)^{1/2}
    \left(
    \EE 
    \|\dot{L}_i(\theta_0)\|^2
    \right)^{1/2}
    +O((\EE\|\ddot{L}_0(\theta_0)-\ddot{L}_i(\theta_0)\|^2\cdot \EE \|\theta^{(K_1/2)}-\theta_0\|^4)^{1/2})\\
    \le&~
    \frac{1+r_i}{\tau_2}\cdot
    \dfrac{C{\rm trace}(\Sigma)}{n_i}.
\end{align*}
Next by Lipschitzness assumption on the Hessian $\ddot{L}_0$ and by uniform convergence of $\ddot{L}_i(\theta)$ to $\ddot{L}_0(\theta)$ for $\|\theta-\theta_0\|\le c_0$ we have:
\begin{align*}
\|T_2\|
=&~\|\EE\left(
    \left(\int_{0}^{1}
    [\ddot{L}_i(\theta_0)
    -
    \ddot{L}_i((1-t)\theta_0+t\theta^{(K_1-1)}_i)]dt\right)(\theta^{(K_1-1)}_i-\theta_0)
    \right)\|\\
\le&~ 2C_1\EE\|\theta_i^{(K_1-1)}-\theta_0\|^2
=2C_1r_i
\end{align*}
via an argument similar to \eqref{eq:xi-bound}. Plugging in the above bounds into \eqref{eq:bias-local-est} one obtains:
\begin{equation}\label{eq:bias-local-est-final}
    \left\Vert
    \EE(\theta_i^{(K_1)}-\theta_0)
    \right\Vert
    \le 
    \frac{4}{\tau_1}
    \left(
    3C_1r_i
    +
    \frac{(1+r_i)}{\tau_1}\cdot\frac{C{\rm trace}(\Sigma)}{n_i}
    \right).
\end{equation}
Combining the bias bounds for $i=1,\dots,m$ we then have
\begin{align*}\label{eq:bias-fed-avg-1step}
    \left\Vert
    \EE(\Bar{\theta}-\theta_0)
    \right\Vert
    =&~
    \left\Vert
    \sum_{i=1}^m
    \frac{n_i}{N}
    \EE(\theta_i-\theta_0)
    \right\Vert\\
    \le&~
    \frac{12C_1}{N\tau_1}
    \sum_{i=1}^m n_ir_i
    +
    \frac{5Cm{\rm trace}(\Sigma)}{\tau_1^2N}\\
    \le&~
    \frac{192C_1}{\tau_1^3}
    \left(
    \frac{m{\rm trace}(\Sigma)}{N}
    +
    \frac{B^2d}{\mu^2N}\sum_{i=1}^m\frac{K_{1i}}{n_i}
    \right)
    +
    \frac{5Cm{\rm trace}(\Sigma)}{\tau_1^2N}
\end{align*}
where in the first inequality we use the assumption that $n_i$ are large enough so that $r_i\le 0.25$ for all $1\le i\le m$. For the variance note that $\Bar{\theta}$ is an average of M-estimators trained on independent datasets, so that:
\begin{align*}
    {\rm trace}(\Var(\Bar{\theta}))
    =\frac{1}{N^2}\sum_{i=1}^mn_i^2
    {\rm trace}(\Var(\theta_i))
    \le~ \frac{1}{N^2}\sum_{i=1}^mn_i^2r_i
    =
    \frac{16}{\tau_1^2}
    \left(
    \frac{{\rm trace}(\Sigma)}{N}
    +
    \frac{mdB^2K_1}{\mu^2N^2}
    \right).
\end{align*}
Thus the MSE for $\Bar{\theta}$ becomes
\begin{align*}
    \EE\|\Bar{\theta}-\theta_0\|^2
    =&~
    {\rm trace}(\Var(\Bar{\theta}))+
    \left\Vert
    \EE(\Bar{\theta}-\theta_0)
    \right\Vert^2\\
    \le&~
    \frac{16}{\tau_1^2}
    \left(
    \frac{{\rm trace}(\Sigma)}{N}
    +
    \frac{mdB^2K_{\max}}{\mu^2N^2}
    \right)
    +
    C_2^2
    \left(
    \frac{m{\rm trace}(\Sigma)}{N}
    +
    \frac{dB^2}{\mu^2N}\sum_{i=1}^m\frac{K_{1i}}{n_i}
    \right)^2.
\end{align*}
\end{proof}

\medskip

\begin{proof}[Proof of Lemma \ref{lem:AG3muGDP}]
    We consider the \FedAvg~(AG3) method to be run for $R$ rounds of communication, and each round consists of $K$ local gradient steps. At round $r$, the local estimates from the clients are denoted as $\theta_{i,r}^{(K)}$. According to the \FedAvg~algorithm, the server then aggregates with weights $w_i = \frac{n_i}{N}$ to obtain $\theta^{(r)} = \sum_{i=1}^m w_i \theta_{i,r}^{(K)}$.

Let $D$ and $D'$ be neighboring global datasets differing in exactly one record. Let $i^\star$ denote the client whose dataset differs.
For ease of notation, we denote, the each client empirical risk as
\[
F_i(\theta)=\frac{1}{n_i}\sum_{j=1}^{n_i} \rho(x_j^{(i)},\theta).
\]

We couple the randomness by using the \emph{same} Gaussian noises $\{Z_{it}\}$ under $D$ and $D'$.
Then for all other clients, $i\neq i^\star$, the local iterates are identical (since the injected noise is also identical). Hence only client $i^\star$ contributes to sensitivity of the quantity server releases. 

Let $\theta^{(t)}$ and $\theta^{'(t)}$ denote the parameter values at iteration $t$, for client $i^\star$ under $D$ and $D'$ respectively.
The local update is

\[
\theta^{(t+1)}
=
\theta^{(t)} - \eta \nabla F_{i^\star}(\theta^{(t)}) + \sigma_{i^\star} Z_{t}.
\]

Since the noise is coupled, it cancels in the difference. Now we compute,

\[
\theta^{(t+1)} - \theta^{'(t+1)}
=
\theta^{(t)} - \theta^{'(t)}
-
\eta
\big(
\nabla F_{i^\star}(\theta^{(t)})
-
\nabla F_{i^\star}(\theta^{'(t)})
\big)
+
\eta
\big(
\nabla F_{i^\star}(\theta^{'(t)};D_{i^\star})
-
\nabla F_{i^\star}(\theta^{'(t)};D'_{i^\star})
\big).
\]
In the above expression, the second term is due to difference in gradients evaluated at $\theta^t$ and $\theta^{'(t)}$ respectively, while the third term is due to the gradients differing between $D_{i^\star}$ and  $D'_{i^\star}$.

By Assumption 2 on $\tau_1$ strong convexity and Assumption 4 on $\tau_2$ smoothness of $F_{i^\star}$, we have the following 
\begin{align}
(\nabla F_{i^\star}(u)-\nabla F_{i^\star}(v))^\top (u-v)
&\ge \tau_1 \|u-v\|_2^2, \label{eq:strong-convexity-grad}
\\
\|\nabla F_{i^\star}(u)-\nabla F_{i^\star}(v)\|_2
&\le \tau_2 \|u-v\|_2. \label{eq:smoothness-grad}
\end{align}

We now show that the gradient descent map $T$ is contractive.
For any $u,v\in\mathbb{R}^d$,
\begin{align*}
\|T(u)-T(v)\|_2^2
&=
\|u-v-\eta(\nabla F_{i^\star}(u)-\nabla F_{i^\star}(v))\|_2^2
\\
&=
\|u-v\|_2^2
+\eta^2\|\nabla F_{i^\star}(u)-\nabla F_{i^\star}(v)\|_2^2
-2\eta (\nabla F_{i^\star}(u)-\nabla F_{i^\star}(v))^\top (u-v).
\end{align*}
Using \eqref{eq:strong-convexity-grad} and \eqref{eq:smoothness-grad}, we obtain
\begin{align*}
\|T(u)-T(v)\|_2^2
&\le
\|u-v\|_2^2 + \eta^2 \tau_2^2 \|u-v\|_2^2 - 2\eta\tau_1 \|u-v\|_2^2
\\
&=
\bigl(1-2\eta\tau_1+\eta^2\tau_2^2\bigr)\|u-v\|_2^2.
\end{align*}

Then we can write
\[
\|T(u)-T(v)\|_2 \leq C_{\eta} \|u-v\|_2,
\]
with suitable assumptions on the step size $\eta$ such that $C_\eta<1$.

Changing one record changes the empirical gradient by at most $2B/n_{i^\star}$,
so

\[
\|
\nabla F_{i^\star}(\theta;D_{i^\star})
-
\nabla F_{i^\star}(\theta;D'_{i^\star})
\|
\le
\frac{2B}{n_{i^\star}}.
\]

Hence, defining $\delta_t := \|\theta^{(t)}-\theta^{'(t)}\|$,

\[
\delta_{t+1}
\le
C_{\eta}\delta_t + \eta \frac{2B}{n_{i^\star}}.
\]

Since $\delta_0=0$, we obtain

\[
\delta_K
\le
\eta \frac{2B}{n_{i^\star}}
\sum_{s=0}^{K-1} (C_{\eta})^s .
\]

Thus

\[
\mathrm{GS}(\theta_{i^\star}^{(K)})
\le
\eta \frac{2B}{n_{i^\star}}
\sum_{s=0}^{K-1} C_{\eta}^s .
\]

Since only client $i^\star$ differs, the global sensitivity at the server is,

\[
\mathrm{GS}(\theta^{(r)})
\le
w_{i^\star}
\mathrm{GS}(\theta_{i^\star}^{(K)})
=
\frac{n_{i^\star}}{N}
\eta \frac{2B}{n_{i^\star}}
\sum_{s=0}^{K-1}C_{\eta}^s = \frac{2B\eta}{N}
\sum_{s=0}^{K-1}C_{\eta}^s  .
\]

Now we turn to the noise propagation. Each noise term injected at step $s$ propagates through the remaining gradient descent iterations. Since the gradient descent map $T(\theta)=\theta-\eta\nabla F(\theta)$ is $C_{\eta}$ contractive, the contribution of a noise vector injected $s$ steps earlier is attenuated by at most $C_{\eta}^{s}$. Therefore the final iterate contains an effective Gaussian component with variance proportional to
\[
\sigma_i^2 \sum_{s=0}^{K-1} C_{\eta}^{2s}.
\]
After aggregation the noise standard deviation is
\[
\mathrm{sd}_{\mathrm{round}}
=
\sqrt{\sum_{i=1}^m (w_i \sigma_i)^2}
\sqrt{\sum_{s=0}^{K-1}(C_{\eta})^{2s}}.
\]
Noting, $
w_i\sigma_i
=
\frac{2B\eta\sqrt{RK}}{\mu N},$ we obtain
\[
\mathrm{sd}_{\mathrm{round}}
=
\frac{2B\eta\sqrt{RK}\sqrt{m}}{\mu N}
\sqrt{
\sum_{s=0}^{K-1}(C_{\eta})^{2s}}.
\]
By the Gaussian mechanism characterization,
\[
\mu_{\mathrm{round}}
\le
\frac{\mathrm{GS}_{\mathrm{round}}}{\mathrm{sd}_{\mathrm{round}}}
=
\frac{\mu}{\sqrt{R m}}
\frac{\sum_{s=0}^{K-1}(C_{\eta})^s}
{\sqrt{K \sum_{s=0}^{K-1}(C_{\eta})^{2s}}}.
\]
We further note,
\[
\frac{\sum_{s=0}^{K-1}(C_{\eta})^s}
{\sqrt{K \sum_{s=0}^{K-1}(C_{\eta})^{2s}}} \leq 1,
\]
by Cauchy-Schwarz inequality. Thus
\[
\mu_{\mathrm{round}}
\le
\frac{\mu}{\sqrt{Rm}} .
\]
By composition for $\mu$-GDP,
\[
\mu_{\mathrm{srv}}
\le
\sqrt{R}\,\mu_{\mathrm{round}}
=
\frac{\mu}{\sqrt{m}} .
\]
\end{proof}

\medskip

\begin{proof}[Proof of Theorem~\ref{thm:ag3_mse}]
Fix the total number of communication rounds $R\in\mathbb{N}$.
Under Algorithm \ref{alg:AG3}, in every round the Gaussian noise multiplier used by client $i$ is
\[
\sigma_i \;=\; \frac{2B\eta\sqrt{RK}}{\mu n_i}, \qquad i=1,\dots,m,
\]
which depends on the fixed value of $R$ but does not depend on the round index $r$.

We prove by induction on $r\in\{1,\dots,R\}$ that
\begin{align}
\label{eq:ag3_bound_no_sigfield}
\mathbb{E}\!\left[\left\|\theta^{(r)}-\theta_0\right\|^2\right]
\;\le\;
\frac{16}{\tau_1^2}
\left(
\frac{{\rm trace}(\Sigma)}{N}
+
\frac{2mB^2RKd}{\mu^2N^2}
\right)
+
\frac{C}{N^2}
\left(
m\,{\rm trace}(\Sigma)
+
\sum_{i=1}^m
\frac{B^2RKd}{\mu^2n_i}
\right)^2.
\end{align}
Since $\hat{\theta}_{\mathrm{(AG3)}}=\theta^{(R)}$, the theorem follows by taking $r=R$.

\paragraph{Base case ($r=1$).}
In round $1$, all clients start from $\theta^{(0)}$, perform $K$ local iterations, and the server aggregates
\[
\theta^{(1)}=\sum_{i=1}^m \frac{n_i}{N}\,\theta_i^{(K,1)}.
\]
This estimator has the same form as $\bar{\theta}$ in Lemma~3.3 with $K_{1i}=K$ and noise multipliers
$\sigma_i=\frac{2B\eta\sqrt{2RK}}{\mu n_i}$.
Applying Theorem \ref{thm:one-step-fedavg} yields \eqref{eq:ag3_bound_no_sigfield} for $r=1$.

\paragraph{Inductive step.}
Assume that \eqref{eq:ag3_bound_no_sigfield} holds for some $r\in\{1,\dots,R-1\}$.
We show that it also holds for $r+1$.

At the beginning of round $r+1$, all clients are initialized at the common value $\theta^{(r)}$,
perform $K$ local iterations, and the server aggregates
\[
\theta^{(r+1)}=\sum_{i=1}^m \frac{n_i}{N}\,\theta_i^{(K,r+1)}.
\]
Consequently, conditional on the shared initialization $\theta^{(r)}$,
$\theta^{(r+1)}$ has the same form as the estimator $\bar{\theta}$ in
Theorem \ref{thm:one-step-fedavg} with noise multipliers
\[
\sigma_i = \frac{2B\eta\sqrt{RK}}{\mu n_i}.
\]

Applying Theorem \ref{thm:one-step-fedavg} at round $r+1$, we obtain
\begin{align*}
\mathbb{E}\!\left[\left\|\theta^{(r+1)}-\theta_0\right\|^2 \,\middle|\, \theta^{(r)}\right]
\;\le\;
\frac{16}{\tau_1^2}
\left(
\frac{{\rm trace}(\Sigma)}{N}
+
\frac{2mB^2RKd}{\mu^2N^2}
\right)
+
\frac{C}{N^2}
\left(
m\,{\rm trace}(\Sigma)
+
\sum_{i=1}^m
\frac{B^2RKd}{\mu^2n_i}
\right)^2.
\end{align*}

By the law of total expectation,
\[
\mathbb{E}\!\left[\left\|\theta^{(r+1)}-\theta_0\right\|^2\right]
=
\mathbb{E}\!\left[
\mathbb{E}\!\left[\left\|\theta^{(r+1)}-\theta_0\right\|^2 \,\middle|\, \theta^{(r)}\right]
\right],
\]
which implies
\[
\mathbb{E}\!\left[\left\|\theta^{(r+1)}-\theta_0\right\|^2\right]
\;\le\;
\frac{16}{\tau_1^2}
\left(
\frac{{\rm trace}(\Sigma)}{N}
+
\frac{2mB^2RKd}{\mu^2N^2}
\right)
+
\frac{C}{N^2}
\left(
m\,{\rm trace}(\Sigma)
+
\sum_{i=1}^m
\frac{B^2RKd}{\mu^2n_i}
\right)^2.
\]
This is \eqref{eq:ag3_bound_no_sigfield} for $r+1$.

By induction, \eqref{eq:ag3_bound_no_sigfield} holds for all $r\in\{1,\dots,R\}$.
When $r=R$ we obtain the desired bound for
$\hat{\theta}_{\mathrm{(AG3)}}=\theta^{(R)}$.
\end{proof}

\medskip

\begin{proof}[Proof of Lemma~\ref{lem:AG4muGDP}]
We design the first step of the method (\FedAvg) to be $\mu/\sqrt{2}-$GDP. From the proof of Lemma \ref{lem:AG2muGDP}, the noise level $\sigma_i = \frac{2B\eta_1\sqrt{2RK}}{\mu n_i/3}=\frac{6B\eta_1\sqrt{2RK}}{\mu n_i}$ guarantees the \FedAvg step to be $\mu/\sqrt{2}-$GDP. 

Then we only need to find the desired noise level for the second (Newton) step to be $\mu/\sqrt{2}-$GDP. To calculate the sensitivity level of the Newton procedure, we note that from triangle inequality, 
\begin{align*}
    \left\|H^{-1}g-H^{\prime-1}g^{\prime}\right\| &\le \left\|H^{-1}-H^{\prime-1}\right\| \cdot \left\|g\right\|+\left\|H^{\prime-1}\right\| \cdot \left\|g-g^{\prime}\right\| \\ & \le \frac{1}{n_i/3}\cdot \frac{2}{\tau_1}\cdot B + \frac{6B}{\tau_1n_i} \\ & \le \frac{12B}{\tau_1 n_i}
\end{align*}

Then \begin{align*} \frac{2B^{(New)}}{\mu n_i}&= \frac{ \frac{12B}{\tau_1 n_i}}{\frac{\mu}{\sqrt{2}}} \Rightarrow B^{(New)}=\frac{6\sqrt{2}B}{\tau_1} \end{align*}
Therefore the Newton step with the given noise level is $\mu/\sqrt{2}-$ GDP.
\end{proof}

\medskip

\begin{proof}[Proof of Theorem~\ref{th:newton-avg}]
    The proof follows along the lines of the bound for MSE of $\Bar{\theta}$, through a bias variance decomposition. To that end, we find an upper bound for the bias of the private Newton step estimators first. Note that:
\begin{align*}
    \theta_i^{(New)}
    =&~
    \Bar{\theta}
    -
    \left(
    \ddot{L}_i(\Bar{\theta},X_i^{(1)})
    \right)^{-1}
    \dot{L}_i(\Bar{\theta},X_i^{(2)})
    +
    \frac{2B^{(New)}}{\mu n_i}Z_i,
\end{align*}
so that
\begin{align*}
    \theta_i^{(New)}-\theta_0
    =&~
    \Bar{\theta}-\theta_0
    -
    \left(
    \ddot{L}_i(\Bar{\theta},X_i^{(1)})
    \right)^{-1}
    (\dot{L}_i(\theta_0,X_i^{(2)}))
    +
    \frac{2B^{(New)}}{\mu n_i}Z_i\\
    &~-
    \left(
    \ddot{L}_i(\Bar{\theta},X_i^{(1)})
    \right)^{-1}
    \left(
    \int_0^1
    \ddot{L}_i(t\Bar{\theta}+(1-t)\theta_0,X_i^{(2)})
    dt
    \right)(\Bar{\theta}-\theta_0)\\
    =&~
    \left(
    I_d-\left(
    \ddot{L}_i(\Bar{\theta},X_i^{(1)})
    \right)^{-1}
    \left(
    \int_0^1
    \ddot{L}_i(t\Bar{\theta}+(1-t)\theta_0,X_i^{(2)})
    dt
    \right)
    \right)(\Bar{\theta}-\theta_0)\\
    &~
    -
    \left(
    \ddot{L}_i(\Bar{\theta},X_i^{(1)})
    \right)^{-1}
    (\dot{L}_i(\theta_0,X_i^{(2)}))
    +
    \frac{2B^{(New)}}{\mu n_i}Z_i\\
    =:&~T_{n1}^{(i)}+T_{n2}^{(i)}+T_{n3}^{(i)}.
\end{align*}
We now bound the norms of each of the above terms one by one.
\begin{align*}\label{eq:T1-up-bd}
    \|T_{n1}^{(i)}\mathbf{1}_{\mathcal{E}}\|
    =&~
    \left\Vert
    \left(
    \ddot{L}_i(\Bar{\theta},X_i^{(1)})
    \right)^{-1}
    \left(
    \ddot{L}_i(\Bar{\theta},X_i^{(1)})
    -
    \left(
    \int_0^1
    \ddot{L}_i(t\Bar{\theta}+(1-t)\theta_0,X_i^{(2)})
    dt
    \right)
    \right)(\Bar{\theta}-\theta_0)
    \right\Vert\\
    \le&~
    \frac{1}{\tau_1-\|\ddot{L}_i(\Bar{\theta},X_i^{(1)})-\ddot{L}_0(\Bar{\theta},X_i^{(1)})\|}
    \times\\
    &~\times 
    \bigg(\left\Vert
    \ddot{L}_0(\bar{\theta})
    -
    \int_0^1
    \ddot{L}_0(t\Bar{\theta}+(1-t)\theta_0,X_i^{(2)})
    dt
    \right\Vert\\
    &~+
    2\sup_{t\in [0,1],k\in \{1,2\}}
    \|
    \ddot{L}_i(t\Bar{\theta}+(1-t)\theta_0,X_i^{(k)})
    -
    \ddot{L}_0(t\Bar{\theta}+(1-t)\theta_0)
    \|\bigg)\|\bar{\theta}-\theta_0\|\\
    \le &~
    \frac{2}{\tau_1}    \|\bar{\theta}-\theta_0\|^2
    +\frac{C}{\tau_1}\|\bar{\theta}-\theta_0\|\cdot\frac{H}{\sqrt{n_i}}\numberthis
\end{align*}
where we use the fact that 
\[
(
\sup_{t\in [0,1]}
\|
\ddot{L}_i(t\bar{\theta}+(1-t)\theta_0,X_i^{(2)})
-
\ddot{L}_0(t\bar{\theta}+(1-t)\theta_0,X_i^{(2)})
)
\le 
\frac{CH}{\sqrt{n_i}}.
\]
with high probability over the event $\mathcal{E}$, on which uniform convergence of $\ddot{L}_i(\theta,X_i^{(1)})$ to $\ddot{L}_0(\theta,X_i^{(1)})$ for $\theta=t\bar{\theta}+(1-t)\theta_0$, for $t\in [0,1]$. As a result on the same high probability event $\mathcal{E}$ we have the upper bound
\begin{align*}
    \EE\left((\theta_i^{(New)}-\theta_0)\right)
    \le&~ \EE(T_{n1}^{(i)}\mathbf{1}_{\mathcal{E}}) + 
    +\EE(T_{n1}^{(i)}\mathbf{1}_{\mathcal{E}^c})
    +\EE(T_{n2}^{(i)})+ \EE(T_{n3}^{(i)})\\
    =&~ \EE(T_{n1}^{(i)}\mathbf{1}_{\mathcal{E}})
    + 
    \EE(T_{n1}^{(i)}\mathbf{1}_{\mathcal{E}^c})
\end{align*}
Here the first equality follows due to sample splitting which ensures 
\[
\EE(T_{n2}^{(i)})=
\EE\left(
    \ddot{L}_i(\Bar{\theta},X_i^{(1)})
    \right)^{-1}
    \EE (\dot{L}_i(\theta_0,X_i^{(2)}))
    =\EE\left(
    \ddot{L}_i(\Bar{\theta},X_i^{(1)})
    \right)^{-1}\cdot\mathbf{0}
    =\mathbf{0}
\]
and since the privacy noise are mean zero Gaussians, it is immediate to see that $\EE(T_{n3}^{(i)})=\mathbf{0}$. Notice that due to sample splitting we have
\begin{align*}
  \left\Vert\EE(T_{n1}^{(i)}\mathbf{1}_{\mathcal{E}})  \right\Vert 
  =&~
  \left\Vert
  \EE\left(
  \left(
    I_d-\left(
    \ddot{L}_i(\Bar{\theta},X_i^{(1)})
    \right)^{-1}
    \left(
    \int_0^1
    \ddot{L}_i(t\Bar{\theta}+(1-t)\theta_0,X_i^{(2)})
    dt
    \right)
    \right)(\Bar{\theta}-\theta_0)
  \right)\right\Vert\\
  \le:&~
  \frac{C}{\tau_1}\EE\|\Bar{\theta}-\theta_0\|^2
  +\left\Vert\EE
   \left(
    I_d-\left(
    \ddot{L}_i(\theta_0,X_i^{(1)})
    \right)^{-1}
    \ddot{L}_i(\theta_0,X_i^{(2)})
    \right)(\Bar{\theta}-\theta_0)\right\Vert\\
   \le &~ 
   \frac{C}{\tau_1}\EE\|\Bar{\theta}-\theta_0\|^2
  +\frac{CH}{\sqrt{n_i}}\EE \|(\Bar{\theta}-\theta_0)\|
  +Cd\exp(-cd/2)
\end{align*}
where the first inequality follows due to the Lipschitz property of $\nabla^2\rho$, and the second inequality using concentration bounds of the sample Hessian $\ddot{L}_i(\theta_0,X_i^{(1)})$ to its expectation $\ddot{L}_0(\theta_0)$ on the event $\mathcal{E}$.

We then have
\begin{align*}
    &~\left\Vert\EE\left(
    \sum_{i=1}^m\frac{n_i}{N}\theta_i^{(New)}-\theta_0
    \right)\right\Vert \\
    \le&~
    \frac{C}{\tau_1}\EE\|\bar{\theta}-\theta_0\|^2
    +
    C\left(\sum\frac{n_i}{N}\cdot\frac{H}{\sqrt{n_i}}\right)(\|\EE(\bar{\theta}-\theta_0)\|)
    +
    \left\Vert\EE
    \left(\sum_{i=1}^m\frac{n_i}{N}T_{n1}^{(i)}\mathbf{1}_{\mathcal{E}^c}
    \right)
    \right\Vert\\
    \le&~
    \frac{C}{\tau_1}\EE\|\bar{\theta}-\theta_0\|^2
    +
    C\sqrt{\frac{m}{N}}(\|\EE(\bar{\theta}-\theta_0)\|)
    +
    Cd \sqrt{\P(\mathcal{E}^c)}   \\
    \le&~
    \frac{C}{\tau_1}\EE\|\bar{\theta}-\theta_0\|^2
    +
    C\sqrt{\frac{m}{N}}(\|\EE(\bar{\theta}-\theta_0)\|)
    +
    Cd \sqrt{\P(\mathcal{E}^c)}   \\
    \le&~
    \frac{C}{\tau_1}\EE\|\bar{\theta}-\theta_0\|^2
    +
    C\sqrt{\frac{m}{N}}(\|\EE(\bar{\theta}-\theta_0)\|)
    +
    Cd \exp(-cd)
\end{align*}
where the last term in the second last line follows by using the definition of $T_{n1}^{(i)}$ to get the bound $\EE\|T_{n1}^{(i)}\|\le Cd$ for all $i$. The last line follows since $\P(\mathcal{E}^c)\le C\exp(-cd)$ for some constant $c,C>0$.

Similarly for the variance we have
\begin{align*}
 &~{\rm trace}(\Var(\bar{\theta}^{(2)}))  
 =
 {\rm trace}\left(\Var
 \left(\sum_{i=1}^m\frac{n_i}{N}{\theta}_i^{(New)}\right)\right)\\
 \le&~
 2{\rm trace}\left(\Var\left(\sum_{i=1}^m\frac{n_i}{N}T_{n1}^{(i)}\right)\right)
 +
 2{\rm trace}\left(\Var\left(\sum_{i=1}^m\frac{n_i}{N}T_{n2}^{(i)}\right)\right)
 +
 {\rm trace}\left(\Var\left(\sum_{i=1}^m\frac{n_i}{N}T_{n3}^{(i)}\right)\right).
\end{align*}
where we use the fact that $T_{n3}$ is composed of privacy noise and is hence independent of $T_{n1}$ and $T_{n2}$. Note that by \eqref{eq:T1-up-bd} and using the bound on $\P(\mathcal{E}^c)$ we have
\begin{align*}
  {\rm trace}\left(\Var\left(\sum_{i=1}^m\frac{n_i}{N}T_{n1}^{(i)}\right)\right)
  \le&~ \frac{4}{\tau_1^2}\EE\|\bar{\theta}-\theta_0\|^4+
  \frac{C H^2}{\tau_1^2}\cdot \EE\|\bar{\theta}-\theta_0\|^2\cdot \left(
  \sum_{i=1}^m\frac{\sqrt{n_i}}{N}
  \right)^2
  +Cd\exp(-cd)
  \\
  \le&~ \frac{4}{\tau_1^2}\EE\|\bar{\theta}-\theta_0\|^4+
  \frac{Cm H^2}{N\tau_1^2}\cdot \EE\|\bar{\theta}-\theta_0\|^2.
\end{align*}

Next, since $\max_i\|\ddot{L}_i(\bar{\theta},X_i^{(1)})\|\le C\tau_1^{-1}$ on $\mathcal{E}$, by a similar decomposition into $\mathcal{E}$ and $\mathcal{E}^c$ we have
\begin{align*}
  &~{\rm trace}\left(\Var\left(\sum_{i=1}^m\frac{n_i}{N}T_{n2}^{(i)}\right)\right)\\
  \le&~ {\rm trace}\left(
  \EE\left(\Var\left(\sum_{i=1}^m
  \frac{n_i}{N}T_{n2}^{(i)}|\bar{\theta}\right)
  \right)\right)
  +{\rm trace}\left(\frac{C}{\tau_1^2}
  \Var\left(\EE\left(\sum_{i=1}^m
  \frac{n_i}{N}T_{n2}^{(i)}|\bar{\theta}\right)
  \right)\right)
  +Cd(\exp(-cd)\wedge N^{-4})
  \\
  \le&~
  \frac{C}{\tau_1^2}
  {\rm trace}\left(
  \left(\left(\sum_{i=1}^m
  \frac{n_i^2}{N^2}
  \Var(\dot{L}_i(\theta_0,X_i^{(2)}))\right)
  \right)\right)
  +0
    +Cd(\exp(-cd)\wedge N^{-4})
  \\
  =&~ 
  \frac{C}{\tau_1^2}
  {\rm trace}
  \left(\sum_{i=1}^m
  \frac{n_i^2}{N^2}
  \cdot\frac{{\rm trace}(\Sigma)}{n_i}
  \right)\\
  \le&~
  \frac{C{\rm trace}(\Sigma)}{N\tau_1^2}
  +Cd(\exp(-cd)\wedge N^{-4}).
\end{align*}

This implies
\begin{align*}
    \EE(\|\bar{\theta}^{(2)}-\theta_0\|^2|)
    \le&~
    \frac{C}{\tau_1^2m}\EE\|\bar{\theta}-\theta_0\|^4
    +
    \frac{C{\rm trace}(\Sigma)}{\tau^2 N}
    +
    \frac{12m(B^{(New)})^2d}{\mu^2N^2}\\
    &~
    +\frac{C}{\tau_1^2}
    \left(
    (\EE\|\bar{\theta}-\theta_0\|^2)^2
    +
    (\EE\|\bar{\theta}-\theta_0\|^2 )
    \cdot
    \frac{ m H^2}{N}
    \right)
\end{align*}
One can follow the proof of Theorem~\ref{thm:one-step-fedavg} to show that $\EE\|\bar{\theta}-\theta_0\|^4\le C (\EE\|\bar{\theta}-\theta_0\|^2)^2$ so that, now using Theorem~\ref{thm:one-step-fedavg} we have
\begin{align*}
   \EE(\|\bar{\theta}^{(2)}-\theta_0\|^2) 
   \le&~
   \frac{C{\rm trace}(\Sigma)}{\tau_1^2 N}
    +
    \frac{12m(B^{(New)})^2d}{\mu^2N^2}\\
    &~+\frac{CmH^2}{\tau_1^2N^3}
    \left(
    m^2({\rm trace}(\Sigma))^2
    +
    m\sum_{i=1}^m\frac{B^4K_{1i}^2d^2}{\mu^4n_i^2}
    \right)\\
    &~\frac{C}{\tau_1^2N^4}
    \left(
    m^4({\rm trace}(\Sigma))^4
    +
    m^3\sum_{i=1}^m\frac{B^8K_{1i}^4d^4}{\mu^8n_i^4}
    \right)\\
    \le&~\frac{C{\rm trace}(\Sigma)}{\tau_1^2 N}
    +
    \frac{12m(B^{(New)})^2d}{\mu^2N^2}
\end{align*}
where we use the assumption that ${\rm trace}(\Sigma)\le d$, $H^2\le C d$, and $m^3d^2\le N^2$ to ensure that the first term dominates over the third and fifth terms. Similarly, the fourth term can be bounded as
\begin{align*}\label{eq:m-by-N-factor-reduction}
   \frac{Cm^2H^2}{\tau_1^2N^3}\sum_{i=1}^m\frac{B^4K_{1i}^2d^2}{\mu^4n_i^2}
   \le&~
   \frac{Cm^3d^2}{\tau_1^2N^3}\cdot\frac{B^4K_{max}^2}{\mu^4n_{\min}^2}\\
   =&~
   \frac{Cm^2K_{max}}{\tau_1^2N}
   \cdot \frac{m(B^{(New)})^2d}{\mu^2N^2}
   \cdot \frac{B^2K_{max}d}{\mu^2n_{\min}^2}\\
   \le&~ \frac{Cm(B^{(New)})^2d}{\mu^2N^2}\numberthis
\end{align*}
where we use the assumptions $m\le \mu\sqrt{Nn_{\min}^2}/(CK_{\max}B\sqrt{d})$. The sixth term can again be bounded as:
\begin{align*}
   \frac{Cm^3}{\tau_1^2N^4}\sum_{i=1}^m\frac{B^8K_{1i}^4d^4}{\mu^8n_i^4}
   \le&~
   \frac{Cm^4d^4}{\tau_1^2N^4}\cdot\frac{B^8K_{max}^4}{\mu^8n_{\min}^4}\\
   =&~
   \frac{Cm^2d^2K_{max}^2}{\tau_1^2N^2}
   \cdot \frac{m^2(B^{(New)})^4d^2}{\mu^4N^2}
   \cdot \frac{B^4K_{max}^2}{\mu^4n_{\min}^4}\\
   \le&~ \frac{Cm(B^{(New)})^2d}{\mu^2N^2}
\end{align*}
provided $m\le \mu^2n_{\min}^{4/3}N^{2/3}/(CdB^2K_{\max}^{4/3})$ which follows if $B=C(\sqrt{d}\vee \sqrt{\log N})$ and
\[
m\le \frac{\mu^{4/3}n_{\min}^{4/3}}{d^{1/3}(d\vee (\log N))K_{\max}^{4/3}}\cdot \frac{\mu^{2/3}N^{2/3}}{Cd^{2/3}}
=\left(\frac{n_{\min}^2\mu^2}{\sqrt{d}(d\vee \log N)^{3/2}K_{\max}^{2}}\right)^{2/3}
\left(\frac{N^2\mu^2}{md^2}\right)^{1/3}\cdot m^{1/3}/C
\]
i.e.,
\[
m\le \frac{n_{\min}^2\mu^2}{C\sqrt{d}(d\vee \log N)^{3/2}K_{\max}^{2}}
\cdot \sqrt{\frac{N^2\mu^2}{md^2}}
\]
which is implied by $m\le \mu^2(Nn_{\min}^2)^{2/3}/(C(d\vee \log N)^2K_{\max}^{4/3})$. The final rate thus becomes
\[
\EE(\|\bar{\theta}^{(2)}-\theta_0\|^2) 
   \le~
   \frac{C{\rm trace}(\Sigma)}{\tau_1^2 N}
    +
    \frac{12md^2}{\mu^2N^2}
\]    
when $B=C(\sqrt{d}\vee \sqrt{\log N})$. The sharper bound that coincides with Theorem~\ref{th:alg1} follows by first noting that
\[\frac{1}{\tau_1^2}
    \left(
 \sum_{i=1}^m
 \left(
    \frac{{\rm trace}(\Sigma)}{n_i}
    +\frac{2C_B^2 d^2\log(dN)}{\mu^2 n_i^2\log(1-\tau_1/3\tau_2)}\right)^{-1}
 \right)^{-1}
  \le \frac{Cd^2}{\sum_{i=1}^m(dn_i\wedge (n_i^2\mu^2/K_{1i}))}
\]
for a constant $C>0$ and by choosing the Newton step aggregation weights to be $\hat{w}_i\propto \left(
\frac{d}{n_i}+\frac{C_B^2d^2K}{\mu^2 n_i^2}\right)^{-1}$ with $\displaystyle\sum_{i=1}^m\hat{w}_i=1$. One can then follow analogous algebraic steps after \eqref{eq:m-by-N-factor-reduction} with $m\le (\mu^2n_{\min}^2N/(d\vee \log N)^3K_{\max}^2)^{1/3}$. We omit the details for brevity.
\end{proof}

\medskip

\begin{proof}[Proof of Theorem~\ref{th:low-bd}] From Corollary 1 in \cite{dong2022gaussian} we have that an estimator $\hat{\theta}$ is $\mu$-GDP if and only if it is $(\eps,\delta(\eps))$-DP where
    \[
\delta(\varepsilon)
=
\Phi\!\left(-\frac{\varepsilon}{\mu}+\frac{\mu}{2}\right)
-
e^{\varepsilon}\Phi\!\left(-\frac{\varepsilon}{\mu}-\frac{\mu}{2}\right)
\]

Let \(\varepsilon = c\mu\) where $\mu\to 0$ and $c=\sqrt{2\log\mu^{-1}}\to \infty$. Then
\begin{align*}
\delta(c\mu)
=&~
\Phi\!\left(-c+\frac{\mu}{2}\right)
-
e^{c\mu}\Phi\!\left(-c-\frac{\mu}{2}\right)\\
=&~
\mu\big(\phi(c)-c\Phi(-c)\big)+O(\mu^2)\\
=&~
\mu\phi(c)\left(\frac{1}{c^2}+O(c^{-4})\right)+O(\mu^2)\\
=&~
\mu\,\frac{\phi(c)}{c^2}\big(1+o(1)\big)\\
=&~
\frac{\mu}{\sqrt{2\pi}}\frac{e^{-c^2/2}}{c^2}\big(1+o(1)\big)\\
=&~\frac{\mu}{\sqrt{2\pi}}\frac{\mu}{c^2}\big(1+o(1)\big)
=
\frac{\mu^2}{\sqrt{2\pi}\,c^2}\big(1+o(1)\big)\\
=&~
\frac{\mu^2}{2\sqrt{2\pi}\,\log(1/\mu)}\big(1+o(1)\big).
\end{align*}
We will now use the technique for deriving lower bounds under distributed $(\eps,\delta)$-privacy as done in \cite{cai2024optimal,auddy2024minimax,xue2024optimal}.

Let $\mathcal{P}$ be the set of all distributions satisfying assumptions 1 to 4. In particular, $\mathcal{P}$ contains $P_{\theta}$, the set of joint distributions on $(\bx,y)$ with $\bx\sim N(0,\II_d)$ and the following generalized linear model on $y|\bx$. 
\[
p_{\theta}(y|\bx):=h(y)
\exp\left(\frac{\bx^{\top}\theta y-\psi(\bx^{\top}\theta)}{c}\right);\, \|\bx\|_{\rm \infty}\le B
\]
satisfying $\max\{\psi'(0),\sup_{u\in \RR}\psi''(u)\}\le L$.

For the lower bound, we will use the multivariate Van Trees inequality (see, e.g., Theorem 1 of \cite{gill1995applications}, Lemma 4.3 of \cite{cai2024optimal}):
\[
\sup_{\theta}
\EE\|\hat{\theta}-\theta\|^2
\ge 
\int \EE\|\hat{\theta}-\theta\|^2 \zeta(\theta) d\theta 
\ge 
\frac{Cd^2}{\,{\rm trace}(\mathbb E_\theta[I_F(T;\theta)])\ +\ I_F(\zeta)\,}, 
\]
where $\zeta$ is a prior on $\theta$ supported on $\RR^d$, and
\begin{equation}\label{eq:van-trees}
\mathbb E_\theta[I_F(T;\theta)]=\int I_F(T;\theta)\,\zeta(\theta)\,d\theta,
\quad
I_F(\zeta)=\int \frac{1}{\xi(\theta)}\|\xi'(\theta)\|^2d\theta.
\end{equation}
Here $T$ is the final privatized transcript obtained after $R$ rounds of communication between the clients and the server.

The above result holds under some regularity conditions, for every estimator $\hat\theta=\hat\theta(T)$ with $\mathbb E[\|\hat\theta-\theta\|^2]<\infty$ under the joint law of $(T,\theta)$. Following the proof of Proposition 10 of \cite{xue2024optimal}, let us define
\begin{align*}
    M^{(r)}:=&\text{all information shared from server to clients in $r$ rounds}\\
    T^{(r)}_s:=&T^{(r)}_s(\{(X_j^{(r)},y_j^{(r)}):1\le j\le n_s^{(r)}\},M^{(r-1)}),
    \text{private transcript from client $s$ in round $r$.}
\end{align*}
Here $1\le s\le m$ and $1\le r\le R$. Note that by sample splitting, the random variables $\{(X_j^{(r)},y_j^{(r)}):1\le j\le n_s^{(r)}\}$ are independent for $1\le s\le m$ and $1\le r\le R$. Then by using conditional expectations, similar to equation 56 in \cite{xue2024optimal} we have
\begin{equation}\label{eq:fisher-decomp}
I_F(T;\theta)=\sum_{s=1}^m\sum_{r=1}^RI_F({T_s^{(r)}|M^{(r-1)}},(\theta)).
\end{equation}
By the definition of Fisher information, we have
\begin{align*}
    I_F({T_s^{(r)}|M^{(r-1)}},(\theta))=\EE \EE[S(\mathcal{D}_s^{(r)})|T_s^{(r)},M^{(r-1)}] \EE[S(\mathcal{D}_s^{(r)})|T_s^{(r)},M^{(r-1)}]^{\top}
\end{align*}
where $S(\mathcal{D}_s^{(r)})$ denotes the score function w.r.t. $\theta$ on $\mathcal{D}_s^{(r)}:=\{(X_j^{(r)},y_j^{(r)}):1\le j\le n_s^{(r)}\}$. The crux of the paper is supplied by Lemma 4.2 of \cite{cai2024optimal} and its use in the proof of Proposition 12 of \cite{xue2024optimal}. It is thus enough to check their conditions. By our assumptions on bounded covariates and link function $\psi$, we can ensure that the score function:
\[
\nabla_{\theta}\log p_{\theta}(y,\bx)=(y-\psi'(\bx^{\top}\theta))\bx
\]
is sub-Gaussian with sub-Gaussianity parameter $Cd$ for some constant $C>0$. We can then follow the proof of Proposition 12 in \cite{xue2024optimal} to write:
\begin{align*}
    {\rm trace}(I_F({T_s^{(r)}|M^{(r-1)}},(\theta)))
    =&~
    {\rm trace}(\EE \EE[S(\mathcal{D}_s^{(r)})|T_s^{(r)},M^{(r-1)}] \EE[S(\mathcal{D}_s^{(r)})|T_s^{(r)},M^{(r-1)}]^{\top})\\
    \le&~ C(n_s^{(r)})^2\mu^2\log(\mu^{-1}).
\end{align*}
We also have the non-private upper bound
\begin{align*}
    {\rm trace}(I_F({T_s^{(r)}|M^{(r-1)}},(\theta)))
    \le Cdn_s^{(r)},
\end{align*}
leading to the combined upper bound:
\begin{align*}
    {\rm trace}(I_F({T_s^{(r)}|M^{(r-1)}},(\theta)))
    \le 
    C((n_s^{(r)})^2\mu^2\log(\mu^{-1}))
    \wedge 
    Cdn_s^{(r)}).
\end{align*}
which when added across $s=1,..., m$ and $r=1,...,R$, together with \eqref{eq:fisher-decomp} yields
\begin{equation}\label{eq:fisher-bound}
I_F(T;\theta)
\le \sum_{s=1}^m\sum_{r=1}^R
    C((n_s^{(r)})^2\mu^2\log(\mu^{-1}))
    \wedge 
    Cdn_s^{(r)}).
\end{equation}
Following \cite{xue2024optimal} we choose the prior $\zeta$ to be $N(\mathbf{0},I_d)$ truncated to $[-1,1]^d$. It can be checked that $I_F(\zeta)\le Cd$ so that \eqref{eq:van-trees} and \eqref{eq:fisher-bound} together imply:
\[
\sup_{\theta}
\EE\|\hat{\theta}-\theta\|^2
\ge \frac{Cd^2}{\sum_{s=1}^m\sum_{r=1}^R
    C((n_s^{(r)})^2\mu^2\log(\mu^{-1}))
    \wedge 
    Cdn_s^{(r)})+Cd}
\]
for a constant $C>0$. Note that $\sum_{r=1}^R(n_s^{(r)})^2\le (\sum_{r=1}^Rn_s^{(r)})^2=n_s^2$, thus implying:
\begin{align*}
\sup_{\theta}
\EE\|\hat{\theta}-\theta\|^2
\ge&~ \frac{d^2}{\sum_{s=1}^m
    C(n_s^2\mu^2\log(\mu^{-1}))
    \wedge 
    Cdn_s)+Cd}\\
\ge&~ \frac{d^2}{\sum_{s=1}^m
    C(n_s^2\mu^2\log(\mu^{-1}))
    \wedge 
    Cdn_s)}    .
\end{align*}
\end{proof}

\section{Additional Simulation Results and Details}
\label{appsim}

\subsection{Logistic regression additional figures}

\begin{figure}[h]
    \centering

    \begin{subfigure}{0.47\textwidth}
        \centering
        \includegraphics[width=\textwidth]{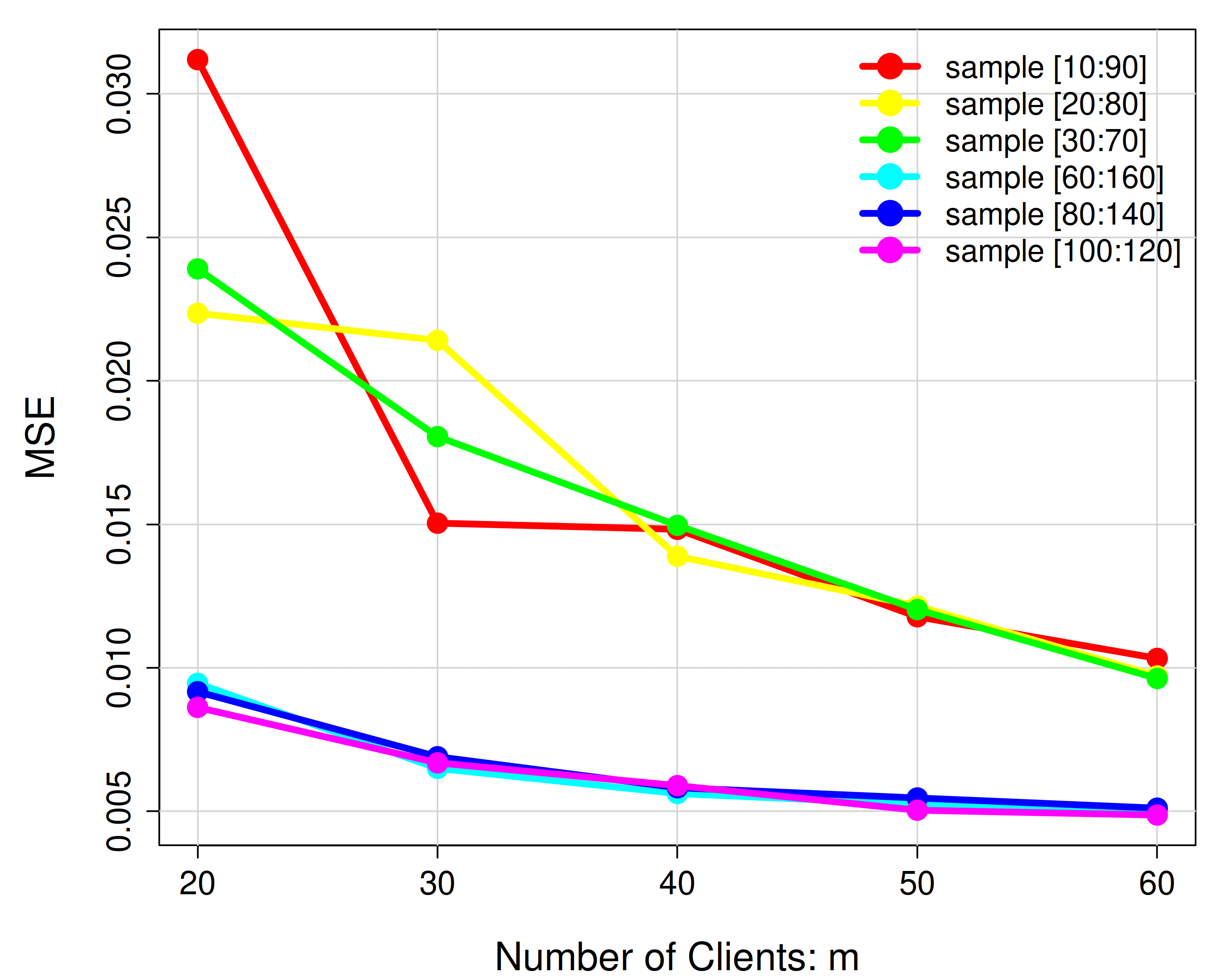}
        \caption{\FedSGD, different $n_i$}
    \end{subfigure}
    \hfill
    \begin{subfigure}{0.47\textwidth}
        \centering
        \includegraphics[width=\textwidth]{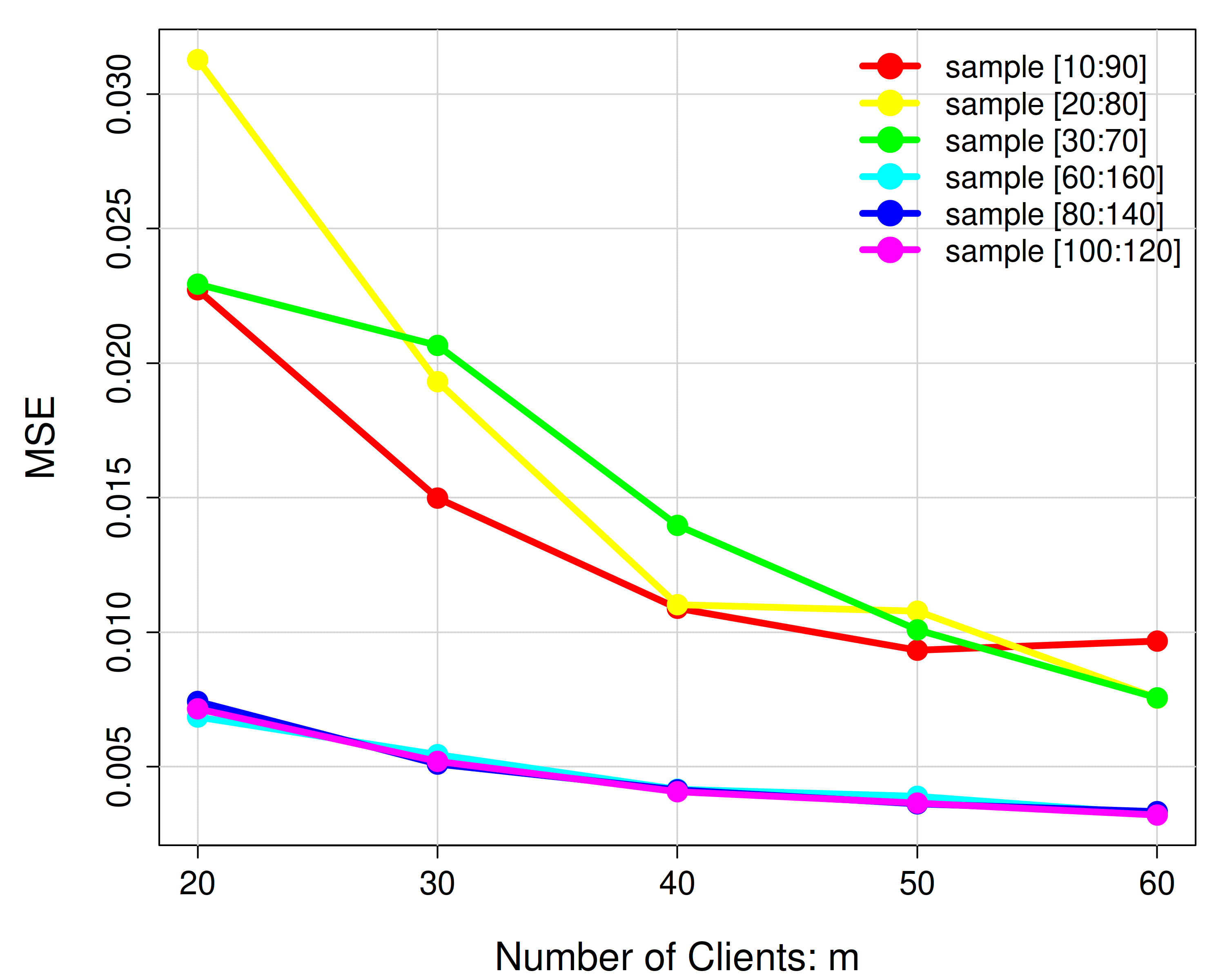}
        \caption{\FedHybrid, different $n_i$}
    \end{subfigure}

    \vspace{0.6cm}

    \caption{Empirical MSE results of DP version of \FedSGD\ and \FedHybrid\ under different local sample sizes in Logistic Regression.}

\label{fig:log_ag12_different_ni}
\end{figure}

This section contains additional figures from the simulation study using logistic regression. These figures depict empirical MSE of the methods with an increasing number of clients for the case when the client sample sizes $n_i$ differ across the clients. 

\begin{figure}[H]
    \centering
    \begin{subfigure}[t]{0.47\textwidth}
        \centering
        \includegraphics[width=\textwidth]{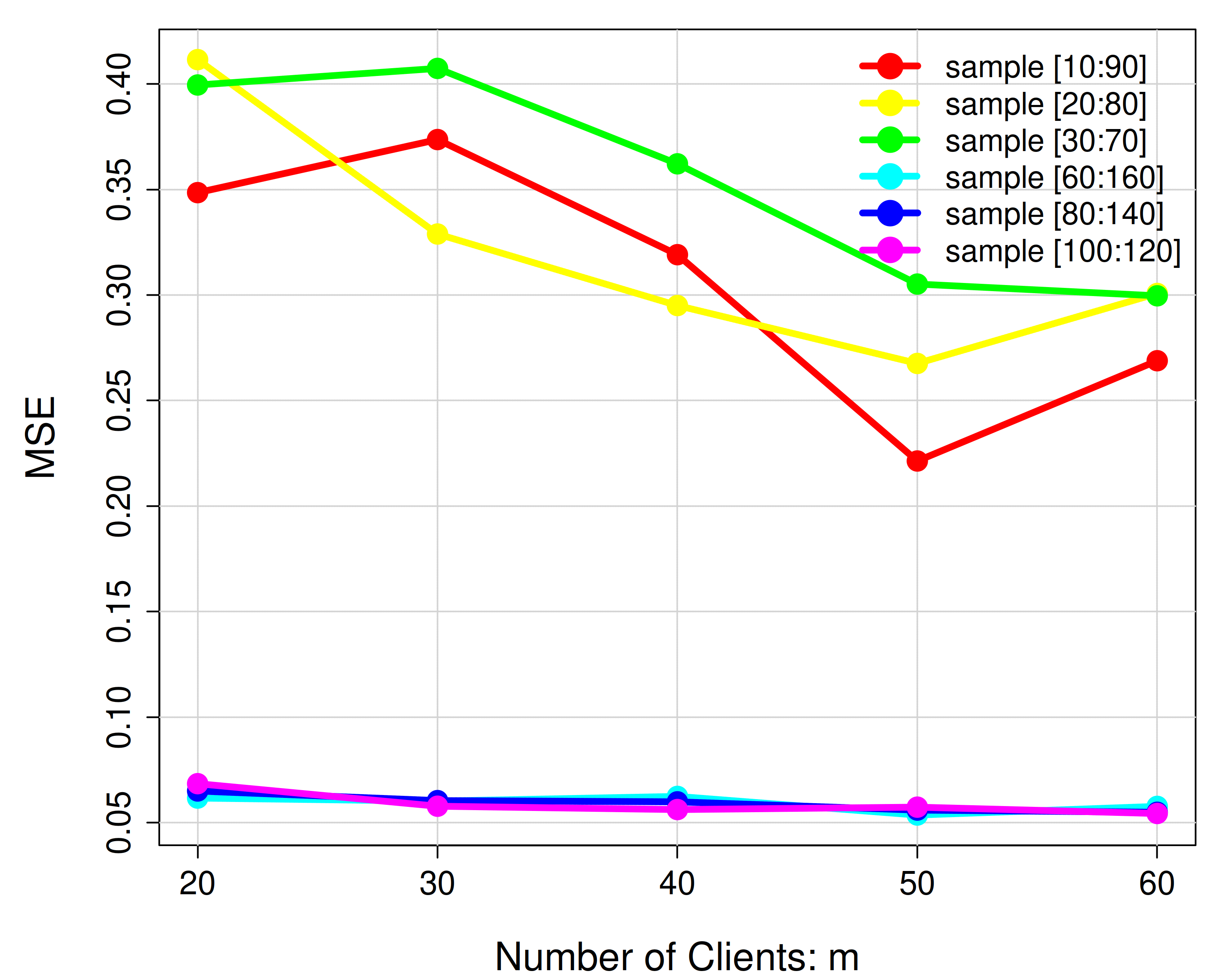}
        \caption{\FedAvg, different $n_i$}
    \end{subfigure}
\hfill
    \begin{subfigure}[t]{0.47\textwidth}
        \centering
        \includegraphics[width=\textwidth]{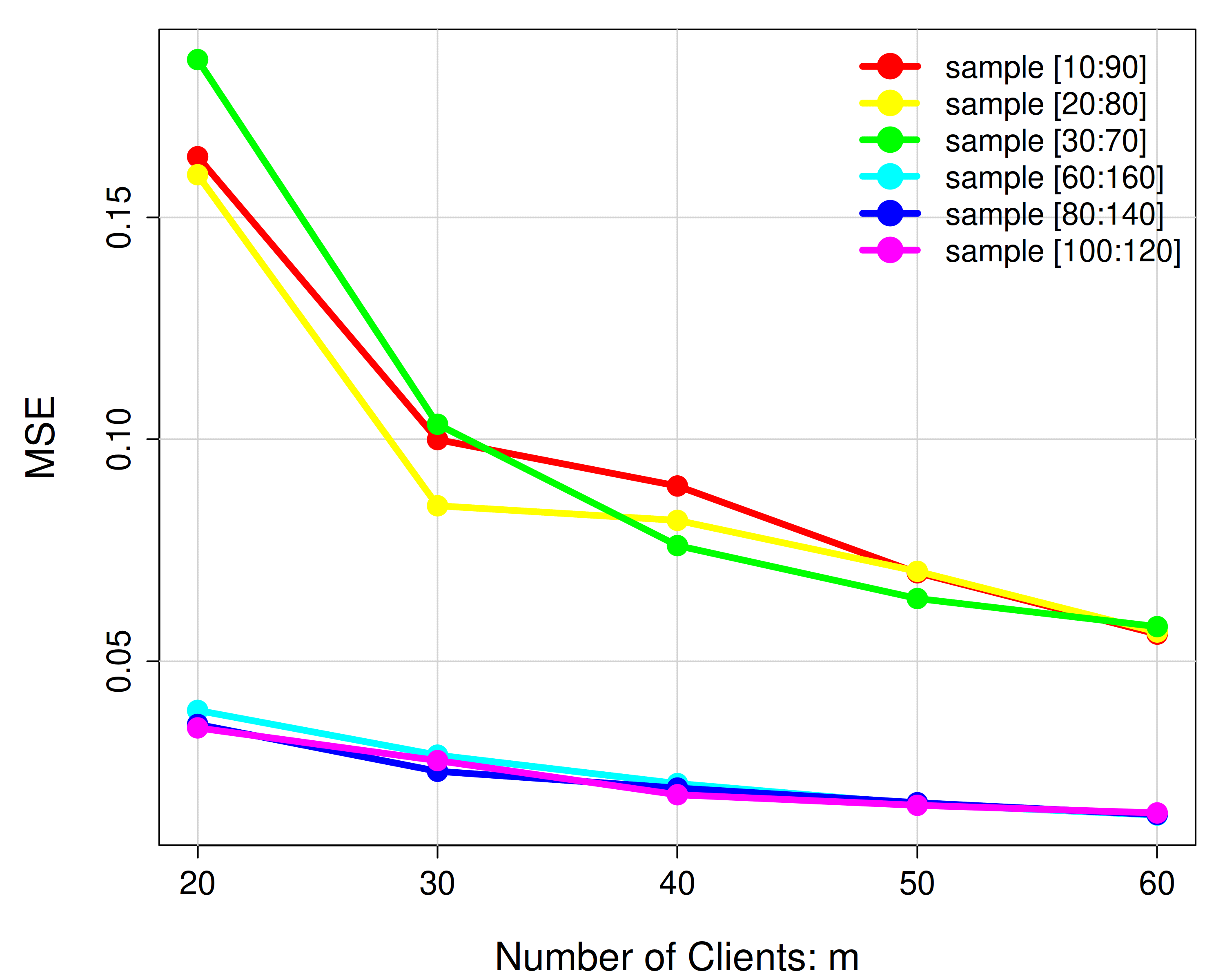}
        \caption{\FedNewton, different $n_i$}
    \end{subfigure}

    \caption{Empirical MSE results of \FedAvg\ and \FedNewton\ under different local sample sizes in Logistic Regression.}
    \label{fig:log_ag34_different_ni}
\end{figure}

\subsection{Poisson GLM simulation figures}

\begin{figure}[h]
    \centering
    \begin{subfigure}[t]{0.45\textwidth}
        \centering
        \includegraphics[width=\textwidth]{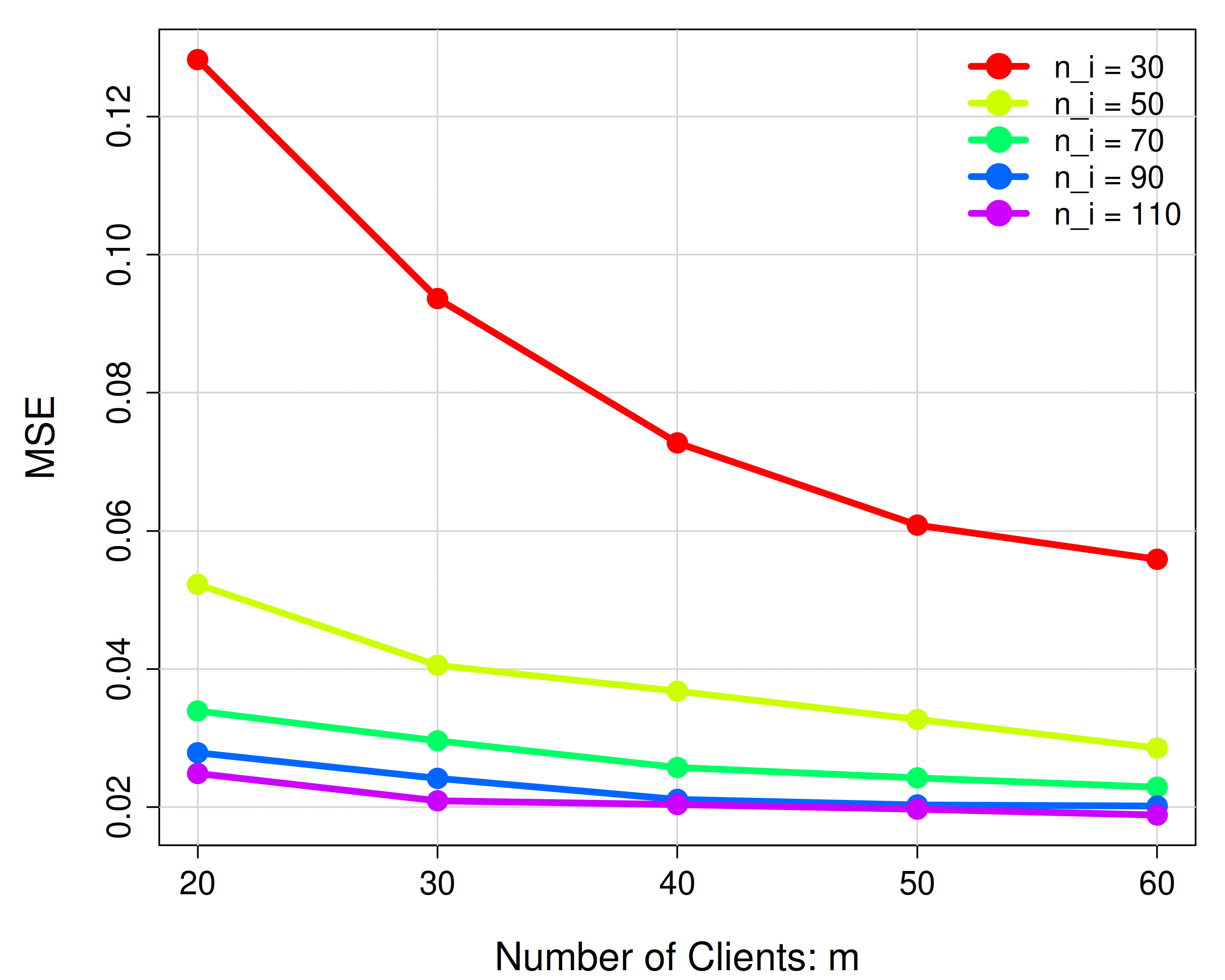}
        \caption{\FedSGD, same $n_i$}
    \end{subfigure}
    \hfill
    \begin{subfigure}[t]{0.45\textwidth}
        \centering
        \includegraphics[width=\textwidth]{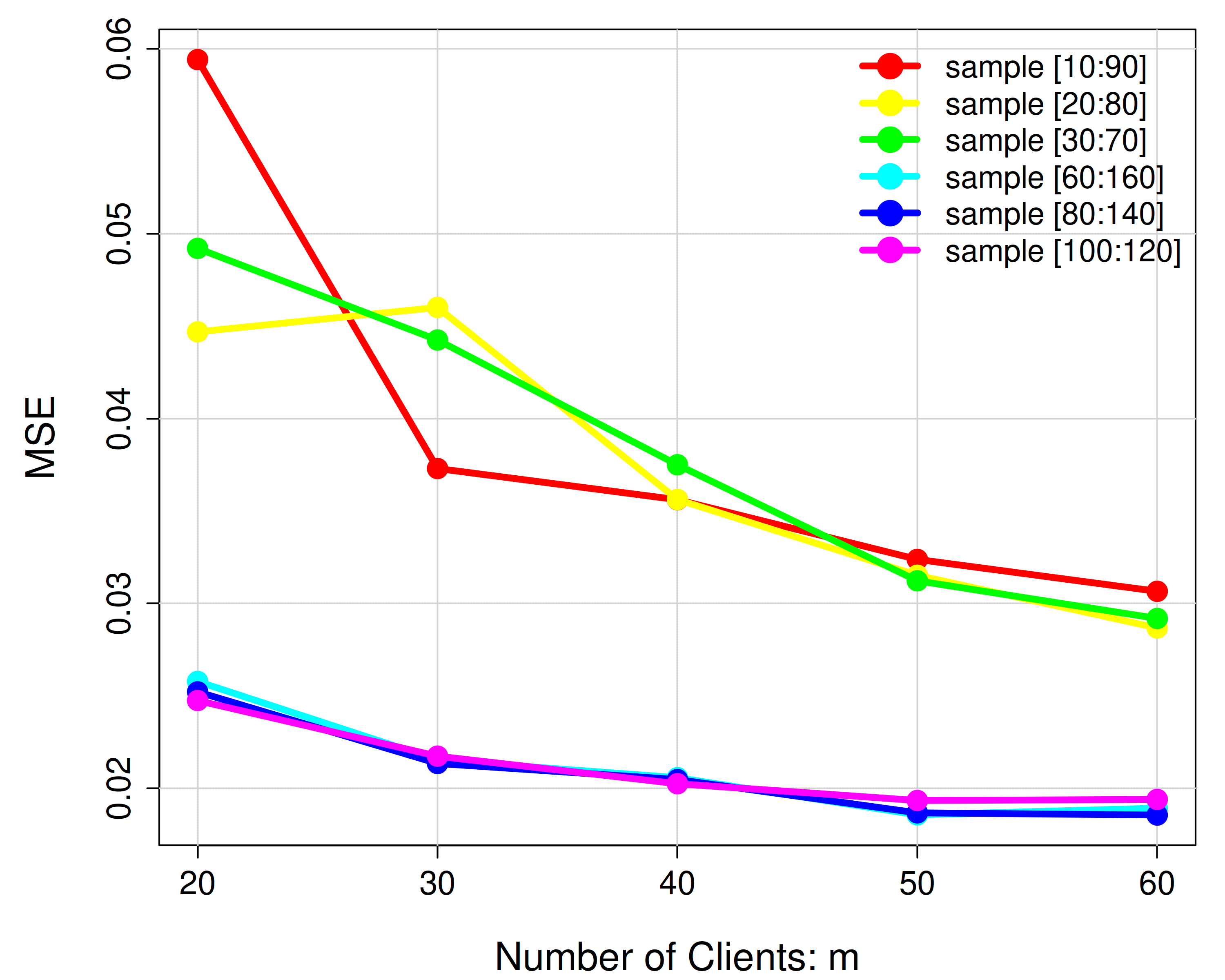}
        \caption{\FedSGD, different $n_i$}
    \end{subfigure}

    \vspace{0.6cm}

    \begin{subfigure}[t]{0.45\textwidth}
        \centering
        \includegraphics[width=\textwidth]{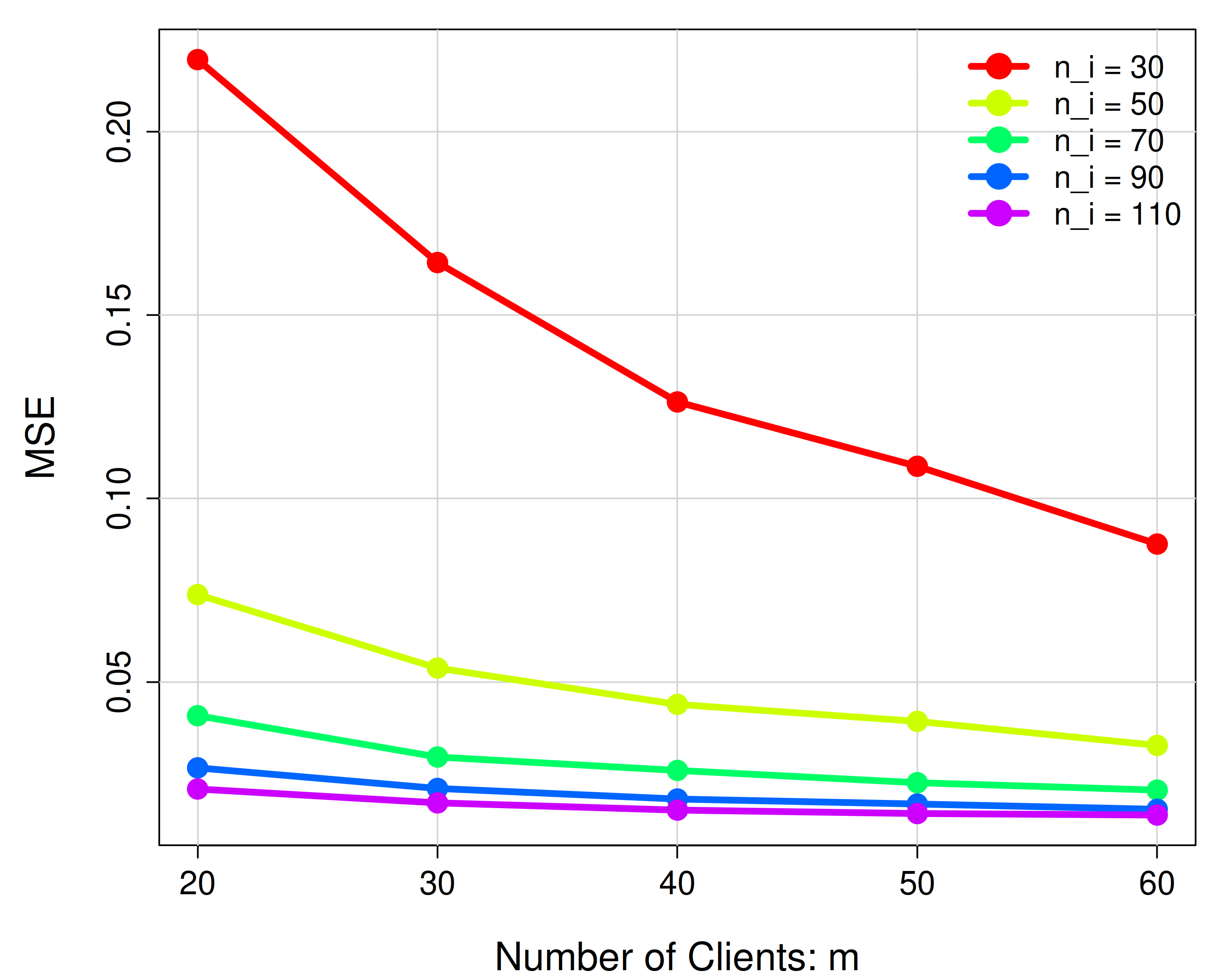}
        \caption{\FedHybrid, same $n_i$}
    \end{subfigure}
    \hfill
    \begin{subfigure}[t]{0.45\textwidth}
        \centering
        \includegraphics[width=\textwidth]{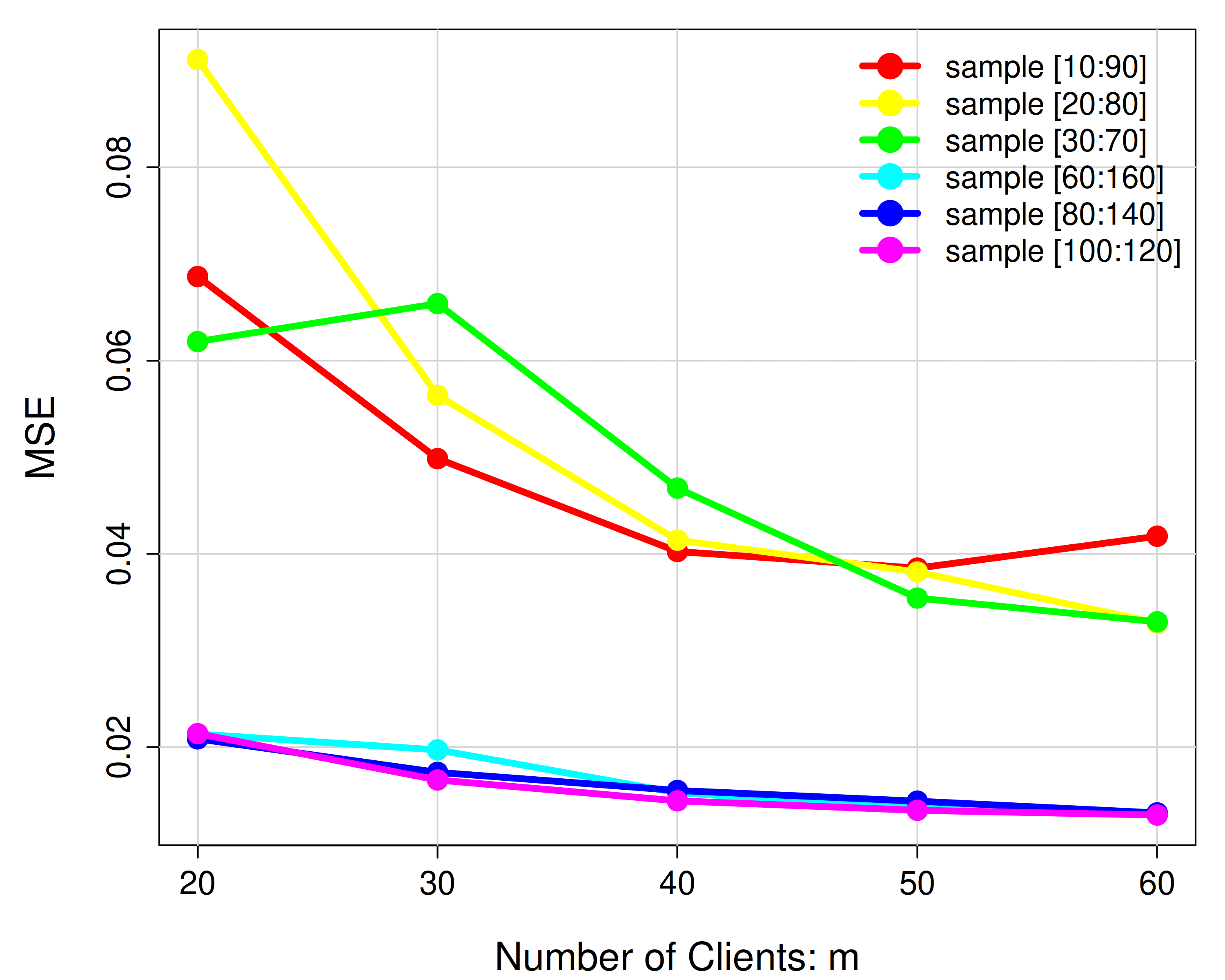}
        \caption{\FedHybrid, different $n_i$}
    \end{subfigure}

    \vspace{0.6cm}

    \caption{Empirical MSE results of \FedSGD\ and \FedHybrid\ under Equal and different Local Sample Sizes in Poisson GLM.}

\label{fig:poi_ag12_equal_different_ni}
\end{figure}

\begin{figure}
    \centering

    \begin{subfigure}[t]{0.45\textwidth}
        \centering
        \includegraphics[width=\textwidth]{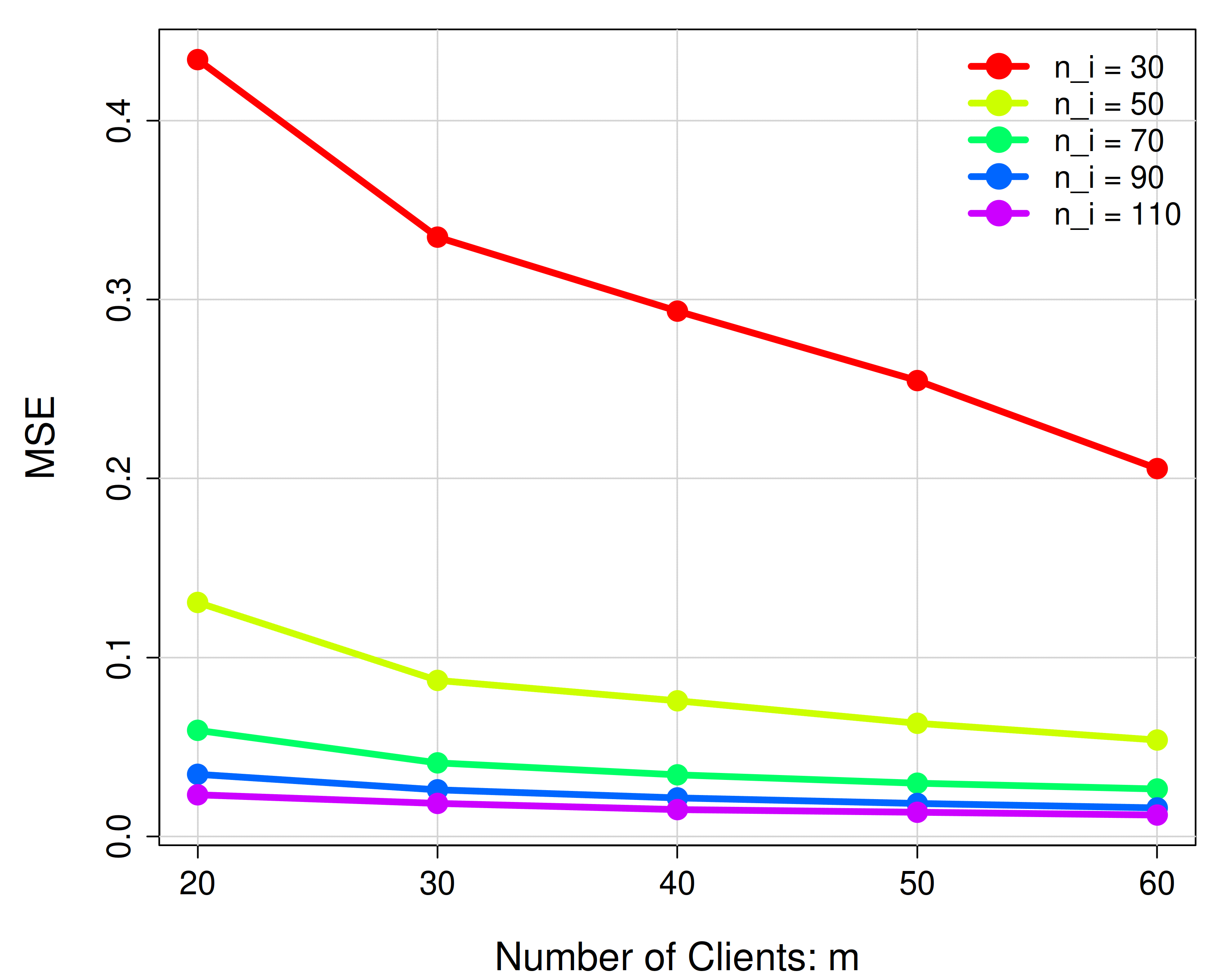}
        \caption{\FedAvg, same $n_i$}
    \end{subfigure}
    \hfill
    \begin{subfigure}[t]{0.45\textwidth}
        \centering
        \includegraphics[width=\textwidth]{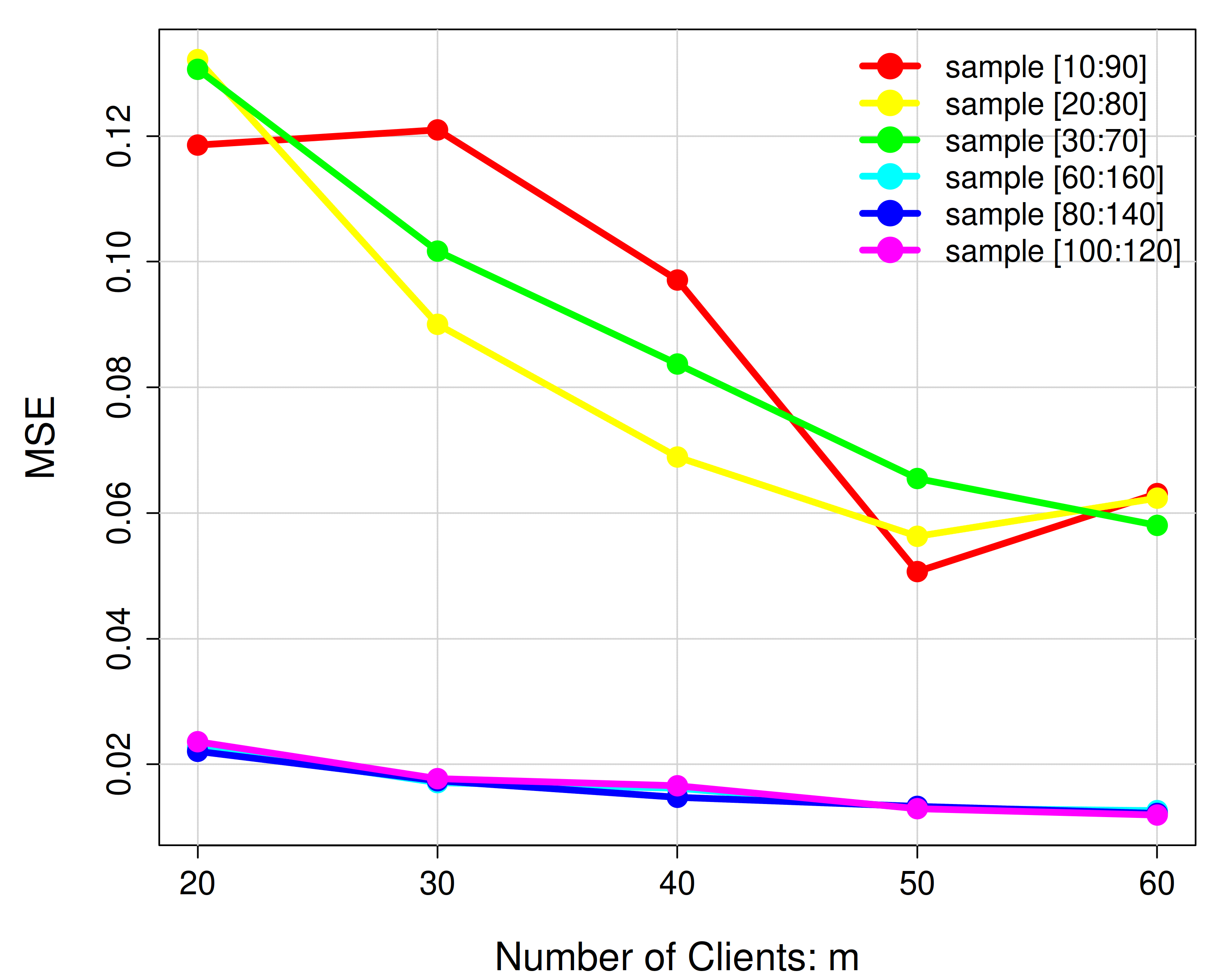}
        \caption{\FedAvg, different $n_i$}
    \end{subfigure}

    \begin{subfigure}[t]{0.45\textwidth}
        \centering
        \includegraphics[width=\textwidth]{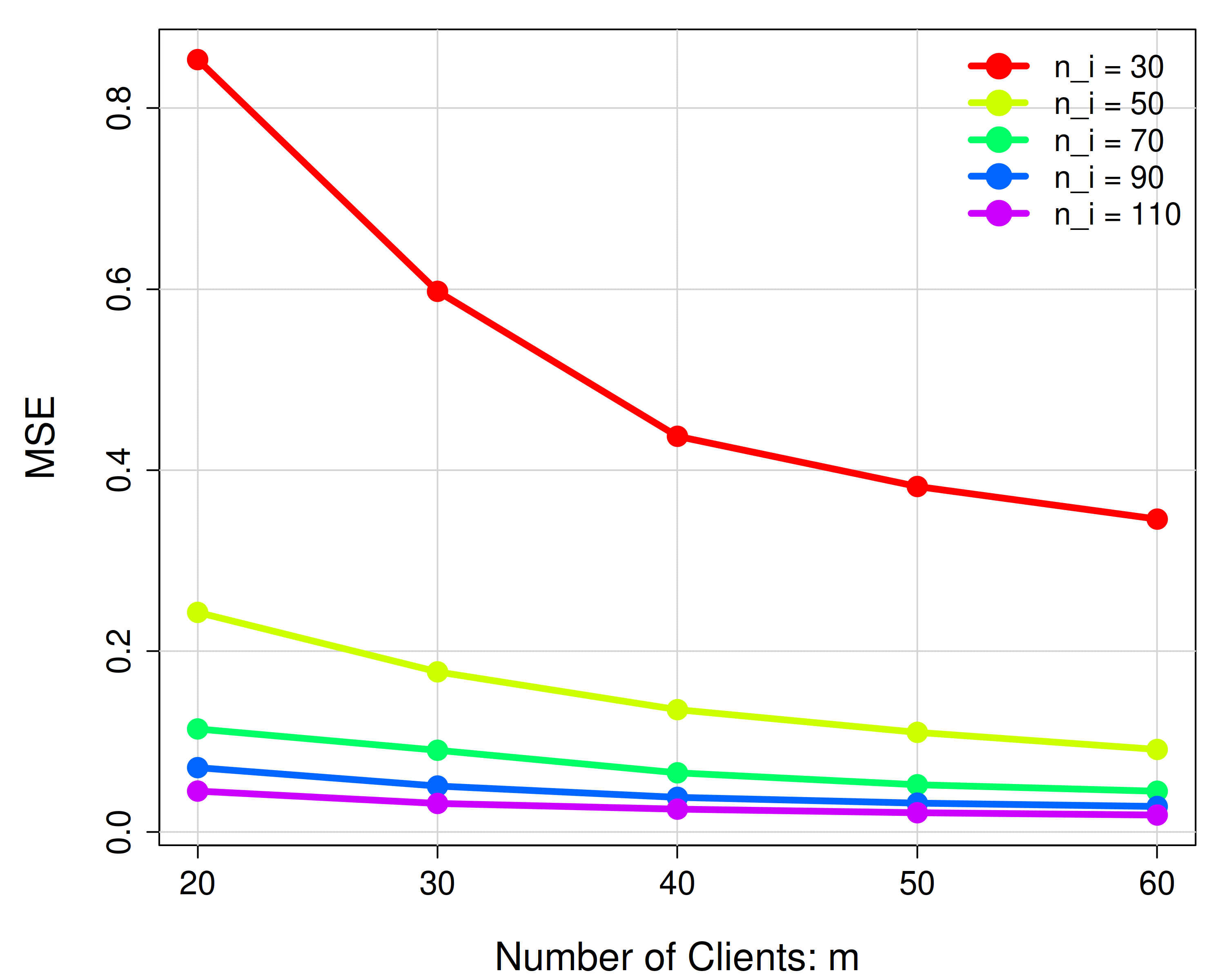}
        \caption{\FedNewton, same $n_i$}
    \end{subfigure}
    \hfill
    \begin{subfigure}[t]{0.45\textwidth}
        \centering
        \includegraphics[width=\textwidth]{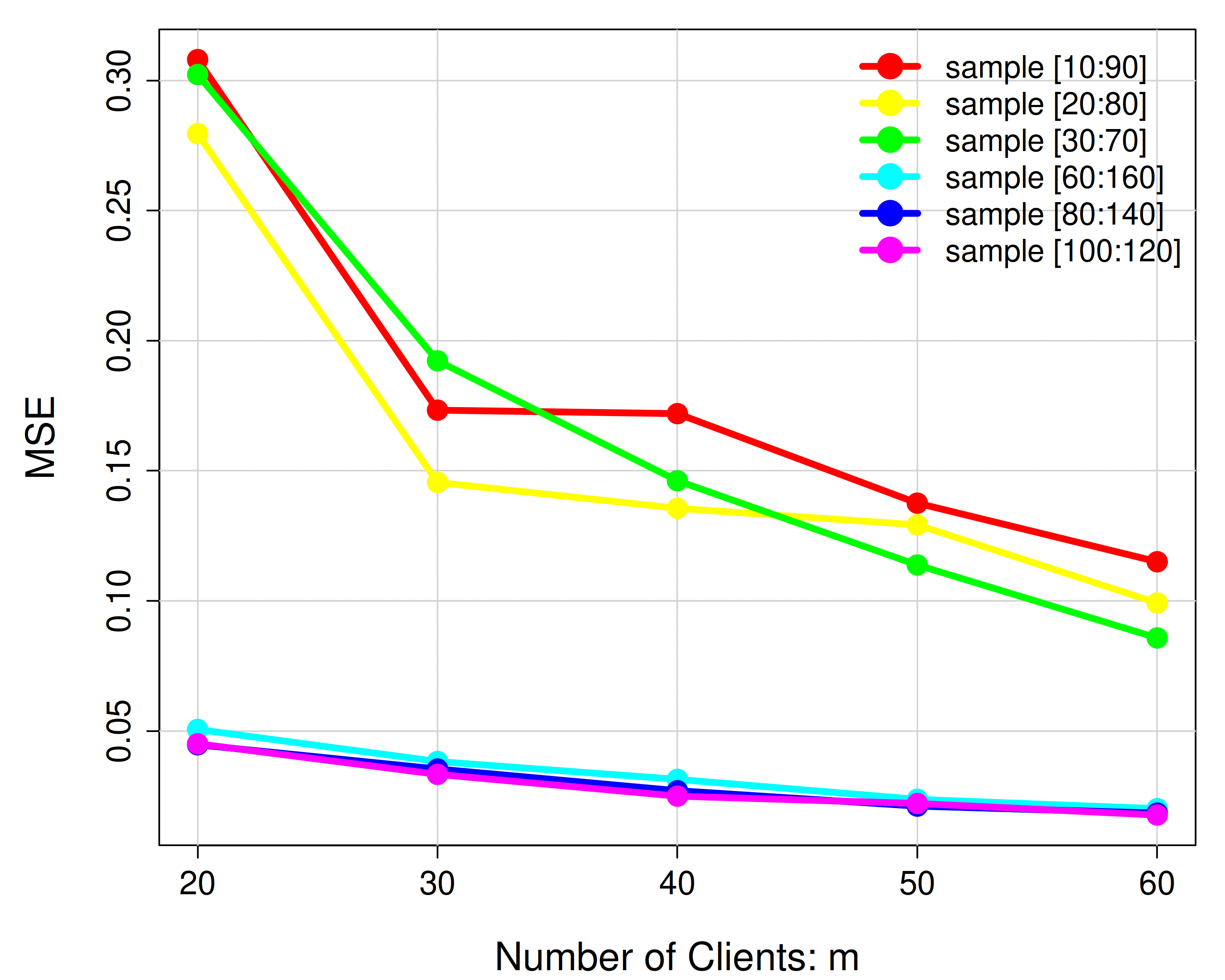}
        \caption{\FedNewton, different $n_i$}
    \end{subfigure}

    \caption{Empirical MSE results of \FedAvg\ and \FedNewton\ under Equal and different Local Sample Sizes in Poisson GLM.}
    \label{fig:poi_ag34_equal_different_ni}
\end{figure}

\begin{figure}[h]
    \centering

    \begin{subfigure}[t]{0.47\textwidth}
        \centering
        \includegraphics[width=\textwidth]{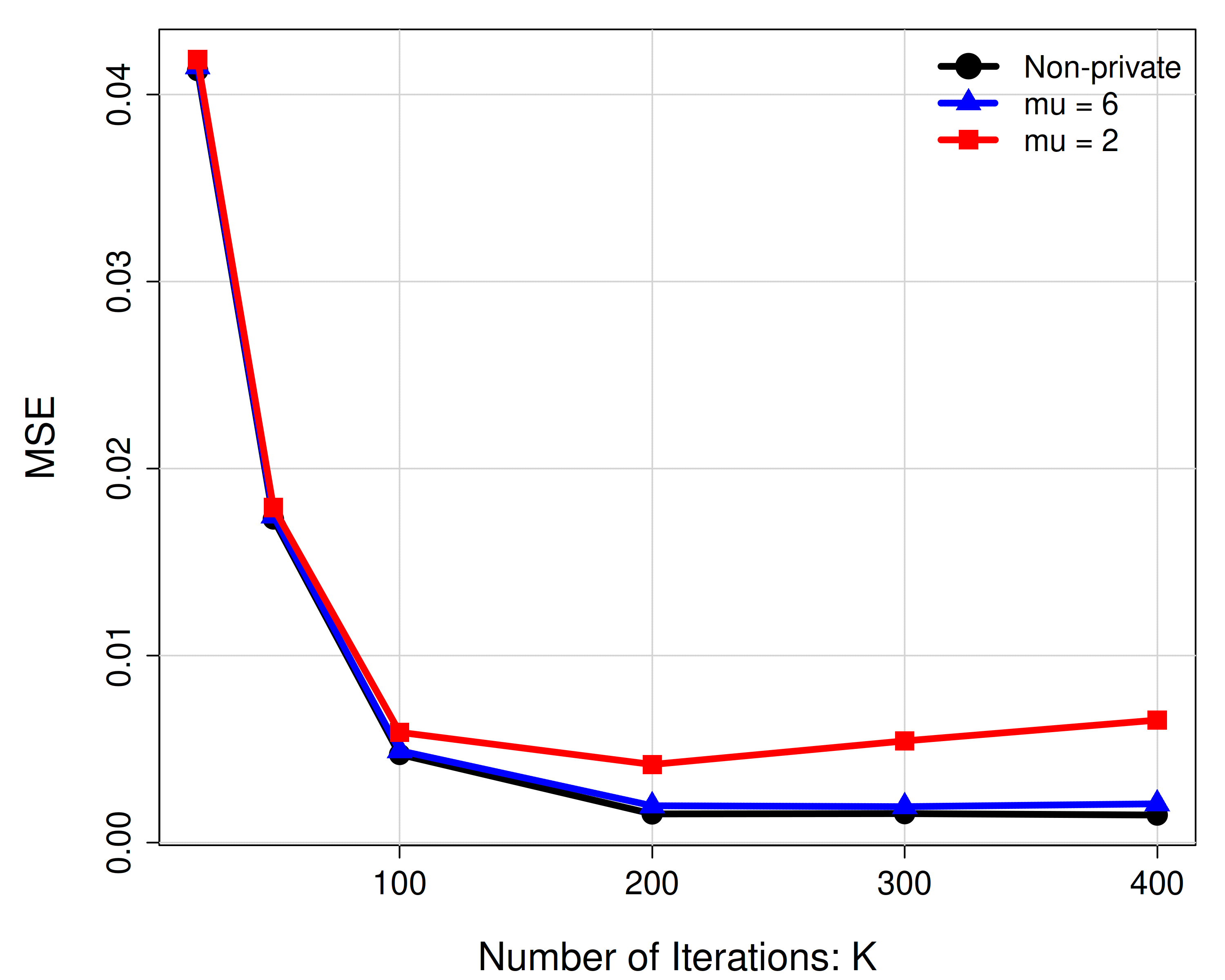}
        \caption{\FedSGD}
    \end{subfigure}
 \hfill
    \begin{subfigure}[t]{0.47\textwidth}
        \centering
        \includegraphics[width=\textwidth]{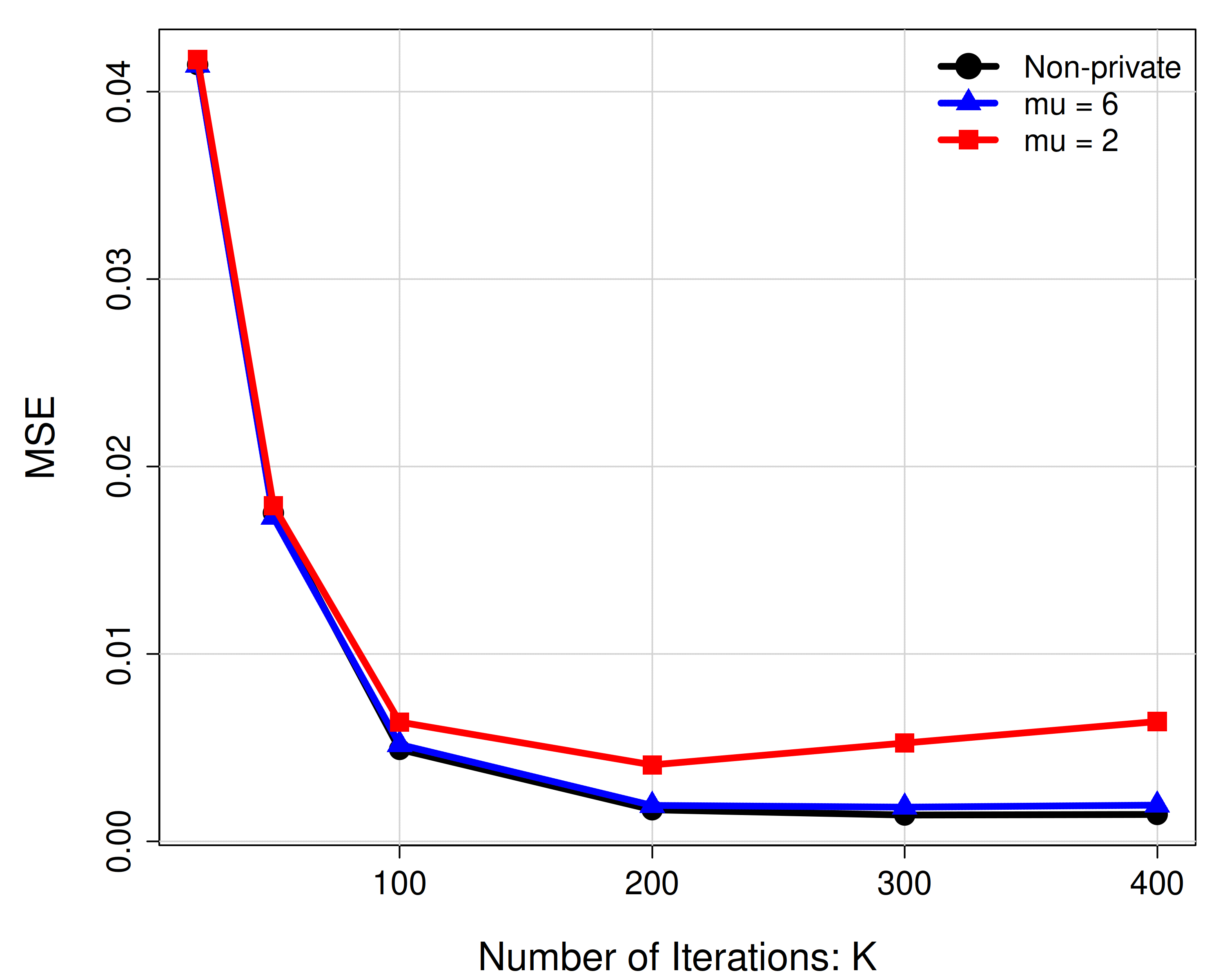}
        \caption{\FedHybrid}
    \end{subfigure}

    \begin{subfigure}[t]{0.47\textwidth}
        \centering
        \includegraphics[width=\textwidth]{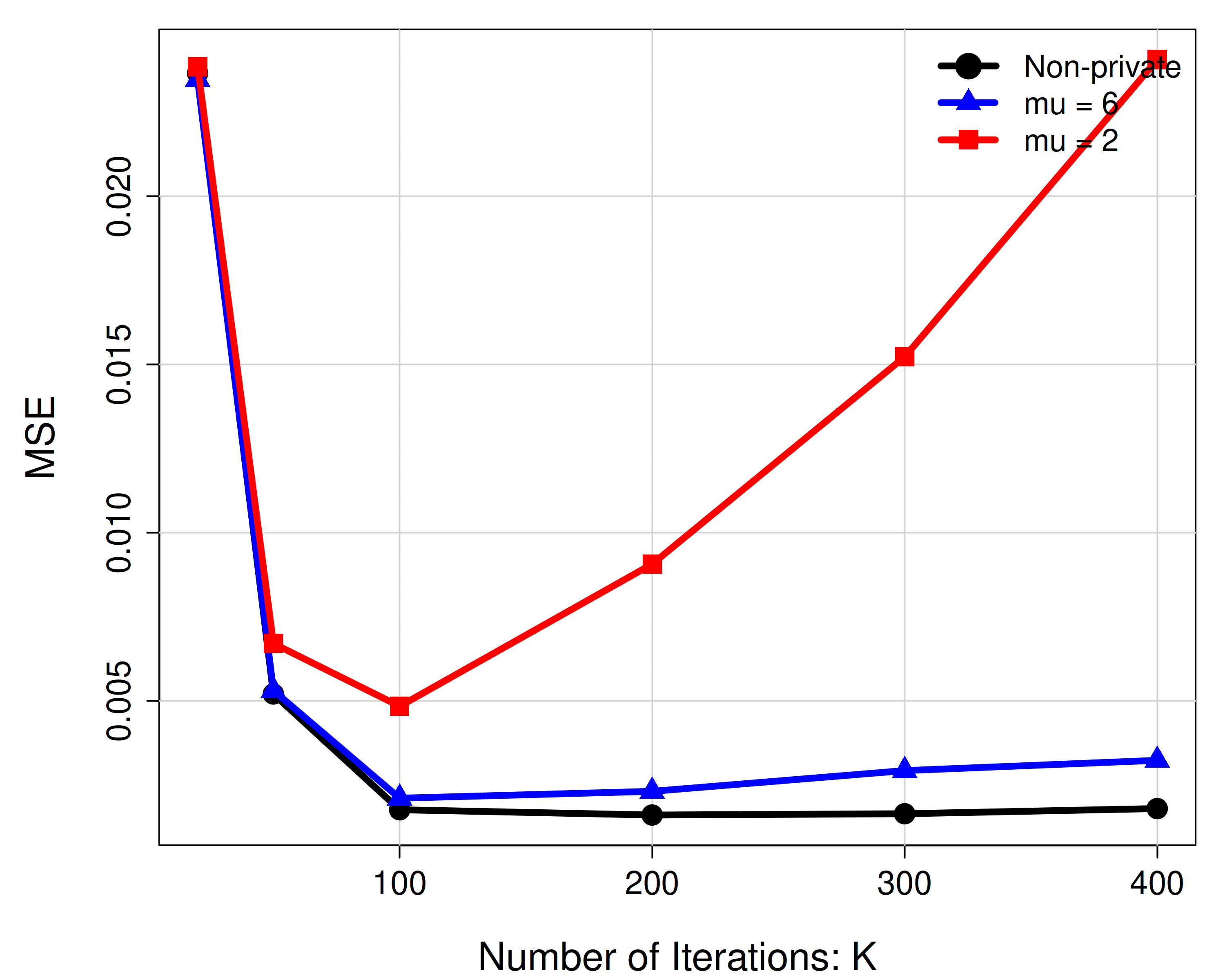}
        \caption{\FedAvg}
    \end{subfigure}
  \hfill
    \begin{subfigure}[t]{0.47\textwidth}
        \centering
        \includegraphics[width=\textwidth]{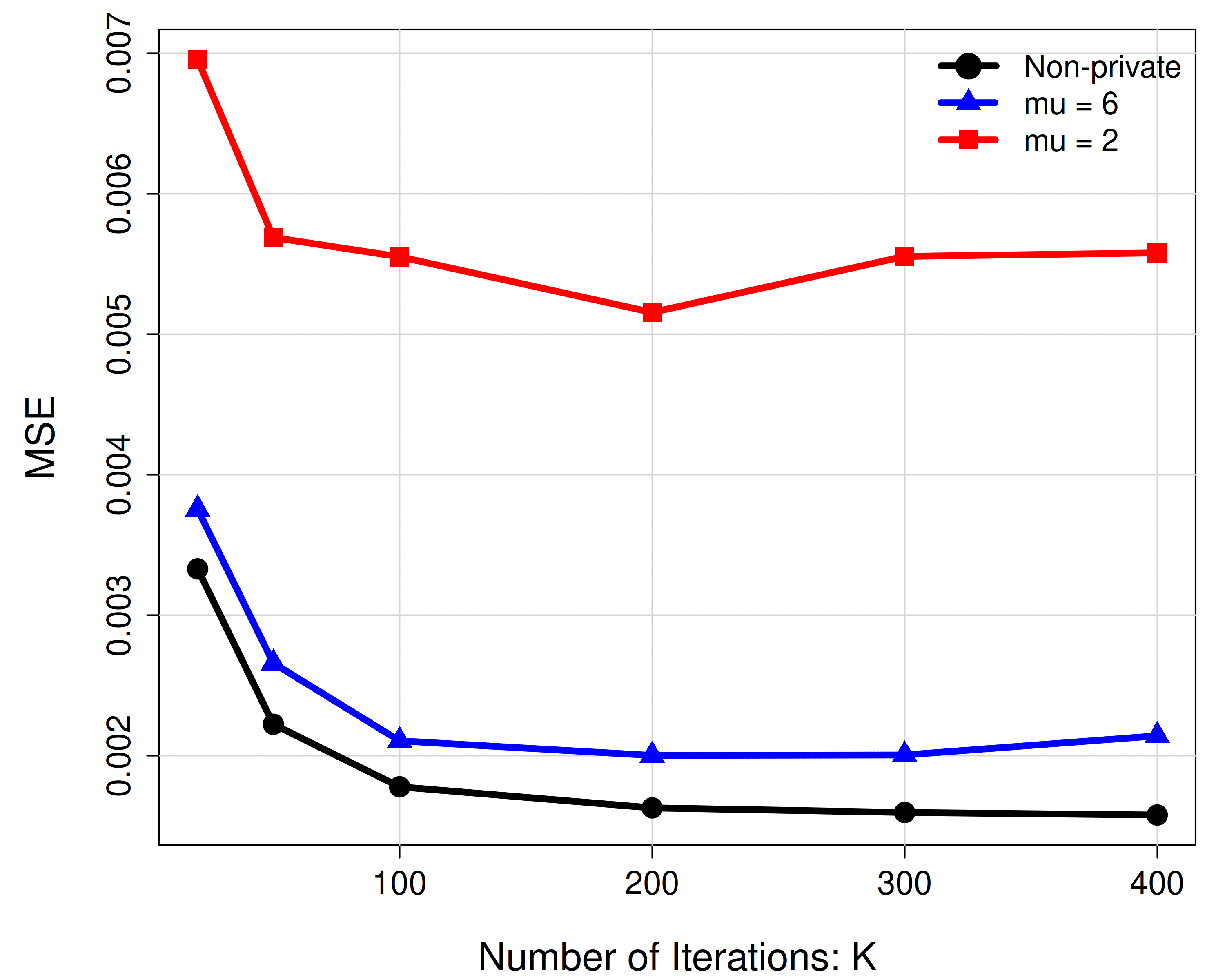}
        \caption{\FedNewton}
    \end{subfigure}

    \caption{The empirical MSE of the methods with increasing number of iterations illustrating the tradeoff between optimization quality and privacy under Poisson GLM.}
    \label{fig:poi_accuracy_pribacy_tradeoff}
\end{figure}

\begin{figure}[h]  
    \centering

    \begin{subfigure}{0.32\textwidth}
        \centering
        \includegraphics[width=\textwidth]{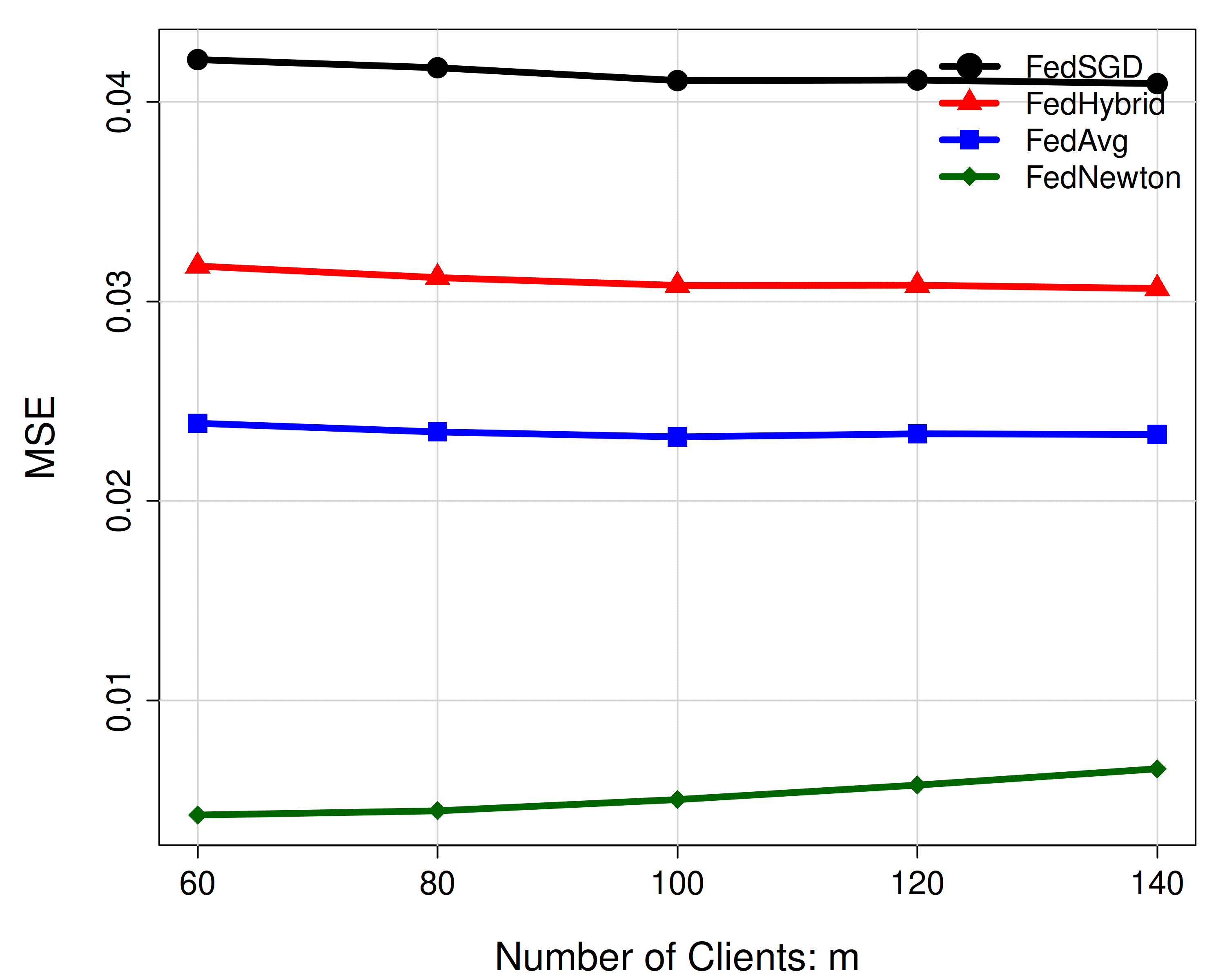}
        \caption{Equal $n_i$}
    \end{subfigure}%
    \begin{subfigure}{0.32\textwidth}
        \centering
        \includegraphics[width=\textwidth]{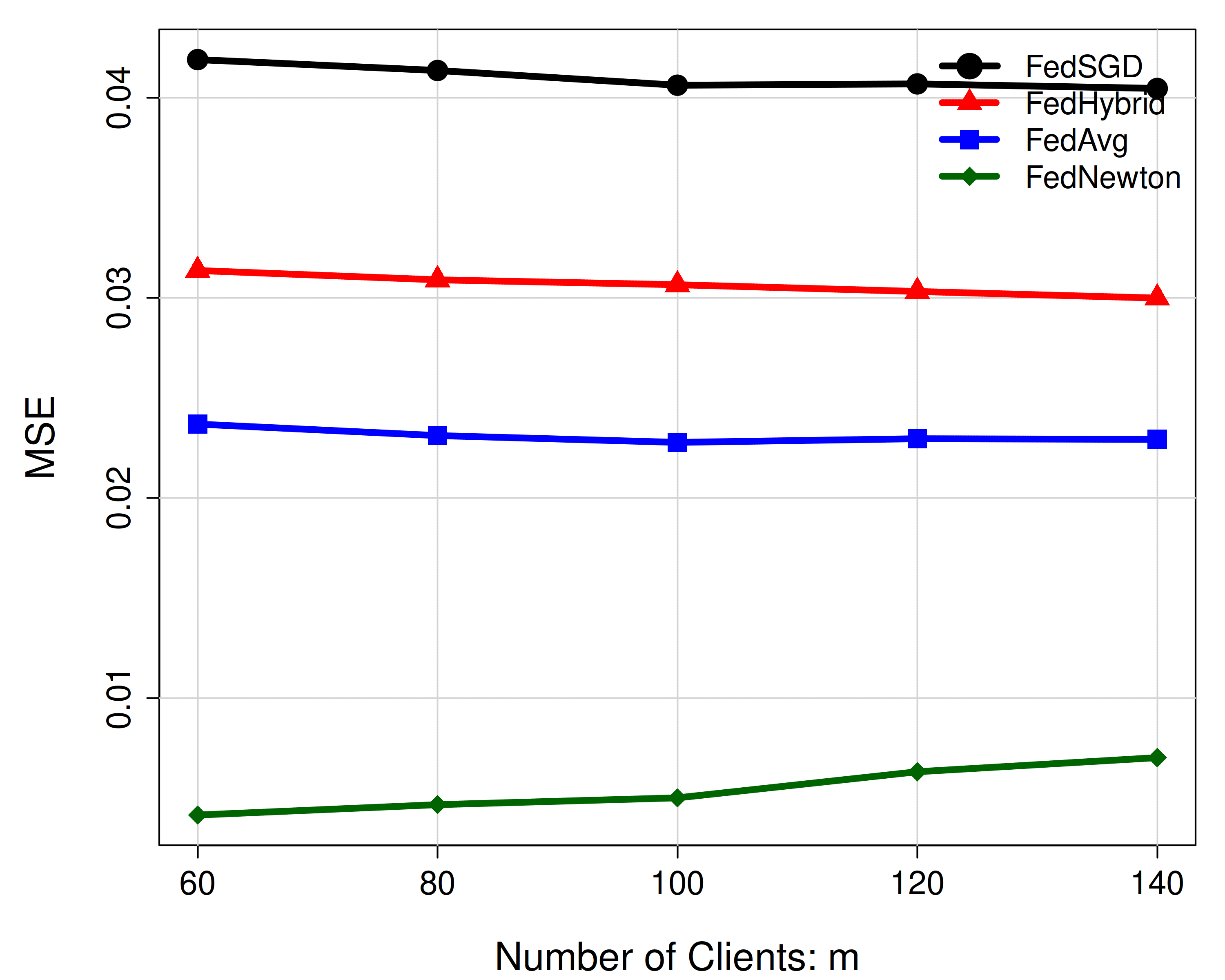}
        \caption{Uniformly distributed $n_i$}
    \end{subfigure}%
    \begin{subfigure}{0.32\textwidth}
        \centering
        \includegraphics[width=\textwidth]{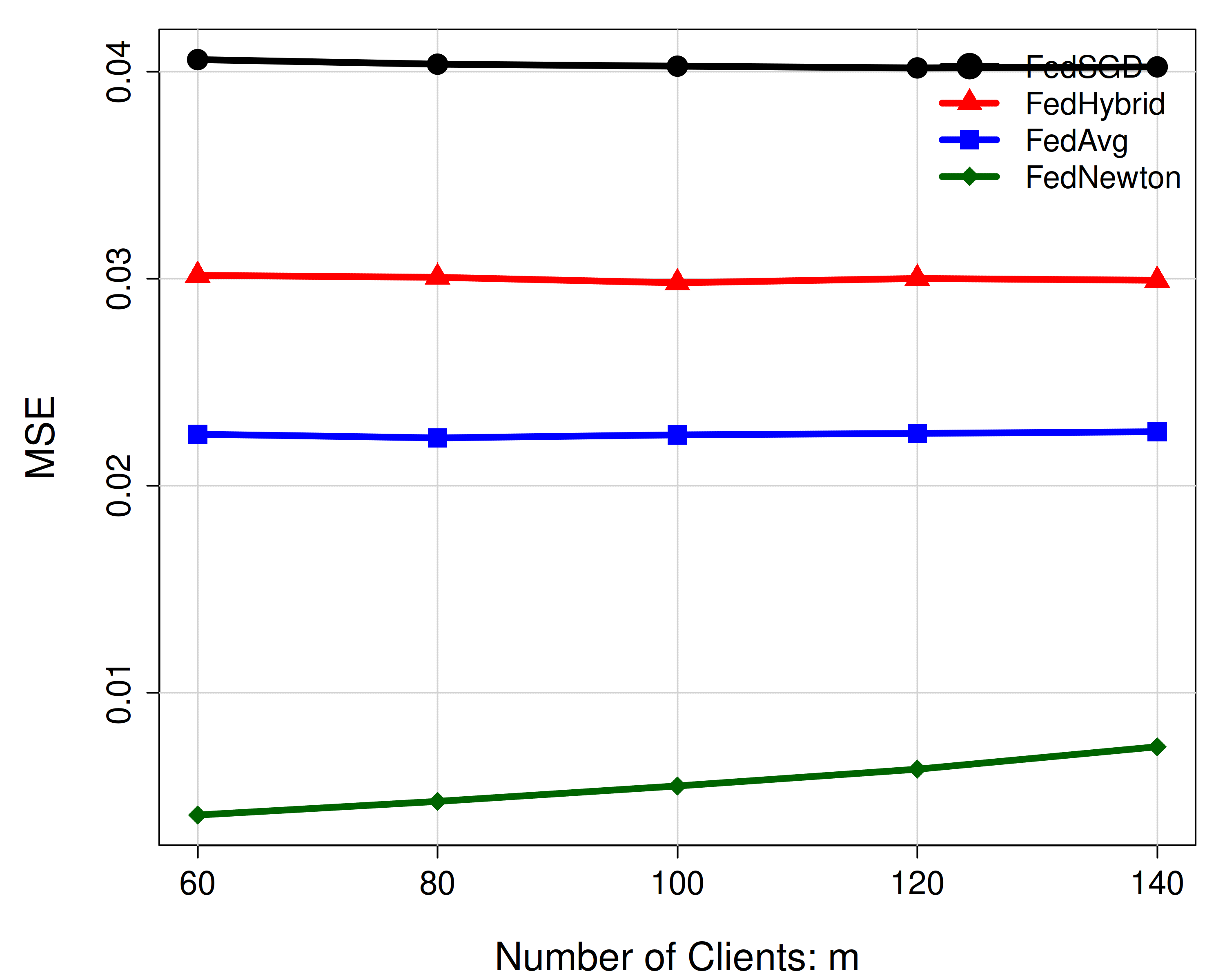}
        \caption{Lognormal distributed $n_i$}
    \end{subfigure}

\caption{MSE comparison of \FedSGD, \FedHybrid, \FedAvg, and \FedNewton\ under fixed total sample size \(N=20000\) with varying number of clients \(m\) in Poisson GLM: (a) equal local sample sizes, (b) uniformly distributed local sample sizes, and (c) lognormally distributed local sample sizes.}
    \label{fig:poi_fixedN}
\end{figure}

\newpage

\end{document}